\documentclass[mnsc,nonblindrev]{informs3} 

\OneAndAHalfSpacedXI 

\usepackage[table,xcdraw]{xcolor}
\usepackage[size=small]{subcaption}
\numberwithin{table}{section}
\usepackage{algorithm}
\usepackage{algorithmicx}
\usepackage{algpseudocode}
\usepackage{alphalph}
\usepackage{booktabs}
\usepackage{anyfontsize}

\usepackage[hidelinks]{hyperref}
\usepackage{amsmath,amssymb,mathtools}
\usepackage{bbm}
\usepackage{enumerate}
\usepackage{cases}
\usepackage{multirow}

\DeclareMathOperator{\tr}{\textsf{tr}}
\DeclareMathOperator{\rank}{\textsf{rank}}
\DeclareMathOperator{\vecr}{\textsf{vec}}

\usepackage{natbib}
 \bibpunct[, ]{(}{)}{,}{a}{}{,}%
 \def\bibfont{\small}%
 \def\bibsep{\smallskipamount}%
 \def\bibhang{24pt}%
\TheoremsNumberedThrough     
\ECRepeatTheorems

\EquationsNumberedThrough    

\begin{document}



\RUNTITLE{Transfer Learning of Word Embeddings}

\TITLE{Group-Sparse Matrix Factorization for Transfer Learning of Word Embeddings}

\ARTICLEAUTHORS{%
\AUTHOR{Kan Xu}
\AFF{W. P. Carey School of Business, Arizona State University, \EMAIL{kanxu1@asu.edu}} 
\AUTHOR{Xuanyi Zhao}
\AFF{University of Pennsylvania, \EMAIL{xuanyi.zhao@hotmail.com}} 
\AUTHOR{Hamsa Bastani}
\AFF{Wharton School, University of Pennsylvania, \EMAIL{hamsab@wharton.upenn.edu}} 
\AUTHOR{Osbert Bastani}
\AFF{University of Pennsylvania, \EMAIL{obastani@seas.upenn.edu}} 
} 

\ABSTRACT{%
Unstructured text provides decision-makers with a rich data source in many domains, ranging from product reviews in retail to nursing notes in healthcare. To leverage this information, words are typically translated into \emph{word embeddings}---vectors that encode the semantic relationships between words---through unsupervised learning algorithms such as matrix factorization. However, learning word embeddings from \textit{new} domains with limited training data can be challenging, because the meaning/usage may be different in the new domain, e.g., the word ``positive'' typically has positive sentiment, but often has negative sentiment in medical notes since it may imply that a patient tested positive for a disease. In practice, we expect that only a small number of domain-specific words may have new meanings. We propose an intuitive two-stage estimator that exploits this structure via a group-sparse penalty to efficiently \emph{transfer learn} domain-specific word embeddings by combining large-scale text corpora (such as Wikipedia) with limited domain-specific text data. We bound the generalization error of our transfer learning estimator, proving that it can achieve high accuracy with substantially less domain-specific data when only a small number of embeddings are altered between domains. Furthermore, we prove that all local minima identified by our nonconvex objective function are statistically indistinguishable from the global minimum under standard regularization conditions, implying that our estimator can be computed efficiently. Our results provide the first bounds on group-sparse matrix factorization, which may be of independent interest. We empirically evaluate our approach compared to state-of-the-art fine-tuning heuristics from natural language processing.
}%

\KEYWORDS{word embeddings, transfer learning, group sparsity, matrix factorization, natural language processing (NLP), text analytics}

\maketitle

%


\section{Introduction} \label{sec:intro}

Natural language processing is an increasingly important part of the analytics toolkit for leveraging unstructured text data in a variety of domains. For instance, service providers mine online consumer reviews to inform operational decisions on platforms~\citep{mankad2016understanding} or to infer market structure and the competitive landscape for products \citep{netzer2012mine}; Twitter posts are used to forecast TV show viewership \citep{liu2016structured}; analyst reports of S\&P 500 firms are used to measure innovation \citep{bellstam2020text}; medical notes are used to predict operational metrics such as readmissions rates \citep{hsu-etal-2020-characterizing}; online ads or reviews are used to flag service providers that are likely engaging in illicit activities \citep{ramchandani2021unmasking, li2021detecting}. 

To leverage unstructured text in decision-making, we must preprocess the text to capture the semantic content of words in a way that can be passed as an input to a predictive machine learning algorithm. In the past, this involved domain experts performing costly and imperfect feature engineering. A much more powerful, data-driven approach is to use unsupervised learning algorithms to learn \textit{word embeddings}, which represent words as vectors \citep{mikolov2013distributed,pennington2014glove}; we focus on widely-used word embedding models that are based on low-rank matrix factorization \citep{pennington2014glove, levy2014neural}. These word embeddings translate semantic similarities between words and the context within which they appear into statistical relationships. Typically, they are trained to encode how frequently pairs of words co-occur in text; these co-occurrence counts implicitly contain semantic properties of words since words with similar meanings tend to occur in similar contexts. Given the large number of words in the English language, to be effective in practice, embeddings must be trained on large-scale and comprehensive text data, e.g., popular embeddings such as Word2Vec \citep{mikolov2013distributed} and GloVe \citep{pennington2014glove} are trained on Wikipedia articles.

However, it is well-known that pre-trained word embeddings can miss out on important domain-specific meaning/usage, hurting downstream interpretation and effectiveness. Take the healthcare domain as an example. The word ``positive'' is typically associated with positive sentiment on Wikipedia; yet, in the context of medical notes, it typically indicates the presence of a medical condition, corresponding to negative sentiment. Thus, using a generic word embedding for ``positive'' may diminish performance in medical applications. Similarly, words like ``adherence" (referring to medication adherence) have a specific meaning in a healthcare context (relative to its context on general Wikipedia entries) and are strongly predictive of patient outcomes; failing to account for its healthcare-specific meaning may result in a loss in the downstream accuracy of healthcare-specific prediction tasks \citep{blitzer2007biographies}. Consequently, there has been a large body of work training specialized embeddings in a number of diverse contexts, ranging from radiology reports \citep{ong2020machine}, stock market prediction \citep{li2017learning}, cybersecurity vulnerability reports \citep{roy2017learning}, and patent classification \citep{risch2019domain}. This approach only works when the decision-maker has access to a sufficiently large domain-specific text corpus, allowing her to train high-quality embeddings. In practice, decision-makers often have limited domain-specific text data, yielding poor results when training new word embeddings, which hurts the quality of downstream modeling and decisions that leverage these embeddings. In other words, word embeddings trained on domain-specific data alone are unbiased but can have high variance due to limited sample size; in contrast, pre-trained word embeddings have low variance but can be significantly biased depending on the extent of domain mismatch.

Then, a natural question is whether we can combine large-scale publicly available text corpora (which we call the \emph{proxy} data hereafter) with limited domain-specific text data (which we call the \emph{gold} data hereafter) to train precise but domain-specific word embeddings. In particular, we aim to use transfer learning to achieve a better bias-variance tradeoff than using gold or proxy data alone. Our key insight to enable transfer learning is that the meaning/usage of most words do not change when changing domains; rather, we expect that only a small number of domain-specific words will have new meaning/usage. To illustrate, Figure \ref{fig:sparsity_ill} shows text data (paragraphs) from a variety of domain-specific Wikipedia articles, including finance, math, computing, and politics.\footnote{Details on the Wikipedia data can be found in \S\ref{sec:exp_wiki}.} Words that have a domain-specific meaning are enclosed in a red box,\footnote{Briefly, we categorize a word as domain-specific if any of the word's definitions on Wiktionary is labeled with key words from that specific domain; see \S\ref{sec:exp_wiki} for details.} while the remaining words share the same meaning/usage as in the standard English language. We observe that only a small number of unique words have domain-specific meaning/usage.

\begin{figure}[htbp]
\centering
\begin{subfigure}[b]{0.8\textwidth}
  \centering
  \includegraphics[width=\textwidth]{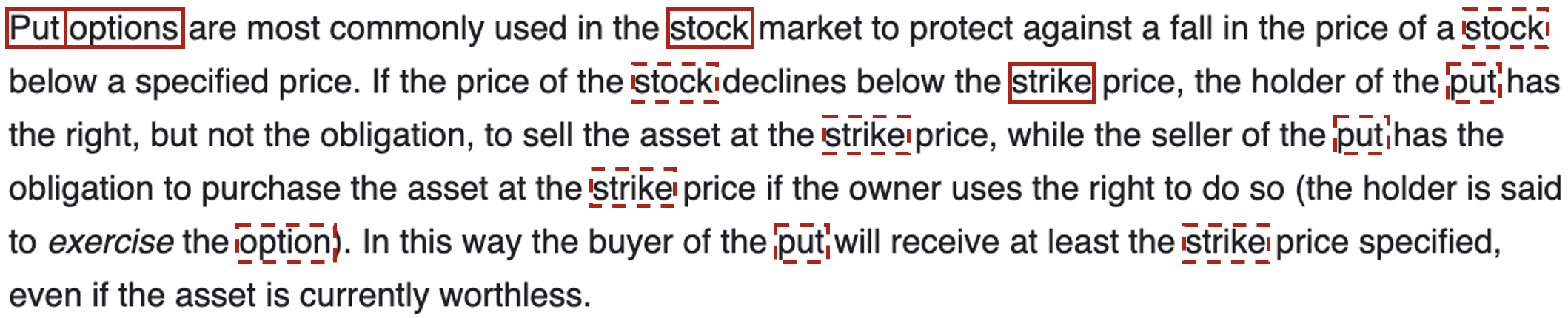}
  \caption{Finance Domain - ``Put Option''}
  \label{fig:sparsity_ill_fin}
\end{subfigure}\\ \medskip
\begin{subfigure}[b]{0.8\textwidth}
  \centering
  \includegraphics[width=\textwidth]{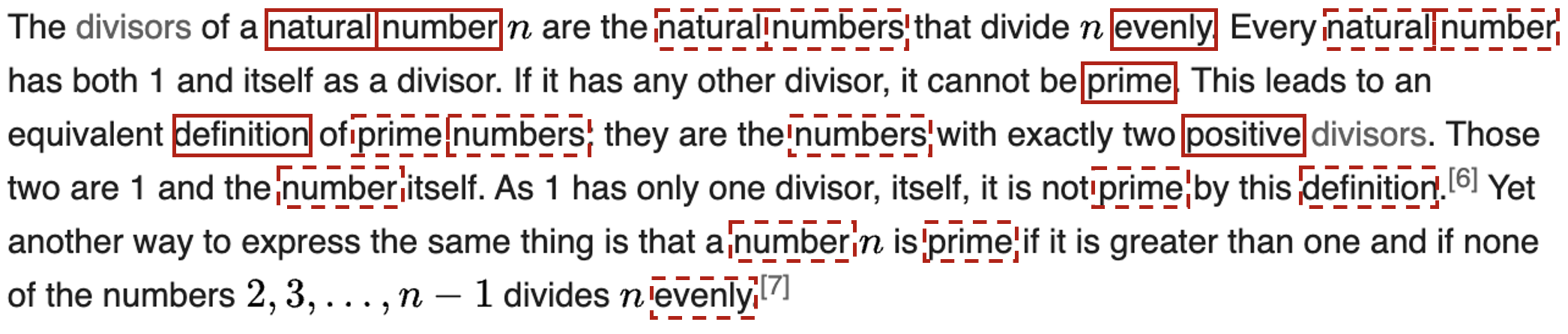}
  \caption{Math Domain - ``Prime Number''}
\end{subfigure} \\ \medskip
\begin{subfigure}[b]{0.8\textwidth}
  \centering
  \includegraphics[width=\textwidth]{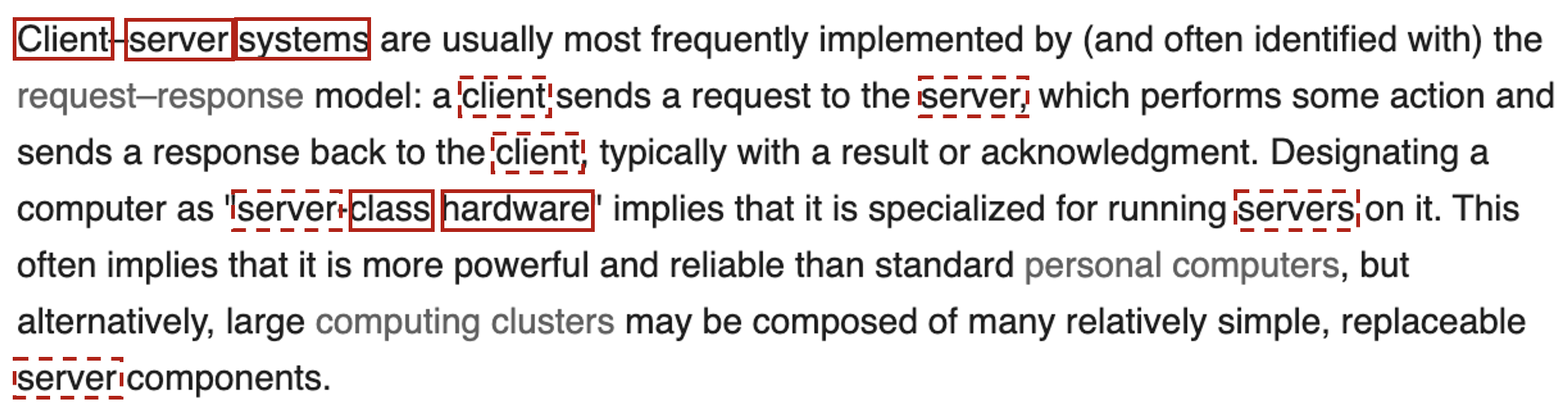}
  \caption{Computing Domain - ``Server''}
\end{subfigure}\\ \medskip
\begin{subfigure}[b]{0.8\textwidth}
  \centering
  \includegraphics[width=\textwidth]{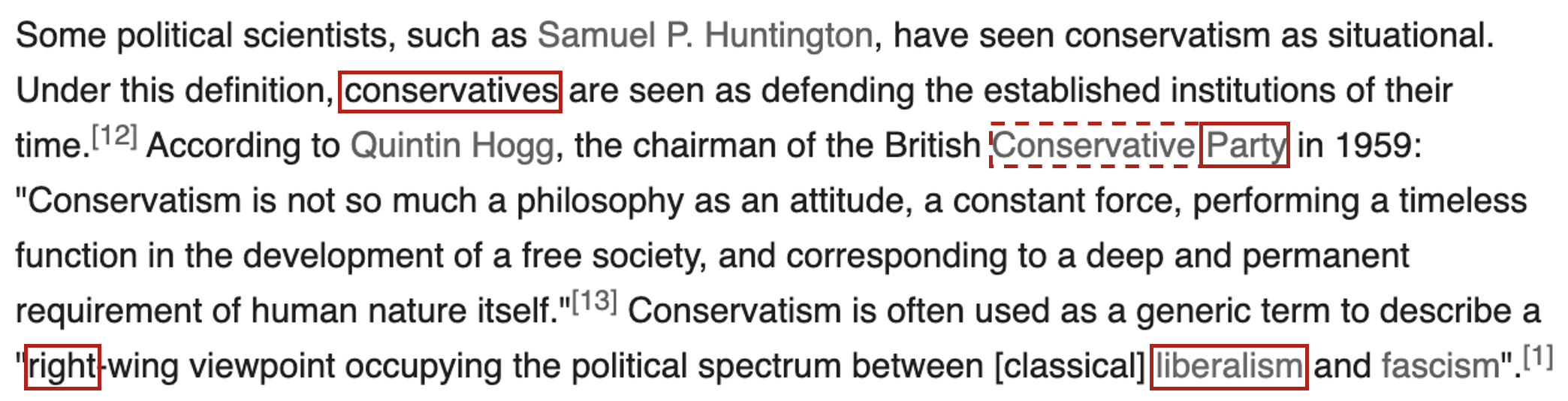}
  \caption{Politics Domain - ``Conservatism''}
\end{subfigure}
\caption{Paragraphs extracted from four Wikipedia articles of four domains respectively. We enclose domain-specific words in red boxes, distinguishing the first occurrence (solid line) from subsequent occurrences (dashed line). See \S\ref{sec:exp_wiki} for our definition of domain words and other experiments on Wikipedia data.}
\label{fig:sparsity_ill}
\end{figure}

\begin{figure}[htbp]
     \centering
     \begin{subfigure}[b]{0.25\textwidth}
         \centering
         \includegraphics[width=\textwidth]{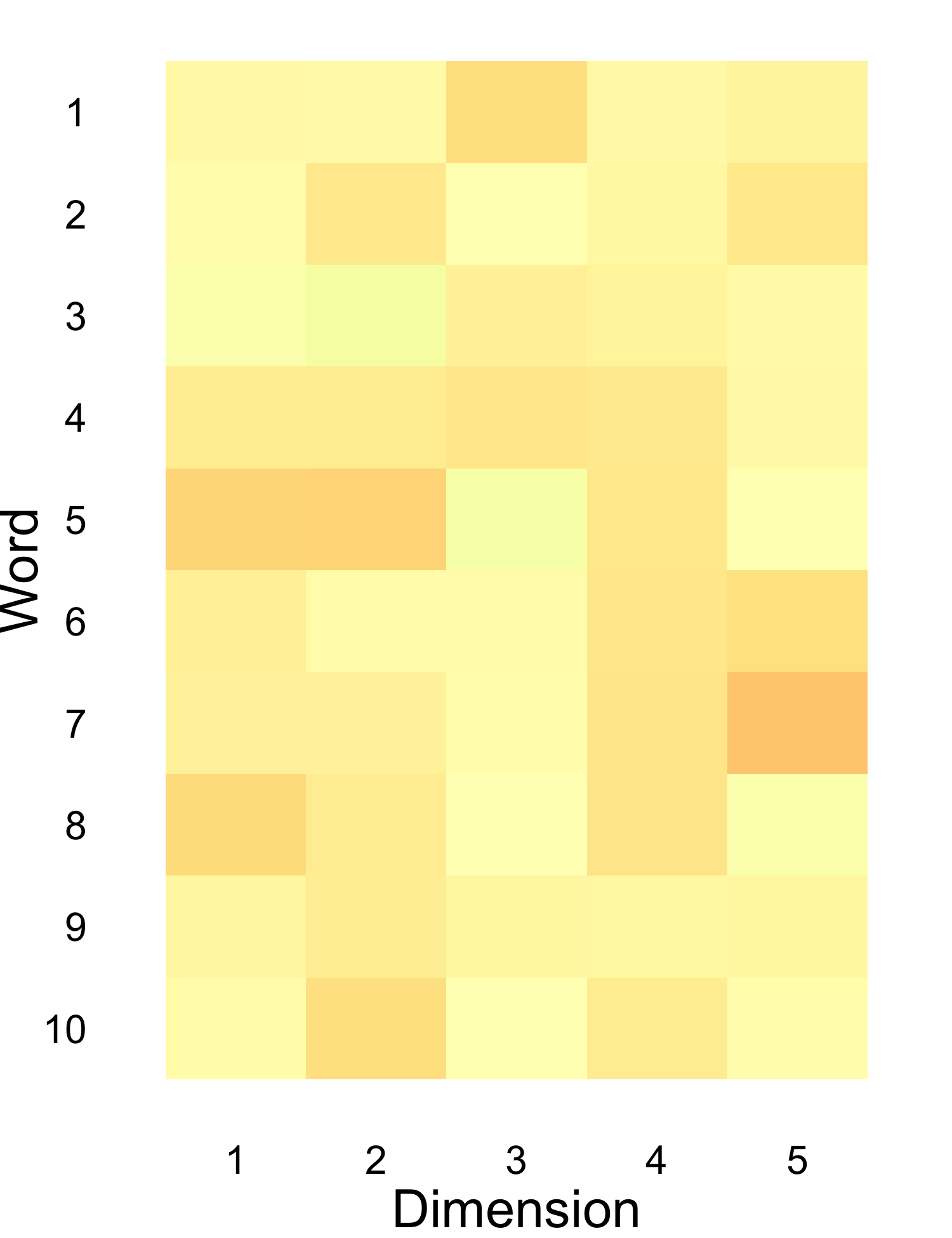}
         \caption{$U_p$}
         \label{fig:toy_proxy}
     \end{subfigure}
     \begin{subfigure}[b]{0.25\textwidth}
         \centering
         \includegraphics[width=\textwidth]{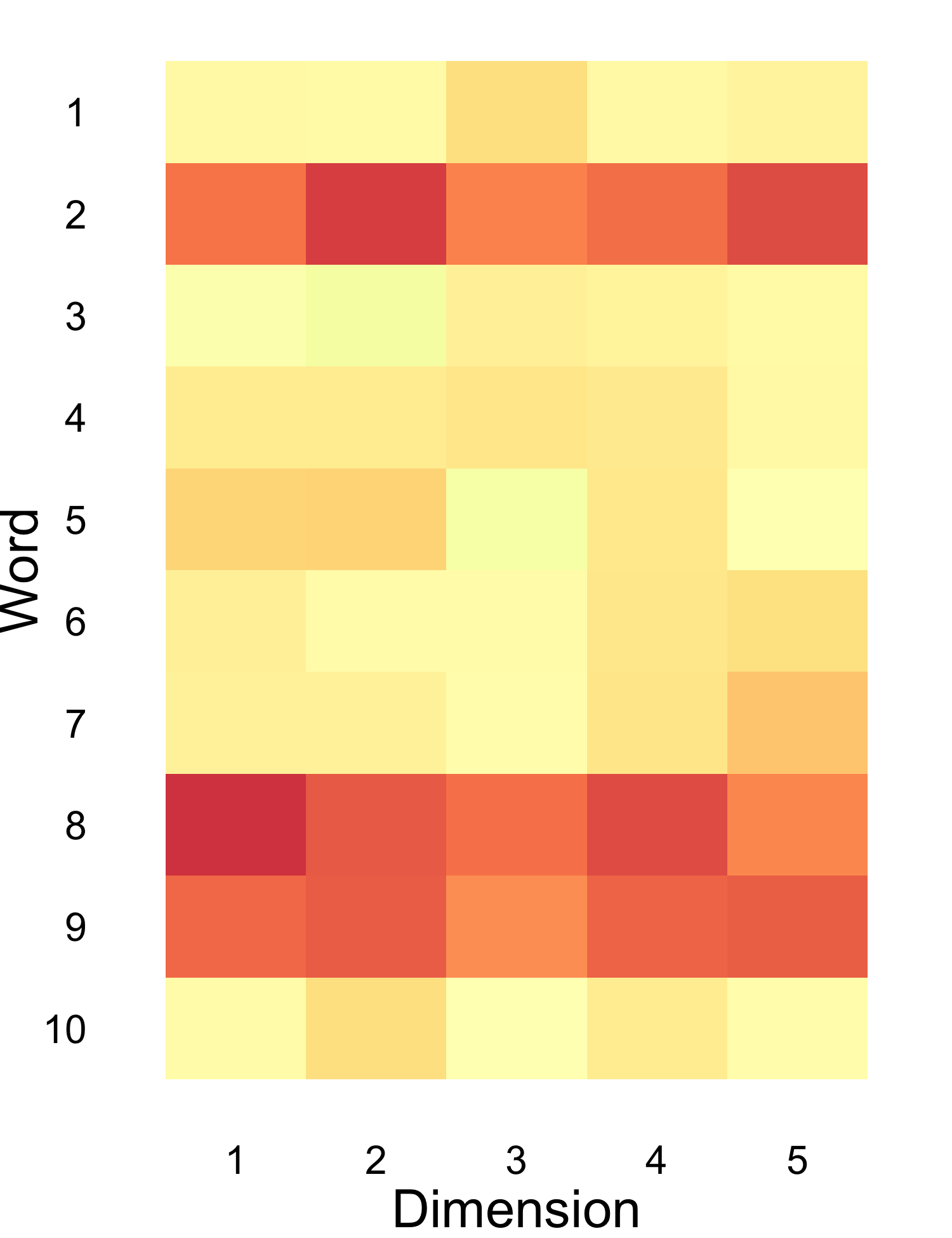}
         \caption{$U_g$}
         \label{fig:toy_gold}
     \end{subfigure}
     \begin{subfigure}[b]{0.25\textwidth}
         \centering
         \includegraphics[width=\textwidth]{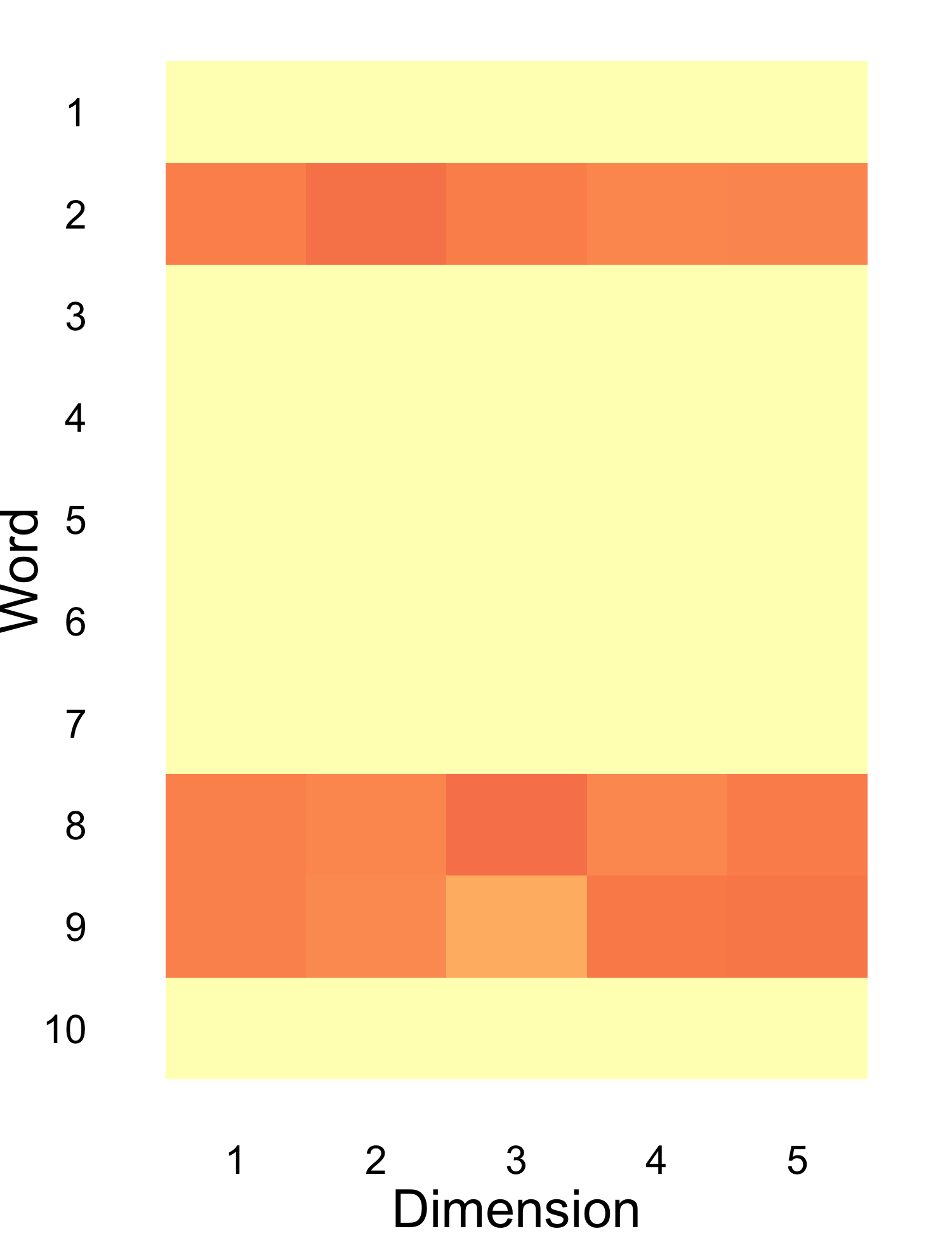}
         \caption{$U_g - U_p$}
         \label{fig:toy_diff}
     \end{subfigure}
        \caption{Toy example of (a) proxy and (b) gold word embedding matrices for $d=10$ and $r=5$. Only $s=3$ words change meaning/usage, inducing a group-sparse structure in the (c) difference matrix. The colors represent the magnitude of coefficients, ranging from zero (yellow) to large (red).}
        \label{fig:toy_rowsparse}
\end{figure}
More formally, consider a corpus of $d$ words. Let $U_p \in \mathbb{R}^{d\times r}$ denote the true (unobserved) proxy word embedding matrix, of which the $i^{\text{th}}$ row $U_p^{(i, \cdot)}$ is the true $r$-dimensional word embedding of word $i \in [d]=\{1,\cdots,d\}$ based on the proxy data; analogously, let $U_g \in \mathbb{R}^{d\times r}$ denote the true (unobserved) gold word embedding matrix. We expect that the meaning/usage for most words are preserved in both domains---i.e.,  the word embeddings $U_g^{(i, \cdot)} \neq U_p^{(i, \cdot)}$ for only a small number $s \ll d$ values of $i\in[d]$. This induces a \textit{group-sparse} structure for the difference matrix $U_g-U_p$, i.e., only a small number $s$ of the rows (groups) are nonzero. Figure~\ref{fig:toy_rowsparse} illustrates this notion of ``sparsity'' on a toy example with $d=10$ words, embeddings with dimension $r=5$, and $s=3$ words with shifted meaning/usage. We find support for this group-sparse structure in our previous examples from Wikipedia---e.g., in the finance domain (Figure~\ref{fig:sparsity_ill}(a)), we observe only $s=4$ unique finance domain-specific words (put, options, stock, strike) out of a total of $d=51$ distinct words, yielding a sparsity ratio $s/d \lesssim 0.08$. Similarly the sparsity ratios $s/d$ are approximately $0.11$, $0.07$ and $0.05$ for the math, computing, and politics examples in Figure~\ref{fig:sparsity_ill}. This trend is consistent when taking a larger-scale perspective at the domain (rather than paragraph) level---across all the Wikipedia articles we collect, the observed sparsity ratios for the aforementioned four domains are $0.09$, $0.12$, $0.08$, and $0.01$ respectively.

Based on this intuition, we formulate an objective that incorporates a group-sparse penalty \citep{friedman2010note,simon2013sparse} on $U_g - U_p$, where each row is treated as a group. In particular, we estimate domain-specific embeddings from gold data, incorporating $\ell_{2,1}$ regularization to impose group sparsity relative to the (estimated) word embeddings trained on the large proxy data. Our approach balances the need to update the embeddings of important domain-specific words based on the gold data (i.e., reduce bias), while matching most words to the embeddings estimated from the large proxy text corpus (i.e., reduce variance).

Our main result establishes that the word embedding estimator trained by group-sparse transfer learning achieves a sample complexity bound that, to leading order, scales quadratically in $s$ (the number of words with altered meaning/usage), as opposed to the conventional bound that scales quadratically in $d$ (the total number of words). In other words, transfer learning allows us to accurately identify domain-specific word embeddings with substantially less domain-specific data than classical low-rank matrix factorization methods. We build on prior work establishing error bounds for the group LASSO \citep{lounici2011oracle} and low-rank matrix problems \citep{ge2017no,negahban2011estimation}. We face two additional technical challenges. First, the literature on nonconvex low-rank matrix problems typically studies the Hessian to ensure that local minima are well-behaved; however, the Hessian may not be well-defined under our nonsmooth group-sparse penalty (since the gradient is not continuous). Second, unlike the traditional high-dimensional literature, transfer learning introduces a quartic form (in terms of $U_g - U_p$) in our objective function. We address both challenges through a new analysis that relies on an assumption we term ``quadratic compatibility condition.'' We show that quadratic compatibility is implied by a natural restricted strong convexity (RSC) assumption, which we prove holds with high probability in a general low-rank matrix factorization problem for the illustrative cases of gaussian data and word co-occurrence count data. Furthermore, under a slightly weaker condition that can characterize all local minima \citep{loh2015regularized}, all local minima identified by our algorithm are statistically indistinguishable from the global minimum, implying that our estimator can be computed efficiently.

While our technical results hold for embeddings trained using matrix factorization, our algorithm straightforwardly applies to nonlinear objectives such as GloVe. 
Simulations on synthetic data and real-world domain-specific Wikipedia articles show that our estimator significantly outperforms common heuristics given rich proxy data and limited domain-specific data. Importantly, we show that this is an \textit{interpretable} strategy to identifying key words with distinct meanings in specific domains such as finance, math, and computing. 

\subsection{Other Applications} \label{ssec:other-apps}

While we focus on natural language processing, our group-sparse transfer learning approach can be used in other settings where low-rank matrix factorization is useful:
\begin{itemize}
\item \textbf{Social networks:} In social network analysis \citep{li2017utility}, the adjacency matrix is often modeled as a low-rank object \citep{shi2015community,borgs2017thy}. One problem of interest is inferring actual co-working relationships from general-purpose social networks (which additionally reflect indirect or weak ties) such as LinkedIn~\citep{tang2012inferring}. One could apply our transfer learning approach to efficiently infer true co-working relationships by leveraging a small set of verified connections (gold data) along with general-purpose social network connections (proxy data). This approach can enhance downstream tasks such as link recommendation and community detection. The gold and proxy data may exhibit sparse structure as long as verified connections coincide with social network connections for most nodes.
\item \textbf{Supply chain management:} Our approach may also help infer long-term partnerships from supply chain data \citep{ren2010information} by applying low-rank matrix factorization to the bipartite graph formed from supplier-buyer relationships \citep{feng2024designing}. We might combine limited digital information on a company's long-term contracts (gold data) with large-scale transactional data (proxy data) to accurately estimate these relationships. Sparsity holds as long as most entities exhibit the same relationships in the transaction data and long-term supplier-buyer relationship data, which may hold since enduring partnerships often involve repeated purchases and transactions \citep{taylor2007supply}.
\item \textbf{Recommendation systems:} Low-rank matrix factorization is also commonly used in recommendation systems~\citep{zhou2008large}. Our framework can be used to learn recommendations for offline shopping by combining data from online browsing data (proxy data) with offline purchase records (gold data). Sparsity holds as long as most customers display consistent online browsing behaviors and offline purchasing behaviors, with only a few that deviate (e.g., browse high-end products online but make purchases mainly during in-store promotions).
\end{itemize}

\subsection{Related Literature}

Transfer learning involves transferring knowledge from a data-rich source (proxy) domain to a data-poor target (gold) domain (also called ``domain adaptation''). In order for such approaches to be effective, the two domains must be related in some way. For instance, the two domains may have the same label distribution $p(y\mid x)$ but different covariate distributions $p(x)$, a setting typically termed as ``covariate shift'' \cite[see, e.g.,][]{ben2007analysis, ben2010theory, ganin2015unsupervised}.
Our problem falls into the more challenging category known as ``label shift,'' where $p(y\mid x)$ itself differs across the two domains (since the underlying embeddings change for some words). 
A number of approaches have been proposed for addressing label shift in supervised learning problems \citep[see, e.g.,][]{lipton2018detecting,zhang2013domain}.\footnote{Problems with labeled source data and unlabeled target data are sometimes referred to as ``unsupervised''; we categorize them as ``supervised'' to distinguish from problems where both source and target data are unlabeled.} Our approach is most closely related to recent work applying LASSO for transfer learning \citep{bastani2020predicting}, where the label shift is driven by a \textit{sparse} shift in the underlying parameter vectors. Their key theoretical result is that relative sparsity between the gold and proxy parameter vectors is sufficient to enable efficient transfer learning in high dimensions. Existing theoretical results are critically limited to supervised learning. To the best of our knowledge, we propose the first framework for theoretically understanding the value of transfer learning in natural language processing (generally considered an unsupervised learning problem), which introduces new technical challenges.

However, a number of practical heuristics have been proposed for domain adaptation for natural language processing. A surprisingly effective transfer learning strategy is to simply \textit{fine-tune} pre-trained word embeddings on data from the target domain. Intuitively, stochastic gradient descent has
regularization properties similar to $\ell_2$ regularization \citep{ali2020implicit}, so this strategy can be interpreted as regularizing the target word embeddings towards the pre-trained word embeddings \citep{dingwall2018mittens,yang2017simple}. We demonstrate empirically that our approach of using $\ell_1$ regularization outperforms these heuristics in the low-data regime.

We build on approaches that construct word embeddings based on low-rank matrix factorization \citep{pennington2014glove,levy2014neural}. \cite{levy2014neural} show that one popular approach---skip-gram with negative sampling---implicitly factorizes a word-context matrix shifted by a global constant. Another popular approach is GloVe \citep{pennington2014glove}, which uses a nonlinear version of our loss function; our estimator extends straightforwardly to this setting.

Accordingly, we build on the theoretical literature on low-rank matrix factorization---specifically the Burer-Monteiro approach \citep{burer2003nonlinear}, which replaces $\Theta$ with a low-rank representation $UU^T$, with $U \in \mathbb{R}^{d \times r}$, and minimizes the objective in $U$. \cite{ge2017no} shows that the local minima of this nonconvex problem are also global minima under the restricted isometry property; \cite{li2019non} extend this by considering a more general objective function that satisfies a restricted well-conditioned assumption. One alternative is nuclear-norm regularization~\citep{recht2007guaranteed,candes2011tight,negahban2011estimation}, but this algorithm lends less naturally to our transfer learning objective and is often computationally inefficient.

A related literature also examines transfer learning for low-rank matrix estimation, but without group-sparse structure. \cite{yang2020precise} and \cite{duan2024target} consider an asymmetric matrix $\Theta = UV^T$, where the target and source domains share the same $U$---in our symmetric setting where $U=V$, this leads to trivial knowledge sharing, since the target and source models are identical.

This paper extends our earlier short conference paper \cite{xu2021group} as follows. First, we show that the quadratic compatibility condition (a critical component of our proofs) is implied by a natural restricted strong convexity condition, which we prove holds with high probability in a general low-rank matrix factorization problem for the illustrative cases of gaussian data and word co-occurrence count data (\S\ref{ssec:qcc}). Second, more importantly, we prove that all local minima identified by our estimator are statistically indistinguishable from the global minimum under a slightly weaker condition proposed by \cite{loh2015regularized} that is likely to hold for all local minima (\S\ref{ssec:localmin}). This result significantly strengthens our main result by showing that the optimization problem used to compute our estimator is tractable in practice. Third, we relate our error bounds back to the scaling specific to word embedding models (Corollary \ref{cor:joint_estimator}--\ref{cor:proxy_estimator}). Finally, we significantly expand the experimental results on both synthetic and real data to illustrate the value and robustness of our approach.

\section{Problem Formulation}

We formalize the problem of learning word embeddings as a low-rank matrix sensing problem (\S\ref{sec:cla_mff}), describe our model setup for transfer learning (\S\ref{ssec:tl}), and define our proposed transfer learning estimator that combines gold and proxy data to learn domain-specific word embeddings (\S\ref{subsec:estproc}). 

\textbf{Notation.}
For any vector $v \in \mathbb{R}^{d}$, let $\|v\|$ denote its $\ell_2$ norm. 
For a matrix $\Theta \in \mathbb{R}^{d_1\times d_2}$, we use superscript $(i,j)$ to represent entry $(i,j)$ of a matrix $\Theta$, $(i, \cdot)$ the $i^{\text{th}}$ row of the matrix, and  $(\cdot, j)$ the $j^\text{th}$ column. 
For a matrix $\Theta$ of rank $r$, we denote its singular values by $\sigma_{\max}(\Theta) = \sigma_1(\Theta) \ge \sigma_2(\Theta) \ge \cdots \ge \sigma_r(\Theta) = \sigma_{\min}(\Theta) > 0$, its Frobenius norm by $\|\Theta\|_F = \sqrt{\sum_{j=1}^r \sigma_j^2(\Theta)}$, its operator norm by $\|\Theta\| = \sigma_1(\Theta)$, its vector $\ell_\infty$ norm by $|\Theta|_{\infty}=\max_{i,j} |\Theta^{(i,j)}|$, its vector $\ell_1$ norm by $|\Theta|_1=\sum_{i,j} |\Theta^{(i,j)}|$, and its matrix $\ell_{2,1}$ norm by $\|\Theta\|_{2,1}=\sum_{i=1}^{d_1}\|\Theta^{(i, \cdot)}\|$. Given $\Theta,\Theta'\in\mathbb{R}^{d_1\times d_2}$, we denote the matrix dot product by $\langle\Theta,\Theta'\rangle=\sum_{i=1}^{d_1}\sum_{j=1}^{d_2}\Theta^{(i,j)}\Theta'^{(i,j)}$. Finally, let $[k]=\{1, 2, \cdots, k\}$. 

\subsection{Matrix Sensing}\label{sec:cla_mff}

Our word embedding model is an instance of the more general setting of matrix sensing~\citep{recht2007guaranteed}, where one aims to recover an unknown symmetric matrix $\Theta^* \in \mathbb{R}^{d \times d}$ with rank $r \ll d$. In other words, we can write $\Theta^*=U^*U^{*T}$ where $U^* \in \mathbb{R}^{d \times r}$. The typical goal in matrix sensing is to estimate $\Theta^*$ given observation matrices $A_i \in \mathbb{R}^{d \times d}$ and $X_i \in \mathbb{R}$, for $i\in[n]$, where
\begin{align}
\label{eqn:observations}
X_i=\langle A_i,\Theta^*\rangle+\epsilon_i,
\end{align}
and $\epsilon_1,\cdots,\epsilon_n$ are independent $\sigma$-subgaussian random variables (Definition~\ref{def:subgaussian_rv}). This model is introduced to generalize low-rank matrix factorization \citep{recht2007guaranteed}---the observation matrices $A_i$ allow for general observation models of the underlying low-rank model $\Theta^*$, intuitively playing a similar role as covariates in classical linear regression.

To simplify notation, we define the linear operator $\mathcal{A}:\mathbb{R}^{d\times d}\to\mathbb{R}^n$, where $\mathcal{A}(\Theta)_i=\langle A_i,\Theta\rangle$. Then, we can write
\begin{align*}
X=\mathcal{A}(\Theta^*)+\epsilon,
\end{align*}
where $X=[X_1, \cdots, X_n]^T$ and $\epsilon=[\epsilon_1,  \cdots, \epsilon_n]^T$.

\begin{definition}\label{def:subgaussian_rv}
A random variable $Z$ is $\sigma$-subgaussian if, for any $t\in \mathbb{R}$, $\mathbb{E}[Z]=0$ and $\mathbb{E}[\exp(tZ)] \le \exp(\sigma^2t^2/2)$.
\end{definition}

As we will discuss at the end of this subsection, in natural language processing, $\Theta^*$ corresponds to the word co-occurrence probability matrix, while $U^*$ corresponds to the word embeddings. Thus, in contrast to the matrix sensing literature which aims to estimate $\Theta^*$, our goal is to estimate the low-rank representation $U^*$. However, we can only compute $U^*$ up to an orthogonal change-of-basis since $\Theta^*$ is preserved under such a transformation---i.e., if we let $\widetilde{U}^*=U^*R$ for an orthogonal matrix $R\in\mathbb{R}^{r\times r}$, then we still obtain $\widetilde{U}^*\widetilde{U}^{*T}=U^*RR^TU^{*T}=U^*U^{*T}=\Theta^*$. Thus, our goal is to compute $\widehat{U}$ such that $\widehat{U}\approx U^*R$ for some orthogonal matrix $R$.

We build on \cite{burer2003nonlinear}, which solves the following optimization problem:
\begin{align*}
\min_{U \in \mathbb{R}^{d \times r}} \frac{1}{n} \|X - \mathcal{A}(U U^T)\|^2.
\end{align*}
Despite its nonconvex loss function, this estimator performs well in practice, and has desirable theoretical properties (i.e., no spurious local minima) under the restricted isometry property~\citep{ge2017no}.

We measure the estimation error of $\widehat{U}$ using the $\ell_{2,1}$ norm, which is more compatible with the group-sparse structure that we will impose shortly. 
In addition, since we can only identify $U^*$ up to orthogonal change-of-basis, we consider the following rotation-invariant error.
\begin{definition}\label{def:error_dir}
Given $\widehat{U},U^*\in\mathbb{R}^{d\times r}$, the error of $\widehat{U}$ is
\begin{align*}
\ell(\widehat{U},U^*) &= \|\widehat{U}-U^*R_{(\widehat{U},U^*)}\|_{2,1},
\end{align*}
where $R_{(\widehat{U},U^*)} = \argmin_{R:R^TR=RR^T=\mathbf{I}} \| \widehat{U} - U^*R \|_F$.
\end{definition}

\begin{remark}
An alternative approach to Burer-Monteiro is to estimate $\Theta^*$ directly using nuclear norm regularization \citep[see, e.g., ][]{candes2011tight,negahban2011estimation}. However, this approach is often too computationally costly in large-scale problems \citep{recht2007guaranteed}. Furthermore, estimating $U^*$ is more natural in our setting since our final goal is to recover $U^*$ (rather than $\Theta^*$), and our transfer learning strategy penalizes deviations in $U^*$.
\end{remark}

\textbf{Word embeddings.} 
Word embedding models typically consider how often pairs of words co-occur within a fixed-length window. Without loss of generality, we consider neighboring word pairs, i.e., a window with length 1. Let the length of our text corpus be $n+1$ so that the total number of neighboring word pairs is $n$. Recall that we have $d$ unique words, and we define our word co-occurrence matrix to be $\Theta^* \in \mathbb{R}^{d\times d}$, where the $(j, k)$ entry $\Theta^{*(j,k)}$ is the probability that word $j$ and word $k$ appear together. To estimate each of these $d^2$ probabilities, e.g., $\Theta^{*(j, k)}$ of word pair $(j, k)$, we randomly draw $n$ word pairs from the text with replacement and record the outcome as a binary indicator for whether the draw matches the pair $(j,k)$. We draw samples independently across all $d^2$ possible word pairs.\footnote{In practice, one could simply enumerate all $n$ word pairs to construct each $\Theta^{*(j,k)}$ instead of using sampling, e.g., this is typically how the GloVe model~\citep{pennington2014glove} is trained. However, for the purpose of establishing convergence guarantees, using the same $n$ observed word pairs to estimate each $\Theta^{*(j, k)}$ induces non-independence in our observations. To remedy this, we randomly draw $n$ word pairs with replacement for each pair $(j, k)$ independently.} This yields $d^2n$ samples in total; for each $i \in [d^2n]$, the outcome is a binary variable named $X_i$ that takes value 1 if the draw $i$ is exactly the pair $(j, k)$ and 0 otherwise. We encode the corresponding word pair $(j, k)$ in a basis matrix $A_i \in \mathbb{R}^{d\times d}$, whose $(j, k)$ entry equals 1 and 0 otherwise --- i.e., $A_i=E_{jk}$ where $E_{jk}$ is the basis matrix with entry $(j, k)$ being 1 and 0 otherwise. Note that the $i^{th}$ draw corresponds to the word pair $(j, k)$ with probability $\Theta^{*(j,k)}=\langle A_i, \Theta^* \rangle$. Therefore, we can think of $X_i$ as a Bernoulli random variable with mean $\langle A_i, \Theta^* \rangle$, i.e., $X_i\sim\texttt{Bernoulli}(\langle A_i,\Theta^*\rangle)$, and our observation model has the form
\begin{align}\label{eqn:observations_we}
X_i = \langle A_i, \Theta^* \rangle + \epsilon_i,
\end{align}
where $A_i$ is a basis matrix, $X_i$ is a Bernoulli random variable, and $\epsilon_i$ is the noise. The model in (\ref{eqn:observations_we}) has been used in other applications (without transfer learning), e.g., for recommendation systems in \cite{farias2019learning}---their outcome $X_i$ is a binary indicator that equals 1 if customer $j$ has purchased product $k$ in the past month and 0 otherwise, and their $\langle A_i, \Theta^* \rangle$ is the probability of such transactions.
We discuss how our general results scale under this model in \S \ref{ssec:mainresult}.

\begin{remark}
In this paper, we primarily focus on symmetric matrix estimation due to the natural symmetry of word co-occurrence. Our approach could be extended to the asymmetric setting. Intuitively, an asymmetric matrix factorization problem can be reformulated as a symmetric one \citep[see, e.g.,][]{ge2017no}. For example, for model \eqref{eqn:observations_we} with asymmetric $\Theta^* = U^*V^{*T}$, we can define a symmetric parameter
$\widetilde{\Theta}^* = \begin{bmatrix} U^* \\ V^* \end{bmatrix} \cdot \begin{bmatrix} U^* \\ V^* \end{bmatrix}^T$, allowing us to rewrite the model as:
\begin{align*}
X_i = \langle \widetilde{A}_i, \widetilde{\Theta}^* \rangle + \epsilon_i, \quad \text{where} \quad
\widetilde{A}_i = \frac{1}{2}\begin{bmatrix} 0 & A_i \\ A_i^T & 0 \end{bmatrix}.
\end{align*}
Despite this connection, the analysis will require different regularization conditions and assumptions. For example, estimating an asymmetric $\Theta^*$ often requires additional regularizer terms in the loss function, such as $\|U^TU-V^TV\|_F^2$, to stabilize convergence \citep{ge2017no,park2017non}. A detailed analysis is beyond the scope of this paper but is a promising direction for future work.
\end{remark}

\subsection{Transfer Learning} \label{ssec:tl}

We now consider transfer learning from a large text corpus to the desired target domain. Let $U_p^*\in\mathbb{R}^{d\times r}$ denote the unknown word embeddings from the proxy (source) domain, and $U_g^*\in\mathbb{R}^{d\times r}$ denote the unknown word embeddings from the gold (target) domain. Our goal is to use data from both domains to estimate $U_g^*$ (up to rotations). In particular, we are given proxy data $\mathcal{A}_p:\mathbb{R}^{d\times d}\to\mathbb{R}^{n_p}$ and $X_p\in\mathbb{R}^{n_p}$ from the source domain, along with gold data $\mathcal{A}_g:\mathbb{R}^{d\times d}\to\mathbb{R}^{n_g}$ and $X_g\in\mathbb{R}^{n_g}$ from the target domain, such that
\begin{align*}
X_p & = \mathcal{A}_p(\Theta_p^*) + \epsilon_p \quad \text{and} \quad X_g = \mathcal{A}_g(\Theta_g^*) + \epsilon_g,
\end{align*}
where $\epsilon_p\in\mathbb{R}^{n_p}$ and $\epsilon_g\in\mathbb{R}^{n_g}$ are independent $\sigma_p$- and $\sigma_g$-subgaussian random variables respectively.

We are interested in the setting where $(n_g/\sigma_g^2)\ll(n_p/\sigma_p^2)$. As we will discuss later, this regime holds when we have limited domain-specific data but a large text corpus from other domains.

\textbf{Group-Sparse Structure.}
To enable transfer learning, we must assume some relationship between the proxy and gold domains. Motivated by our previous discussion, we assume that the bias term
\[
\Delta_U^* = U_g^* - U_p^*,
\]
has a row-sparse structure---i.e., most of its rows are 0. This structure arises when the embeddings of most words are preserved across domains, but a few words have a different meaning/usage in the new domains (see illustration in Figure~\ref{fig:toy_diff}). More precisely, let the index set
\begin{align*}
J = \left\{j \in [d] \Bigm\vert \|\Delta_U^{*(j, \cdot)}\| \ne 0 \right\},
\end{align*}
correspond to the set of rows with nonzero entries. The \emph{group sparsity} of $\Delta_U^*$ is $s=|J|$. Then, a high-quality estimate of $U_p^*$ (from the large text corpus) can help us recover $U_g^*$ with less data, since the sample complexity of estimating $\Delta_U^*$ (due to its sparse structure) is less than that of $U_g^*$.

Note that the row-sparse structure of $\Delta_U^*$ is preserved under orthogonal transformations that are applied to both $U_g^*$ and $U_p^*$---i.e., if $\widetilde{U}_p^*=U_p^*R$ and $\widetilde{U}_g^*=U_g^*R$ for an orthogonal matrix $R$, then $\widetilde{\Delta}_U^*=\widetilde{U}_g^*-\widetilde{U}_p^*=(U_g^*-U_p^*)R=\Delta_U^*R$ has the same group sparsity as $\Delta_U^*$.

\subsection{Estimation Procedure}\label{subsec:estproc}

Our proposed two-step transfer learning estimator is as follows:
\begin{align}\label{eq:joint_optim}
\widehat{U}_p = & \argmin_{U_p} \frac{1}{n_p} \| X_p - \mathcal{A}_p(U_p U_p^T) \|^2, \nonumber \\
\widehat{U}_g^{TL} = & \argmin_{U_g:\|U_g-\widehat{U}_p\|_{2,1} \le 2L} \frac{1}{n_g} \| X_g - \mathcal{A}_g(U_gU_g^T) \|^2 + \lambda\|U_g-\widehat{U}_p\|_{2,1}.
\end{align}
The first step estimates the proxy word embeddings from a large text corpus; the second step estimates gold word embeddings from limited domain-specific data, regularizing our estimates towards the estimated proxy embeddings via a group-sparse penalty term.

As discussed earlier, our estimator aims to exploit the fact that the bias term $\Delta_U^* = U_g^* -U_p^*$ is group-sparse, and can therefore be estimated much more efficiently than $U_g^*$ itself. In particular, a simple variable transformation on \eqref{eq:joint_optim} in terms of $\Delta_U$ yields:
\begin{align}\label{eq:joint_optimtrans}
\widehat{\Delta}_U = \argmin_{\Delta_U:\|\Delta_U\|_{2,1} \le 2L} \frac{1}{n_g} \| X_g - \mathcal{A}_g((\widehat{U}_p+\Delta_U)(\widehat{U}_p+\Delta_U)^T) \|^2 + \lambda\|\Delta_U\|_{2,1},
\end{align}
where our final estimator for the gold data is $\widehat{U}_g^{TL} = \widehat{\Delta}_U + \widehat{U}_p$.
Since we have a large proxy dataset, we expect $\widehat{U}_p \approx U_p^*$; when this is the case, we will show that the second stage can efficiently debias the proxy estimator using limited gold domain-specific data. 
Since our problem is nonconvex and nonsmooth, we follow \citet{loh2015regularized} and define a compact search region for $\Delta_U$ --- i.e., $\|\Delta_U\|_{2,1} \le 2L$. Here, $L$ is a tuning parameter that should be chosen large enough to ensure feasibility, i.e., we will assume that $\|\Delta_U^*\|_{2, 1}=\|U_g^*-U_p^*\|_{2,1} \le L$. 

In (\ref{eq:joint_optim}), the regularization parameter $\lambda$ trades off bias and variance. When $\lambda \rightarrow 0$, we recover the usual low-rank estimator on gold data, which is unbiased but has high variance due to the scarcity of domain-specific data; when $\lambda \rightarrow \infty$, we simply obtain the proxy word embeddings, which have low variance but are biased due to domain mismatch. Our main result will provide a suitable value of $\lambda$ to appropriately balance the bias-variance tradeoff in this setting.

One technical challenge is that, while the group-sparse penalty in \eqref{eq:joint_optimtrans} would normally be operationalized to recover a group-sparse ``true'' parameter, this is not the case here due to estimation noise from our first stage. Specifically, the true minimizer of the (expected) low-rank objective on gold data is $U_g^*$; then, under our variable transformation $\Delta = U_g - \widehat{U}_p$, the corresponding parameter we wish to recover in \eqref{eq:joint_optimtrans} is not $\Delta_U^*$ but rather
\begin{align*}
\widetilde{\Delta}_U = \Delta_U^* - \nu,
\end{align*}
where $\nu = \widehat{U}_p - U_p^*$ is the residual noise from estimating the proxy word embeddings in the first step.
But $\widetilde{\Delta}_U$ is \textit{not} row-sparse unlike $\Delta_U^*$, since $\nu$ is not sparse. Thus, we may be concerned that the faster convergence rates promised for the group LASSO estimator may not apply here. On the other hand, we expect our estimation error $\|\nu\|$ to be small since we are in the regime where our proxy dataset is large. Thus, we expect $\widetilde{\Delta}_U$ to be \textit{approximately} row-sparse. We will prove that this is sufficient to recover $\widetilde{\Delta}_U$ (and therefore $U_g^*$) at faster rates.

\begin{remark}
The two-step design of our estimator provides significant practical benefits. In practice, training on a large text corpus can be computationally intensive, so analysts often prefer to download pre-trained word embeddings $\widehat{U}_p$; these can directly be used in the second step of our estimator, which is then trained on the much smaller domain-specific dataset. Furthermore, our approach does not require the proxy and gold datasets to be simultaneously available at training time, which is desirable in the presence of regulatory or privacy constraints.
\end{remark}

\section{Group-Sparse Transfer Learning}
\label{sec:proposed}

In this section, we state our assumptions (\S\ref{sec:assumps}), and prove sample complexity bounds for our transfer learning estimator (\S\ref{ssec:mainresult}). Then, we provide intuition for our assumptions, specifically quadratic compatibility condition (\S \ref{ssec:qcc}), discuss local minima (\S\ref{ssec:localmin}), and illustrate how our estimator can also be leveraged with typical word embedding algorithms such as GloVe (\S\ref{ssec:glove}).

\subsection{Assumptions}
\label{sec:assumps}

We first make two assumptions on the proxy and gold linear operators. Our first assumption is a standard restricted well-conditionedness (RWC) property on $\mathcal{A}_p$ from the matrix factorization literature~\citep{li2019non}, which allows us to recover high-quality estimates of the proxy word embeddings $U_p^*$. 
\begin{definition}\label{def:rwc}
A linear operator $\mathcal{A}$ satisfies the $r$-\textsf{RWC}$(\alpha, \beta)$ condition if
\[
\alpha \|Z\|_F^2 \le \frac{1}{n}\|\mathcal{A}(Z)\|^2 \le \beta \|Z\|_F^2,
\]
with $3\alpha > 2\beta$ and for any $Z \in \mathbb{R}^{d \times d}$ with $\rank(Z) \le r$. 
\end{definition}
\begin{assumption} \label{ass:proxy-rwc}
The proxy linear operator $\mathcal{A}_p$ satisfies $2r$-\textsf{RWC}$(\alpha_p, \beta_p)$.
\end{assumption}
The RWC condition ensures sufficient convexity of the loss function near $U_p^*$, and further, guarantees statistical consistency for all local minima in a nonconvex matrix factorization problem \citep{li2019non}. 
Specifically, $\alpha\|Z\|_F^2\le\frac{1}{n}\|\mathcal{A}(Z)\|^2$ ensures that the loss function has sufficient convexity to recover the low-rank matrix $\Theta_p^* = U_p^*U_p^{*T}$ consistently. This is comparable to the minimum eigenvalue condition in a linear regression problem. The rest of the definition provides sufficient smoothness in terms of $U_p^*$ so that the nonconvex matrix factorization problems have no spurious local minima---i.e., they are all global minima \citep{bhojanapalli2016global,park2017non,ge2017no}. The RWC condition is a generalization of the standard restricted isometry property (RIP) in the matrix factorization literature \citep[see, e.g.,][]{candes2005decoding}. However, RIP is very restrictive as it requires all the eigenvalues of the Hessian matrix to be within a small range of 1. 

Our first assumption is mild since we have a large proxy dataset, i.e., $n_p \gg d^2$. The degrees of freedom of a $d \times d$ matrix $Z$ of rank $r$ is $r(2d-r)$; thus, in general, we only require $n \ge r(2d-r)$ observations to achieve the lower bound in Definition~\ref{def:rwc}. For instance, when $\mathcal{A}$ is a gaussian ensemble, RIP holds with high probability when $n \gtrsim dr$ \citep{candes2011tight,recht2007guaranteed}.

\begin{remark}\label{rmk:rwc_we}
Note that the operator $\mathcal{A}$ in our word embedding model consists of basis matrices; as a result, our model has a relatively lower signal-to-noise ratio (e.g., compared to the case where $\mathcal{A}$ is from a gaussian ensemble), since $\frac{1}{n}\|\mathcal{A}(\Theta)\|^2 \approx \frac{1}{d^2}\|\Theta\|_F^2$. Therefore, when we later present our bounds for the word embedding model, we will scale the parameters $\alpha, \beta$ in the RWC assumption by $\frac{1}{d^2}$. Such a scaling is standard in the low-rank matrix literature when using basis matrix observations \citep[see, e.g.,][]{ge2016matrix}.
\end{remark}

Our second assumption is a quadratic compatibility condition (QCC) on $\mathcal{A}_g$, which allows us to recover $U_g^*$ despite our nonsmooth and quartic objective function. This condition is adapted from the standard compatibility condition in the high dimensional statistics literature \citep{van2009conditions,buhlmann2011statistics,lounici2011oracle,negahban2012unified}.
\begin{definition}[QCC] \label{def:quad_comp} A linear operator $\mathcal{A}$ satisfies the quadratic compatibility condition (\textsf{QCC}$(U^*, \kappa)$) with matrix $U^*$ and constant $\kappa$  if
\begin{align*}
\frac{1}{n}\|\mathcal{A}(\Delta U^{*T} + U^*\Delta^T +\Delta \Delta^T)\|^2 \ge \frac{\kappa}{s} \left(\sum_{j \in J}\|\Delta^{(j, \cdot)}\|\right)^2,
\end{align*}
for any $\Delta \in \mathbb{R}^{d \times r}$ that satisfies $\sum_{j \in J^c} \|\Delta^{(j, \cdot)}\| \le 7 \sum_{j \in J} \|\Delta^{(j, \cdot)}\|$.
\end{definition}
\begin{assumption} \label{assmp:quad_comp} 
The gold linear operator $\mathcal{A}_g$ satisfies \textsf{QCC}$(U_g^*, \kappa)$.
\end{assumption}
QCC imposes a much weaker convexity requirement than RWC, since RWC is unlikely to hold in the low-data regime ($n_g < dr$). Intuitively, we cannot guarantee a minimum eigenvalue condition holds for $\mathcal{A}_g$ with very few gold samples, precluding us from obtaining high-quality estimates of $\Theta_g^*$. However, we can instead impose a convexity guarantee on a \textit{restricted} subspace that contains $U_g^*-U_p^*$. The same intuition can be found in the LASSO literature \citep{van2009conditions,buhlmann2011statistics,negahban2012unified} for linear regression---in the low-data regime, we cannot impose the standard minimum eigenvalue condition on the covariance matrix, so we instead impose a compatibility condition on a restricted subspace that contains the non-sparse elements of the true parameter. Note that our QCC assumption takes a different form than the compatibility condition in group-sparse linear regression \citep{lounici2011oracle}---specifically, QCC includes an additional quadratic term $\Delta \Delta^T$ on the left hand side due to the fact that we are studying a nonconvex matrix factorization problem. We give a detailed discussion of this condition in the next subsection. 

\begin{remark}\label{rmk:qcc_we}
Analogous to Remark \ref{rmk:rwc_we}, when we later present our bounds for the word embedding model, we will scale the parameter $\kappa$ in the QCC assumption by $\frac{1}{d^2}$. Proposition \ref{prop:rsc_concenineq_we} in the next section provides support for this argument. 
\end{remark}

\subsection{Main Result} \label{ssec:mainresult}

Our main result characterizes the estimation error of our transfer learning estimator $\widehat{U}_g^{TL}$. We first introduce the following concept of smoothness of our operator $\mathcal{A}$ from \cite{chi2019nonconvex}; we obtain tighter bounds with higher smoothness, but we show that our problem always satisfies some level of smoothness (as will be made precise in Remark~\ref{rem:smoothness}).\footnote{Note that smoothness here refers to the operator $\mathcal{A}$; our objective function is not smooth due to the group-sparse penalty.}
\begin{definition} \label{def:smoothness}
A linear operator $\mathcal{A}:\mathbb{R}^{d \times d} \rightarrow \mathbb{R}^n$ satisfies the $r$-smoothness$(\beta)$ condition if for any $Z \in \mathbb{R}^{d \times d}$ with $\rank(Z) \le r$, we have that 
\[
\frac{1}{n}\|\mathcal{A}(Z)\|^2 \le \beta \|Z\|_F^2.
\]
\end{definition}

Intuitively, smoothness alleviates the nonconvexity of the problem, making it easier to identify $U_g^*$ in spite of the nonconvex loss function \citep{chi2019nonconvex}. Note that the weakest form of the assumption is when $r=1$, i.e., the upper bound is only imposed for matrices $Z$ with $\rank(Z) \le 1$. Thus, we state the following result with a $1$-smoothness assumption on the gold operator:
\begin{theorem}\label{thm:joint_estimator}
Assume $\mathcal{A}_g$ satisfies $1$-smoothness$(\beta_g)$. 
Let
\begin{align*}
\lambda = \max & \left\{ \sqrt{\frac{2048 L^2 \beta_g \sigma_g^2 }{n_g} \log(\frac{10d^2}{\delta})}, \sqrt{\frac{256 \beta_g \sigma_g^2 \sigma_1^2(U_g^{*}) }{n_g} \left(r + 2 \sqrt{r\log(\frac{5d}{\delta})} + 2 \log(\frac{5d}{\delta})\right)} \right\}.
\end{align*}
Suppose $n_p$ and $d$ are such that
$\frac{L \sigma_{r}(U_p^*) (3\alpha_p - 2\beta_p)}{8\sqrt{d}} \ge \sqrt{\frac{8\beta_p\sigma_p^2}{n_p} \left(2r(2d+1)\log(36\sqrt{2}) + \log(\frac{10}{\delta})\right)}$. Then, with probability at least $1 - \delta$, we have
\begin{align*}
\ell(\widehat{U}_g^{TL},U_g^*) 
\le & C_1 s\sqrt{\frac{\sigma_g^2 }{n_g} \log(\frac{10d^2}{\delta})} 
+ C_2 s\sqrt{\frac{\sigma_g^2}{n_g} \left(2r + 3 \log(\frac{5d}{\delta})\right)} \\
& \qquad \qquad + C_3 \sqrt{\frac{\sigma_p^2}{n_p} d\left(2r(2d+1)\log(36\sqrt{2}) + \log(\frac{10}{\delta})\right)} \\
= & \mathcal{O}\left(\sqrt{\frac{\sigma_g^2s^2(r+\log(\frac{d^2}{\delta}))}{n_g}} + \sqrt{\frac{\sigma_p^2(rd^2 + d \log(\frac{1}{\delta}))}{n_p}} \right)
\end{align*}
where 
$C_1 = \frac{16\sqrt{2048 L^2 \beta_g}}{\kappa}$, $C_2 = \frac{16\sqrt{256 \beta_g \sigma_1^2(U_g^{*})}}{\kappa}$, and $C_3 = \frac{128 \sqrt{2\beta_p}}{(3\alpha_p - 2\beta_p)\sigma_{r}(U_p^*)}$.
\end{theorem}
We provide a proof in Appendix~\ref{sec:thmjointproof}. At a high level, this result shows that our proposed estimator can accurately recover the target low-dimensional embeddings even when we have limited domain-specific data. Furthermore, it provides practical guidance on the amount of proxy and gold data required for transfer learning to be effective. In particular, the estimation error bound of our transfer learning estimator consists of two parts and depends on the gold and proxy data respectively. The first term characterizes the estimation accuracy for the bias term $\Delta_U^*$ via group-sparse penalty in the second stage of Eq.~\eqref{eq:joint_optim}. As long as the sparsity level of $\Delta_U^*$ is small (i.e., $s \ll d$), then a small gold dataset (i.e., $n_g \approx \log(d)$) suffices to accurately estimate it. The second term only depends on the proxy data, and captures the variance of estimating the proxy embeddings $U_p^*$ in the first step of \eqref{eq:joint_optim}. This term is small when the proxy dataset is large (i.e., $n_p \gg d^2$), enabling our estimator to accurately learn the portion of the structure that is shared between the source and target domains. Thus, Theorem~\ref{thm:joint_estimator} confirms that transfer learning can be effective in the ``proxy-rich and gold-scarce'' setting. Finally, the low-rank structure reduces the effective dimensionality of the problem from $d^2$ to $dr\log(d)$, ensuring fast convergence when $r\ll d$.

Theorem~\ref{thm:joint_estimator} also specifies an optimal choice of regularization parameter $\lambda$; it decreases with the number of target domain examples $n_g$ and increases with the noise term $\sigma_g$. However, the optimal choice of $\lambda$ depends on several parameters that are typically unknown; in practice, one can use cross-validation to select it instead (as we do in our experiments in \S \ref{sec:experiments}).

\begin{remark} \label{rem:smoothness}
The operator $\mathcal{A}$ naturally satisfies smoothness (Definition~\ref{def:smoothness}) as long as $\mathcal{A}$ has bounded eigenvalues. Specifically, let $\sigma_{\max}(\mathcal{A}^*\mathcal{A})$ be the maximum eigenvalue of $\mathcal{A}^*\mathcal{A}$, defined as
\begin{align*}
\sigma_{\max}(\mathcal{A}^*\mathcal{A}) & = \sup_{\|R\|_F=1} \langle R, \mathcal{A^*}(\mathcal{A}(R)) \rangle.
\end{align*}
Then $\mathcal{A}$ satisfies $r$-smoothness($\beta$) for any $\beta \le \sigma_{\max}(\frac{\mathcal{A}^*\mathcal{A}}{n})$ and $r \le d$. Thus, we can simply take $\beta_g = \sigma_{\max}(\frac{\mathcal{A}_g^*\mathcal{A}_g}{n})$ and $r=1$ to satisfy the smoothness assumption in Theorem \ref{thm:joint_estimator}.
\end{remark}

Our proof strategy differs from the standard analysis of the Burer-Monteiro method for low-rank problems~\citep{ge2017no} because our focus is on identifying group-sparse structure within a low-rank problem instead of identifying the low-rank structure itself. Furthermore, \cite{ge2017no} mainly base their analysis on the Hessian of the objective function, while the Hessian of our nonsmooth objective function \eqref{eq:joint_optimtrans} is not well-defined. Our proof adapts high-dimensional techniques for the group LASSO estimator~\citep{lounici2011oracle} to the nonconvex low-rank matrix factorization problem. Our analysis accounts for quartic (rather than the typical quadratic) dependence on the target parameter, for which we leverage QCC rather than the standard compatibility condition.

In \S \ref{sec:naive}, we contrast the error bounds for our transfer learning estimator with those we obtain on classical low-rank estimators (on just proxy or gold data), illustrating significant gains via transfer learning. 

\textbf{Word Embeddings.} Next, we examine the scaling of this bound specifically for word embedding model given in \eqref{eqn:observations_we} (described in \S\ref{sec:cla_mff}). Recall that, for word co-occurrence count data, the observation matrices $A_i$ are basis matrices, resulting in a lower signal-to-noise ratio than the more typical gaussian ensemble observation matrices studied in the general low-rank matrix factorization literature. Thus, as discussed in Remarks \ref{rmk:rwc_we}--\ref{rmk:qcc_we}, we scale the parameters in the QCC and RWC assumptions by $\frac{1}{d^2}$. However, this is counter-balanced by the fact that we have $d^2$ more samples in the word embedding setting. In particular, for a corpus of $n+1$ words ($n$ consecutive word pairs), we obtain one observation for each observed word pair and each of $d^2$ possible basis matrices $A_i=E_{jk}$ with $j,k\in[d]$ (see details in \S\ref{sec:cla_mff}). This results in $d^2n$ samples. Put together, we obtain the following result on the error of our transfer learning estimator for our word embedding model:
\begin{corollary}\label{cor:joint_estimator}
Assume $\mathcal{A}_g$ satisfies \textsf{QCC}$(U_g^*,\frac{\kappa}{d^2})$ and $1$-smoothness$(\frac{\beta_g}{d^2})$, and $\mathcal{A}_p$ satisfies $r$-\textsf{RWC}$(\frac{\alpha_p}{d^2}, \frac{\beta_p}{d^2})$. Let
\begin{align*}
\lambda = \max \left\{ \sqrt{\frac{2048L^2\beta_g\sigma_g^2}{d^4n_g} \log(\frac{10d^2}{\delta})}, \sqrt{\frac{256\beta_g\sigma_g^2\sigma_1^2(U_g^{*}) }{d^4n_g} \left(r + 2 \sqrt{r\log(\frac{5d}{\delta})} + 2 \log(\frac{5d}{\delta})\right)} \right\}.
\end{align*}
Suppose $n_p$ and $d$ are such that
$\frac{L \sigma_{r}(U_p^*) (3\alpha_p - 2\beta_p)}{8\sqrt{d}} \ge \sqrt{\frac{8\beta_p\sigma_p^2}{n_p}(2r(2d+1)\log(36\sqrt{2}) + \log(\frac{10}{\delta}))}$. Then, with probability at least $1 - \delta$, the estimate $\widehat{U}_g^{TL}$ of problem~(\ref{eqn:observations_we}) satisfies
\begin{align*}
\ell(\widehat{U}_g^{TL},U_g^*) 
= & \mathcal{O}\left(\sqrt{\frac{\sigma_g^2s^2(r+\log(\frac{d^2}{\delta}))}{n_g}} + \sqrt{\frac{\sigma_p^2(rd^2 + d \log(\frac{1}{\delta}))}{n_p}} \right).
\end{align*}
\end{corollary}
The result follows Theorem~\ref{thm:joint_estimator} directly by appropriately scaling the parameters and noting that we have $d^2n_g$ gold and $d^2n_p$ proxy observations. Note that the signal-to-noise ratio and sample sizes counter-balance each other, so the error bound in Corollary~\ref{cor:joint_estimator} is of the same scale as the general bound we obtained in Theorem~\ref{thm:joint_estimator}, despite the different setting/assumptions. This result has practical implications for the amount of proxy and gold data require to learn word embeddings specifically. In particular, it demonstrates that our transfer learning estimator only requires a small amount of domain-specific textual data (i.e., $n_g \gg \log d$) to obtain accurate domain-specific embeddings, as long as two conditions hold: (1) it has substantial domain-agnostic data such as Wikipedia text (i.e., $n_p \gg d^2$), and (2) the meanings of most words remain unchanged in the target domain (see Figure~\ref{fig:sparsity_ill}). Under these conditions, our estimator uses the large proxy corpus to accurately estimate embeddings for words that remain consistent, and uses the group-sparse penalty with the gold data to adjust the sparse set of words with different domain-specific meanings---e.g., in the finance domain, words such as ``option'' or ``strike,'' as shown in Figure~\ref{fig:sparsity_ill_fin}.

\subsection{Quadratic Compatibility Condition} \label{ssec:qcc}

We now bridge our QCC assumption (Definition~\ref{def:quad_comp}) with the more standard restricted strong convexity (RSC) condition adapted to our setting; the RSC condition is common in the high-dimensional statistics and low-rank matrix factorization literature \citep{negahban2011estimation,negahban2012restricted,negahban2012unified,klopp2014noisy}. We prove that the RSC condition holds with high probability in the low-rank matrix factorization problem for the commonly-studied case of gaussian data as well as our word co-occurrence count data.
\begin{definition}[RSC] \label{assmp:rsc_global} 
The operator $\mathcal{A}$ satisfies restricted strong convexity (\textsf{RSC}$(U^*, \eta, \tau)$) with matrix $U^*$, constant $\eta$ and function $\tau$ if
\begin{align*}
\frac{1}{n}\|\mathcal{A}(\Delta U^{*T} + U^*\Delta^T + \Delta \Delta^T)\|^2 \ge \eta\|\Delta U^{*T} + U^*\Delta^T + \Delta \Delta^T\|_F^2
- \tau(n, d, r) \|\Delta\|_{2, 1}^2
\end{align*}
for any $\Delta \in \mathcal{D} \subset \mathbb{R}^{d \times r}$.
\end{definition}
Our condition closely resembles Definition 2 of \cite{negahban2012unified}. RSC conditions are an alternative to compatibility conditions that provide a weak convexity guarantee for the problem. Indeed, without the last term on the right hand side, the RSC condition is reduced to a minimum eigenvalue condition on a specific low-rank subspace. Typically, the function $\tau$ is a small term that is model-dependent and relies on parameters such as $n, d, r$ \citep[see, e.g., Section 4 in][]{negahban2012unified}.
The following proposition shows that QCC holds given the above RSC condition when considering a bounded set of feasible $\Delta$, i.e., $\|\Delta\|_{2, 1} \le \bar{L}$ for some positive constant $\bar{L}$. Focusing on bounded $\Delta$ is not restrictive since we will formulate our transfer learning optimization problem over a compact set in the following section.
\begin{proposition}\label{thm:rsc_quadcomp}
Assume $\mathcal{A}$ satisfies \textsf{RSC}$(U^*, \eta, \tau)$ on $\mathbb{R}^{d\times r}$ and $\|U^*\|_{2, \infty} \le \frac{D}{\sqrt{d}}$ for some constant $D>0$. If $n$ and $d$ are such that 
$\frac{\eta\sigma_r^2(U^{*})}{32s} \ge 4\frac{\eta D\bar{L}}{\sqrt{d}} + \tau(n, d, r)$,
then $\mathcal{A}$ satisfies \textsf{QCC}$(U^*, \kappa)$ with $\kappa=2\eta\sigma_r^2(U^{*})$.
\end{proposition}
The proof is provided in Appendix~\ref{sec:qccdiscussion}. Note that we've imposed that the ``row-spikiness'' of the matrix $U^*$ is bounded, i.e., $\|U^*\|_{2, \infty} \le \frac{D}{\sqrt{d}}$, to ensure identifiability~\citep[see, e.g., similar assumptions in][]{agarwal2012noisy,negahban2012restricted}. In other words, $U^*$ itself is unlikely to be row-sparse.
This matches practice since individual word embeddings (rows) are never zero. Furthermore, one need not employ our transfer learning approach when $U^*$ is row-sparse, since the sample complexity of directly estimating $U^*$ is already low.

Proposition \ref{prop:rsc_concenineq} (proof in Appendix~\ref{sec:qccdiscussion}) below shows that an RSC condition holds with high probability when the linear operator $\mathcal{A}$ is sampled from a gaussian ensemble (the most commonly considered setting in the literature). To simplify notation, we define the matrix vectorization operator $\vecr: \mathbb{R}^{d_1 \times d_2} \rightarrow \mathbb{R}^{d_1d_2}$ with $\vecr(\Theta)=[\Theta^{(\cdot, 1)T}, \Theta^{(\cdot, 2)T}, \cdots, \Theta^{(\cdot, d_1)T}]^T$. Define an operator $T_{\Sigma}: \mathbb{R}^{d \times d} \rightarrow \mathbb{R}^{d \times d}$ such that $\vecr(T_{\Sigma}(\Theta)) = \sqrt{\Sigma}\vecr(\Theta)$. We still consider $\|\Delta\|_{2,1} \le \bar{L}$. 
\begin{proposition}\label{prop:rsc_concenineq}
Consider a random operator $\mathcal{A}$ sampled from a $\Sigma$-gaussian ensemble, i.e., $\vecr(A_{i}) \sim N(0, \Sigma)$. Let $\Sigma'=K^{(d,d)} \Sigma K^{(d,d)}$ with $K^{(d,d)}$ being the commutation matrix, and let
\[
\Sigma = \begin{bmatrix}
\bar{\Sigma}_{11} & \bar{\Sigma}_{12} & \cdots & \bar{\Sigma}_{1d} \\
\vdots & \vdots & \ddots & \vdots \\
\bar{\Sigma}_{d1} & \bar{\Sigma}_{d2} & \cdots & \bar{\Sigma}_{dd}
\end{bmatrix}, \quad \text{ and } \quad
\Sigma' = \begin{bmatrix}
\bar{\Sigma}'_{11} & \bar{\Sigma}'_{12} & \cdots & \bar{\Sigma}'_{1d} \\
\vdots & \vdots & \ddots & \vdots \\
\bar{\Sigma}'_{d1} & \bar{\Sigma}'_{d2} & \cdots & \bar{\Sigma}'_{dd}
\end{bmatrix},
\]
with $\bar{\Sigma}_{ij} \in \mathbb{R}^{d \times d}$ the covariance matrix of the $i^{\text{th}}$ and $j^{\text{th}}$ columns of $A_i$. Then, with probability greater than $1 - c \exp(-c'n)$ for some constants $c, c'>0$, we have for any $\Delta$, 
\begin{align*}
\frac{\|\mathcal{A}(\Delta U^{*T} + U^*\Delta^T + \Delta \Delta^T)\|}{\sqrt{n}} \ge \frac{1}{4} \|T_{\Sigma}(\Delta U^{*T} + U^*\Delta^T + \Delta \Delta^T)\|_F - 3C_6 \left(\sqrt{\frac{r}{n}} + \frac{3}{2}\sqrt{\frac{\log d}{n}}\right) \|\Delta\|_{2, 1},
\end{align*}
where
$C_6 = 2\bar{L}\max_{i \in [d^2]}\sqrt{\Sigma^{(i,i)}} + \sigma_1(U^*) \left(\max_{i \in [d]}\sqrt{\sigma_1(\bar{\Sigma}_{ii})} + \max_{i \in [d]}\sqrt{\sigma_1(\bar{\Sigma}'_{ii})}\right)$.
\end{proposition}

We now move to our word embedding model, where our observation matrices $A_i$ are basis matrices. The next result reinforces our claim that a similar RSC condition holds with high probability in this setting.
Recall that we encode a randomly sampled word pair $(j, k)$ in a basis matrix $A_i \in \mathbb{R}^{d\times d}$, whose $(j, k)$ entry equals 1 and 0 otherwise --- i.e., $A_i=E_{jk}$ where $E_{jk}$ is the basis matrix with entry $(j, k)$ being 1 and 0 otherwise. Thus, we consider the linear operator $\mathcal{A}$ being sampled from a standard weighted sampling distribution $\Pi=\{\pi_{jk}\}_{j,k\in[d]}$ with bounded $\pi_{jk}$, where $\pi_{jk}=\mathbb{P}(A_i = E_{jk})$; a similar sampling distribution is considered in prior work, e.g., \cite{klopp2014noisy} and \cite{negahban2012restricted}. Define the $L_2(\Pi)$ norm of a matrix $\Theta$ as $\|\Theta\|_{L_2(\Pi)}^2 = \mathbb{E}[\langle A_i, \Theta\rangle^2]$. 
\begin{proposition}\label{prop:rsc_concenineq_we}
Consider a random operator $\mathcal{A}$ sampled from a weighted sampling ensemble $\Pi=\{\pi_{jk}\}_{j,k\in[d]}$ with $\frac{\mu_1}{d^2}\le\pi_{jk}\le\frac{\mu_2}{d^2}$ for some constant $\mu_1,\mu_2$. Then, with probability greater than $1-c\exp(\frac{c'}{B^4}n)$ for some constants $c,c'>0$, we have for any $\Delta\in\{\Delta \mid \frac{\|\Delta U^{*T} + U^*\Delta^T + \Delta \Delta^T\|_\infty}{\|\Delta U^{*T} + U^*\Delta^T + \Delta \Delta^T\|_{L_2(\Pi)}}\le B\}$, 
\begin{align*}
\frac{1}{n}\|\mathcal{A}(\Delta U^{*T} + U^*\Delta^T + \Delta \Delta^T)\|^2 \ge \frac{\mu_1}{4d^2}\|\Delta U^{*T} + U^*\Delta^T + \Delta \Delta^T\|_F^2
- 36C_7^2\left(\sqrt{\frac{\log d}{nd^2}}+\frac{\log d}{n}\right)^2\|\Delta\|_{2,1}^2,
\end{align*}
where $C_7 = 88B(\bar{L}(\sqrt{2\mu_2}+4/3) + 2\sigma_1(U^*)(\sqrt{4r\mu_2}+8/3))$, and $B$ is a positive constant.
\end{proposition}
We give a proof in Appendix~\ref{sec:qccdiscussion}. Note that, by construction of the observation matrices in our word embedding model (described in \S\ref{sec:cla_mff}), each $\pi_{jk} = 1/d^2$ so $\mu_1,\mu_2 = 1$ in Proposition~\ref{prop:rsc_concenineq_we} above. This is because, for a given pair $(j,k)$, we draw exactly $n$ samples, among which each sample $i$ has the observation matrix $A_i=E_{jk}$. Therefore, out of the total $d^2n$ samples, we can find $A_i$ takes value $E_{jk}$ with equal probability for all $d^2$ pairs of $(j, k)$; that is, $\pi_{jk}=\mathbb{P}(A_i = E_{jk}) = 1/d^2$.

As discussed earlier, due to the low signal-to-noise ratio of the operator $\mathcal{A}$ in the word embedding setting, the RSC condition (as well as the QCC condition, by Proposition \ref{thm:rsc_quadcomp}) hold with parameters scaling as $\mathcal{O}(\frac{1}{d^2})$. Therefore, when we discuss our bounds for the word embedding problem (Corollary \ref{cor:joint_estimator}--\ref{cor:proxy_estimator}), we will scale the parameters by $\frac{1}{d^2}$.

\subsection{Local Minima}\label{ssec:localmin}

An important practical consideration is that the nonconvexity of the optimization problem in \eqref{eq:joint_optim} may result in our algorithm converging to a local rather than global minimum. Characterizing these local minima is important to ensure that our estimator is computationally tractable in practice. The RSC condition in Definition \ref{assmp:rsc_global} (or equivalently QCC in Definition \ref{def:quad_comp}) holds for the global minimum but may not apply to local minima; to that end, \cite{loh2015regularized} propose an alternative restricted strong convexity condition for nonconvex problems, enabling them to show that the resulting local minima are within statistical precision of the global minimum. We build on this last approach, adapting to our loss function:
\begin{align}\label{eq:loss_joint}
f(\Delta_U) & = \frac{1}{n_g} \| X_g - \mathcal{A}_g((\widehat{U}_p+\Delta_U)(\widehat{U}_p+\Delta_U)^T) \|^2. 
\end{align}
In particular, we introduce the following weaker restricted strong convexity condition (that we term LRSC) directly to the loss function (\ref{eq:loss_joint}) that is more likely to hold for the optimization landscape of all local minima, yielding non-asymptotic bounds for local minima.

\begin{assumption}[LRSC] \label{assmp:rsc_local} 
The loss function in \eqref{eq:loss_joint} satisfies the following restricted strong convexity with constant $\eta_1, \eta_2$ and functions $\tau_1(n, d, r), \tau_2(n, d, r)$:
\begin{subnumcases}{\mathbb{E}_{X_g\mid \mathcal{A}_g}[\langle \nabla f(\widetilde{\Delta}_U + \Delta) - \nabla f(\widetilde{\Delta}_U), \Delta \rangle]
\ge}
\eta_1 \|\Delta\|_F^2 - \tau_1(n_g, d, r) \|\Delta\|_{2, 1}^2, & $\forall \|\Delta\|_F \le \rho$, \label{eq:rsc_local_1} \\
\eta_2 \|\Delta\|_F - \tau_2(n_g, d, r) \|\Delta\|_{2, 1}, & $\forall \|\Delta\|_F \ge \rho$. \label{eq:rsc_local_2}
\end{subnumcases}
for any $\Delta \in \mathbb{R}^{d \times r}$ and some constant $\rho>0$.
\end{assumption}

Our LRSC condition provides a lower bound on the \textit{expected} Hessian of the loss function in \eqref{eq:loss_joint}, conditioned on a fixed design, where the expectation is taken over the randomness of the noise terms. Note that the LRSC condition we propose is weaker than the original RSC condition proposed in \cite{loh2015regularized}, which lower bounds the \textit{realized} Hessian directly. This is because the usual problem formulation is quadratic in the target parameter, and thus the Hessian is a deterministic quantity given a fixed design. In contrast, our transfer learning objective induces a quartic dependence on the target parameter $\Delta_U$, and thus our Hessian is a random variable that depends on the realized noise terms, introducing additional complexity. 

Intuitively, the LRSC condition serves a similar function as our earlier RSC condition (Definition \ref{assmp:rsc_global}), imposing restricted weak convexity on our loss so that we can recover high-quality estimates of the gold embeddings $U_g^*$. 
We now show that the LRSC (for local minima) is weaker than the RSC (for the global minimum) we used in Theorem~\ref{thm:joint_estimator}. Note that LRSC is composed of two separate statements; condition~(\ref{eq:rsc_local_1}) restricts the geometry locally around the global minimum, and condition~(\ref{eq:rsc_local_2}) provides a lower bound for parameters that are well-separated from the global minimum. 
First, the following Proposition \ref{thm:rsc_local} shows that condition~(\ref{eq:rsc_local_1}) is equivalent to the more traditional RSC condition for convex problems (Definition~\ref{assmp:rsc_global}) in a neighborhood of the global minimum. Next, Lemmas 8--9 in \cite{loh2015regularized} show that (\ref{eq:rsc_local_1}) usually implies (\ref{eq:rsc_local_2}), given that the function $\tau_2(n,d,r)$ in (\ref{eq:rsc_local_2}) typically scales as $\mathcal{O}(\sqrt{\tau_1(n, d, r)})$.
\begin{proposition}\label{thm:rsc_local}
When $\|\Delta\|_F \le \rho$, (i) for any $\mathcal{A}_g$ that satisfies \textsf{RSC}$(\sqrt{\frac{2}{3}}U_g^*, \eta, \tau)$ and $r$-smoothness($\beta_g$) with $9\eta \ge \beta_g$, condition~(\ref{eq:rsc_local_1}) holds with $\rho \le \sigma_r(U_g^*)/3$, $\eta_1 = 4\eta \sigma_r(U_g^*)^2$ and $\tau_1 = 3\tau/2$; (ii) for any loss function that satisfies condition~(\ref{eq:rsc_local_1}), $\mathcal{A}_g$ satisfies \textsf{RSC}$(\sqrt{\frac{2}{3}}U_g^*, \eta, \tau)$ with $\eta = \frac{\eta_1}{2(2\sigma_1(U_g^*) + 3\rho/2)^2}$ and $\tau=\tau_1/3$.
\end{proposition}
The proof is provided in Appendix~\ref{sec:localminima}. The following theorem shows that LRSC ensures all local minima are within statistical precision of the true parameter. 
\begin{theorem}\label{thm:joint_estimator_local}
Assume LRSC holds for loss function $f$ in (\ref{eq:loss_joint}) and $\mathcal{A}_g$ satisfies $1$-smoothness($\beta_g$). Let
\begin{multline*}
\lambda = \max \left\{ \sqrt{\frac{32768 L^2 \beta_g \sigma_g^2 }{n_g} \log(\frac{10d^2}{\delta})}, \sqrt{\frac{512 \beta_g \sigma_g^2 \sigma_1^2(U_g^{*}) }{n_g} \left(r + 2 \sqrt{r\log(\frac{5d}{\delta})} + 2 \log(\frac{5d}{\delta})\right)}, \right.\\
\left. \frac{4}{3}\tau_2(n_g,d,r), 16 L\tau_1(n_g, d, r)\right\}.
\end{multline*}
Suppose $n_p$ and $d$ are such that $\frac{L \sigma_{r}(U_p^*) (3\alpha_p - 2\beta_p)}{8\sqrt{d}} \ge \sqrt{\frac{8\beta_p\sigma_p^2}{n_p} (2r(2d+1)\log(36\sqrt{2}) + \log(\frac{10}{\delta}))}$,
and $n_g$ and $d$ are such that $\lambda \le \rho\eta_2/(8L)$. Then, any local minimum $\widehat{U}_g^{TL}$ satisfies
\[
\ell(\widehat{U}_g^{TL},U_g^*) =
\mathcal{O}\left(\sqrt{\frac{\sigma_g^2s^2(r+\log(\frac{d^2}{\delta}))}{n_g}} + \sqrt{\frac{\sigma_p^2(rd^2 + d \log(\frac{1}{\delta}))}{n_p}} \right)
\]
with probability at least $1 - \delta$.
\end{theorem}
We provide a proof in Appendix~\ref{sec:localminima}. In particular, the above estimation error bound for all local minima has the same scale as Theorem~\ref{thm:joint_estimator} for the global minimum. This result establishes that every local minimum of our nonconvex problem is statistically indistinguishable from the global minimum. Intuitively, the optimization landscape is benign despite our nonconvex and nonsmooth objective arising from our combination of matrix factorization and the group-sparse penalty. This is because the loss function satisfies a relaxed restricted strong convexity condition, which guarantees local curvature around all minima. Practically speaking, this result shows that our estimator is computationally tractable (in addition to being statistically efficient, as we previously showed in Theorem \ref{thm:joint_estimator}), avoiding concerns about converging to spurious local minima.

\subsection{Transfer Learning with GloVe} \label{ssec:glove}

Our transfer learning approach extends straightforwardly to nonlinear loss functions such as GloVe~\citep{pennington2014glove}, a state-of-the-art technique often used to construct word embeddings in practice.
The original GloVe method solves the following optimization problem:
\begin{align}\label{eq:glove_ori}
\min_{U_i, V_j, b_i, c_j} \sum_{i, j \in [d]} f(Y_{ij}) (\log(Y_{ij}) - (U_i V_j^T + b_i + c_j))^2,
\end{align}
where $d$ is the number of unique words, $Y_{ij}$ is the total number of co-occurrences of word pair $(i,j)$, and $\{U_i\}_{i\in[d]}$ and $\{V_j\}_{j\in[d]}$ are two sets of word embeddings (one typically takes the sum of the two $U_i+V_i$ as the final word embedding for word $i$ in a post-processing step). $\{b_i\}$ and $\{c_j \} \in \mathbb{R}$ are bias terms (tuning parameters) designed to improve fit. Finally, $f(x)$ is a non-decreasing weighting function defined as 
\begin{numcases}{f(x)=}
(x/x_{\max})^{\alpha}, & \text{if}~ $x<x_{\max}$, \nonumber \\
1, & \text{otherwise}. \nonumber
\end{numcases}
\cite{pennington2014glove} set the tuning parameters above to be $x_{\max}=100$ and $\alpha=3/4$.

We first show that our model~(\ref{eqn:observations_we}) includes a linear version of GloVe as a special case. Define the index set $I_{jk} = \{i \in [d^2n] \mid A_i = E_{jk}\}$, where remember $E_{jk}$ is a basis matrix with entry $(j, k)$ being 1 and 0 otherwise. Taking the average of (\ref{eqn:observations_we}) over the set $I_{jk}$, we have
\begin{align*}
\frac{1}{|I_{jk}|}\sum_{i \in I_{jk}} X_i = \left\langle \frac{1}{|I_{jk}|}\sum_{i \in I_{jk}} A_i, \Theta^*\right\rangle + \frac{1}{|I_{jk}|}\sum_{i \in I_{jk}}\epsilon_i = \Theta^{*(j,k)} + \frac{1}{|I_{jk}|}\sum_{i \in I_{jk}}\epsilon_i.
\end{align*}
In other words, we can create a sample word co-occurrence matrix as an empirical estimate of $\Theta^*$; factorizing this provides an estimate of $U^*$.
GloVe then deviates from our linear model by taking the logarithm of $Y_{jk}=\sum_{i \in I_{jk}} X_i$, adding bias terms for extra model complexity, and weighting up frequent word pairs through $f$. Moreover, it implements alternating-minimization with asymmetric factorization to speed up optimization; recall that GloVe takes the sum $U_i + V_i$ to obtain the word embedding for word $i$. To leverage our transfer learning approach, we can simply add an analogous group LASSO penalty to this objective:
\begin{align}\label{eq:glovejoint_optim}
\min_{U_i, V_j, b_i, c_j} \sum_{i, j \in [d]} f(Y_{ij}) (\log(Y_{ij}) - (U_i V_j^T + b_i + c_j))^2 + \lambda \sum_{i \in [d]} \| (U^i + V^i) - \widehat{U}_p^i\|,
\end{align}
where $\widehat{U}_p$ is a matrix of pre-trained (proxy) word embeddings. We also evaluate this approach empirically in \S \ref{sec:exp_wiki} on real datasets.

\section{Comparing Error Bounds}
\label{sec:naive}

In this section, we assess the value of transfer learning by comparing to the bounds we obtain if we trained our embeddings on only gold or proxy data. 

\subsection{Gold Estimator}\label{sec:goldest}

A natural unbiased approach to learning domain-specific embeddings $U_g^*$ is to apply the Burer-Monteiro approach to only gold data:
\begin{equation}\label{eq:gold_optim}
\widehat{U}_g = \argmin_{U_g} \frac{1}{n_g} \| X_g - \mathcal{A}_g(U_g U_g^T) \|^2.
\end{equation}
We follow the approach of \cite{ge2017no} to obtain error bounds on $\widehat{U}_g$ under the following standard regularity assumption:
\begin{assumption} \label{ass:gold}
The gold linear operator $\mathcal{A}_g$ satisfies $2r$-\textsf{RWC}$(\alpha_g, \beta_g)$.
\end{assumption}
Note that Assumption \ref{ass:gold} may not hold in our regime of interest where $n_g \ll d$. As discussed in \S\ref{sec:assumps}, in general, we need $n \gtrsim dr$ observations to satisfy RWC, and so the gold estimator may not satisfy any nontrivial guarantees under our data-scarce setting. In contrast, our QCC (Assumption~\ref{assmp:quad_comp}) is mild and holds in the high-dimensional setting when $n_g \gg \log(d)$. For the purposes of comparison, we examine the conventional error bounds for problem (\ref{eq:gold_optim}) under RWC.
\begin{theorem}\label{thm:gold_estimator}
The estimation error of the gold estimator has
\begin{align*}
\ell(\widehat{U}_g,U_g^*)
& \le C_4 \sqrt{\frac{\sigma_g^2 d(2r(2d+1)\log(36\sqrt{2}) + \log(\frac{2}{\delta}))}{n_g}} \\
& = \mathcal{O}\left(\sqrt{\frac{\sigma_g^2(rd^2 + d \log(\frac{1}{\delta}))}{n_g}}\right)
\end{align*}
with probability at least $1 - \delta$,
where $C_4 = \frac{16 \sqrt{2\beta_g}}{(3\alpha_g - 2\beta_g)\sigma_{r}(U_g^*)}$.
\end{theorem}
We give a proof in Appendix~\ref{sec:thmgoldproof}. Theorem~\ref{thm:gold_estimator} shows that when we have sufficient gold samples (i.e., $n_g \gg d^2$), the gold estimator achieves estimation error scaling as $\mathcal{O}(\sqrt{d^2/n_g})$. However, when $n_g \lesssim d^2$, the gold estimator has very high variance, resulting in substantial estimation error. 

Next we apply this result to our word embedding model as described in \S\ref{ssec:mainresult}.
\begin{corollary}\label{cor:gold_estimator}
Assume $\mathcal{A}_g$ satisfies $r$-\textsf{RWC}$(\frac{\alpha_g}{d^2}, \frac{\beta_g}{d^2})$. Then, with probability at least $1 - \delta$, the estimation error of the gold estimator of problem~(\ref{eqn:observations_we}) satisfies
\begin{align*}
\ell(\widehat{U}_g,U_g^*)
& = \mathcal{O}\left(\sqrt{\frac{\sigma_g^2(rd^2 + d \log(\frac{1}{\delta}))}{n_g}}\right).
\end{align*}
\end{corollary}
This result directly follows Theorem~\ref{thm:gold_estimator} by scaling the parameters in our assumptions by $1/d^2$ and expanding the number of samples to $d^2n_g$ (see analogous discussion under Corollary~\ref{cor:joint_estimator}). Indeed, Corollary \ref{cor:gold_estimator} shares a similar scaling and insight as Theorem \ref{thm:gold_estimator}.

\subsection{Proxy Estimator}\label{sec:proxyest}

An alternative approach is to estimate domain-agnostic word embeddings $U_p^*$ from the proxy data, and ignore the domain-specific bias $\Delta_U^*$:
\begin{equation}\label{eq:proxy_optim}
\widehat{U}_p = \argmin_{U_p} \frac{1}{n_p} \| X_p - \mathcal{A}_p(U_p U_p^T) \|^2.
\end{equation}
This corresponds to the common practice of using pre-trained word embeddings. Recall that we have already made the RWC assumption for $\mathcal{A}_p$ in Assumption~\ref{ass:proxy-rwc}.
\begin{theorem}\label{thm:proxy_estimator}
The estimation error of the proxy estimator has
\begin{align*}
\ell(\widehat{U}_p,U_g^*)
& \le \|\Delta_U^*\|_{2,1} + \omega + C_5 \sqrt{\frac{\sigma_p^2 d(2r(2d+1)\log(36\sqrt{2}) + \log(\frac{2}{\delta}))}{n_p}} \\
& = \mathcal{O}\left(\|\Delta_U^*\|_{2,1} + \omega + \sqrt{\frac{\sigma_p^2(rd^2 + d \log(\frac{1}{\delta}))}{n_p}}\right)
\end{align*}
with probability at least $1 - \delta$, where $\omega = \|U_p^* (R_{(\widehat{U}_p, U_p^*)} - R_{(\widehat{U}_p, U_g^*)})\|_{2,1}$ and $C_5 = \frac{16 \sqrt{2\beta_p}}{(3\alpha_p - 2\beta_p)\sigma_{r}(U_p^*)}$.
\end{theorem}
We give a proof in Appendix~\ref{sec:thmproxyproof}, following the approach of \cite{ge2017no}. However, as discussed in \S\ref{sec:cla_mff}, recall that $U^*$ is only identifiable only up to an orthogonal change-of-basis, so we consider the rotation $R_{(\widehat{U},U^*)}$ that best aligns $\widehat{U}$ with the true parameter $U^*$. Therefore, to compare $\widehat{U}_p$ with the true gold word embeddings $U_g^*$, we use the rotation $R_{(\widehat{U}_p,U_g^*)}$. Yet, $\widehat{U}_p$ is best aligned with $U_p^*$ under a different rotation $R_{(\widehat{U}_p,U_p^*)}$. The choice of rotation affects the error from the group-sparse bias term $\Delta_U^* = U_g^*-U_p^*$, resulting in a term $\omega$ accounting for the misalignment between the two rotations $R_{(\widehat{U}_p,U_g^*)}$ and $R_{(\widehat{U}_p,U_p^*)}$ in Theorem~\ref{thm:proxy_estimator}. 

Since we $n_p$ is large in our regime of interest, the third term in the estimation error bound (capturing the error of $\widehat{U}_p-U_p^*$) is small, scaling as $\mathcal{O}(d/\sqrt{n_p})$. Instead, the first two terms capturing the bias between $U_p^*$ and $U_g^*$ dominate the estimation error. Note that when $\Delta_U^* \rightarrow 0$, we have $R_{(\widehat{U}_p,U_g^*)} \rightarrow R_{(\widehat{U}_p,U_p^*)}$. Thus, when there are few domain-specific differences between the gold and proxy data, the proxy estimator can be more accurate than the gold estimator. 

Next we apply this result to our word embedding model as described in \S\ref{ssec:mainresult}.
\begin{corollary}\label{cor:proxy_estimator}
Assume $\mathcal{A}_p$ satisfies $r$-\textsf{RWC}$(\frac{\alpha_p}{d^2}, \frac{\beta_p}{d^2})$. Then, with probability at least $1 - \delta$, using $\omega$ specified in Theorem~\ref{thm:proxy_estimator}, the estimation error of the proxy estimator of problem~(\ref{eqn:observations_we}) satisfies
\begin{align*}
\ell(\widehat{U}_p,U_g^*)
& = \mathcal{O}\left(\|\Delta_U^*\|_{2,1} + \omega + \sqrt{\frac{\sigma_p^2(rd^2 + d \log(\frac{1}{\delta}))}{n_p}}\right).
\end{align*}
\end{corollary}
This result directly follows Theorem~\ref{thm:proxy_estimator} by scaling the parameters in our assumptions by $1/d^2$ and expanding the number of samples to $d^2n_p$ (see analogous discussion under Corollary~\ref{cor:joint_estimator}). Indeed, Corollary \ref{cor:proxy_estimator} shares a similar scaling and insight as Theorem \ref{thm:proxy_estimator}.

\subsection{Comparison of Error Bounds} \label{ssec:comparison}

\begin{table}[htbp]
\begin{center}
\begin{tabular}{lccc}
\toprule
\textbf{Estimator} & TL & Gold & Proxy \\
\midrule
\textbf{Error Bound} 
& $\mathcal{O}\left( \sqrt{\frac{s^2r}{n_g}} + \sqrt{\frac{rd^2}{n_p}} \right)$ 
& $\mathcal{O}\left(\sqrt{\frac{rd^2}{n_g}}\right)$ 
& $\mathcal{O}\left(\|\Delta_U^*\|_{2,1} + \omega + \sqrt{\frac{rd^2}{n_p}}\right)$ \\
\bottomrule
\end{tabular}
\end{center}
\caption{Error bound for the transfer learning, gold and proxy estimators. $\omega$ is defined in Theorem~\ref{thm:proxy_estimator}. The error bounds for word embeddings are the same.}
\label{tab:upperbound_comp}
\end{table}

We now summarize and compare the estimation error bounds we have derived so far in Table \ref{tab:upperbound_comp}. We first consider the general low-rank matrix factorization environment. In the regime of interest---i.e., lots of proxy data ($n_p \gg d^2$) but limited gold data ($n_g \ll d^2$)---the upper bound of our transfer learning estimator is much smaller than the conventional scaling of error bounds applied to the gold or proxy data alone. In particular, when our $n_p$ is sufficiently large, our error bound scales as $\sqrt{\log d/n_g}$ whereas the gold error bound scales as $\sqrt{d^2/n_g}$, i.e., transfer learning yields a significant improvement in the vocabulary size $d$ (recall that $s,r\ll d$). On the other hand, the proxy error bound is dominated by the size of the domain bias term $\|\Delta_U^*\|_{2,1}$, implying that it never recovers the true gold word embeddings $U_g^*$. In contrast, transfer learning can leverage limited gold data to efficiently estimate $U_g^*$ by recovering the bias between $U_g^*$ and $U_p^*$ based on a sufficiently good estimate of $U_p^*$. Note that Corollary~\ref{cor:joint_estimator}-\ref{cor:proxy_estimator} for the word embedding model yield the same error bounds, and thus Table~\ref{tab:upperbound_comp} applies to the word embedding setting as well.

\section{Experiments} \label{sec:experiments}

We evaluate our approach on both synthetic data and real Wikipedia data. We compare to both the gold and proxy estimators; on Wikipedia data, we additionally compare to state-of-the-art fine-tuning heuristics. We present our main results here, relegating setup details to Appendix~\ref{sec:add_exp} and additional supporting results to Appendix~\ref{sec:exp_ex}.

\subsection{Comparison to Gold and Proxy Estimators on Synthetic Data}
\label{ssec:synth_data}

First, we evaluate the performance of our approach at recovering the ground truth parameters on synthetic data, comparing it to the gold and proxy estimators.

\paragraph{Setup.}

We generate synthetic data with abundant proxy data ($n_p=5,000$), limited gold data $(n_g=50)$, and a sparse number $s \ll d$ of words with altered embeddings. We consider various values of the vocabulary size $d$, sparsity $s$, and rank $r$. Recall that any estimator $\widehat{U}$ of the embeddings $U_g^*$ are invariant to rotations (see discussion in \S\ref{sec:cla_mff}), so we evaluate the estimation error of $\widehat{U}$ using the rotation-invariant Frobenius norm error for $\Theta_g^*$, i.e., $\|\Theta_g^*-\widehat{U}\widehat{U}^T\|_F$. Details on data generation and hyperparameter selection (through cross-validation) are in Appendix~\ref{app:exp-synth}.

\begin{figure}[htbp]
\centering
\begin{subfigure}[b]{0.32\textwidth}
  \centering
  \includegraphics[width=\textwidth]{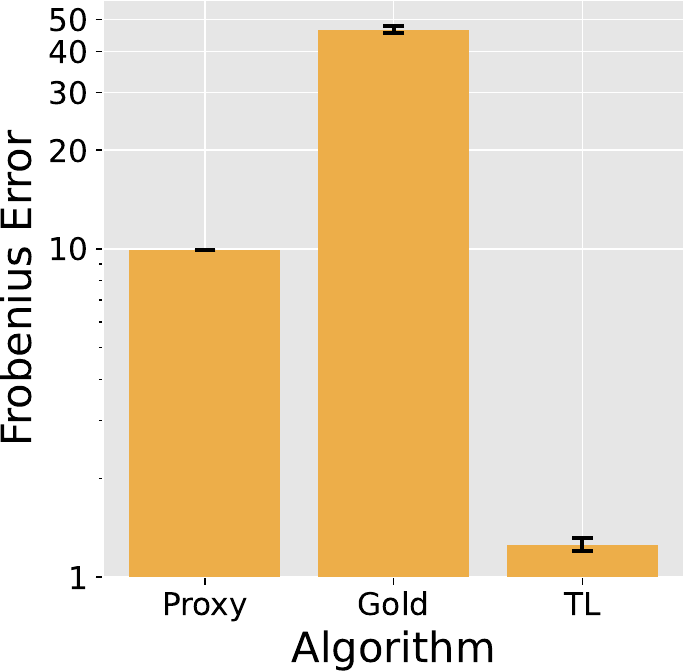}
  \caption{$d=20, r=5, s=2$}
\end{subfigure}
\begin{subfigure}[b]{0.32\textwidth}
  \centering
  \includegraphics[width=\textwidth]{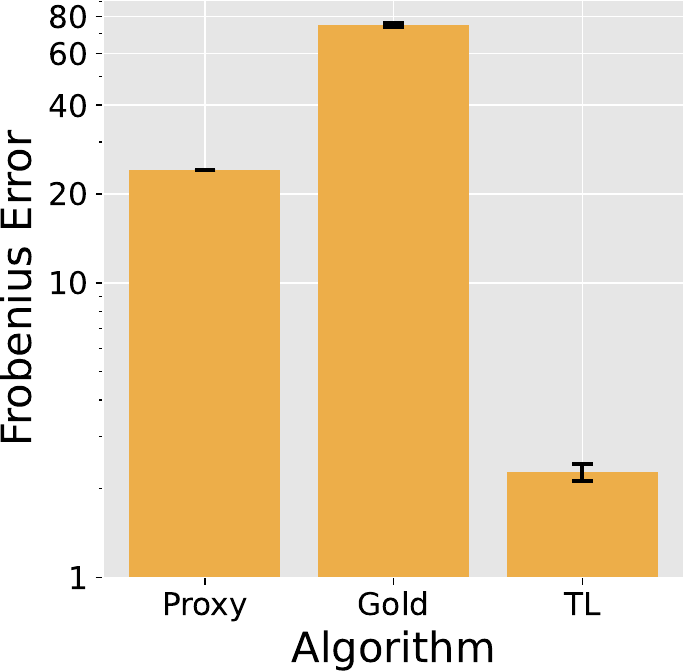}
  \caption{$d=20, r=10, s=2$}
\end{subfigure} \\
\begin{subfigure}[b]{0.32\textwidth}
  \centering
  \includegraphics[width=\textwidth]{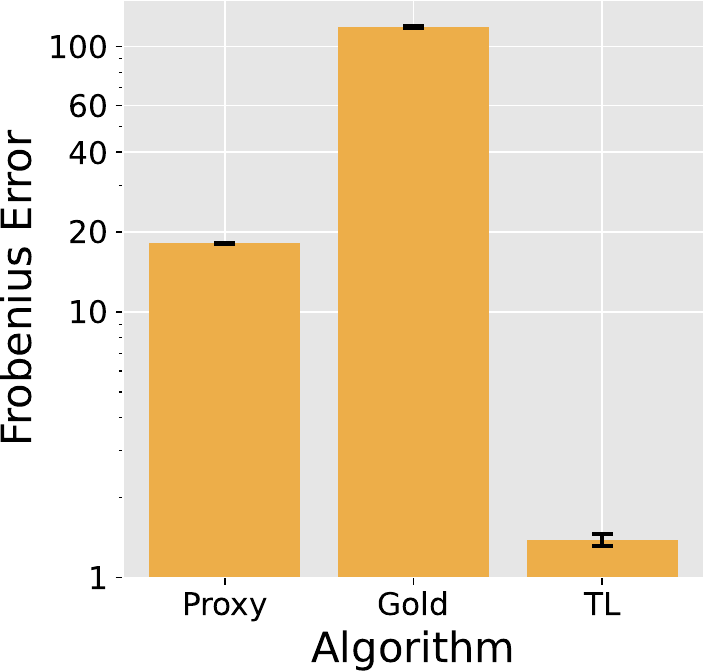}
  \caption{$d=40, r=5, s=2$}
\end{subfigure}
\begin{subfigure}[b]{0.32\textwidth}
  \centering
  \includegraphics[width=\textwidth]{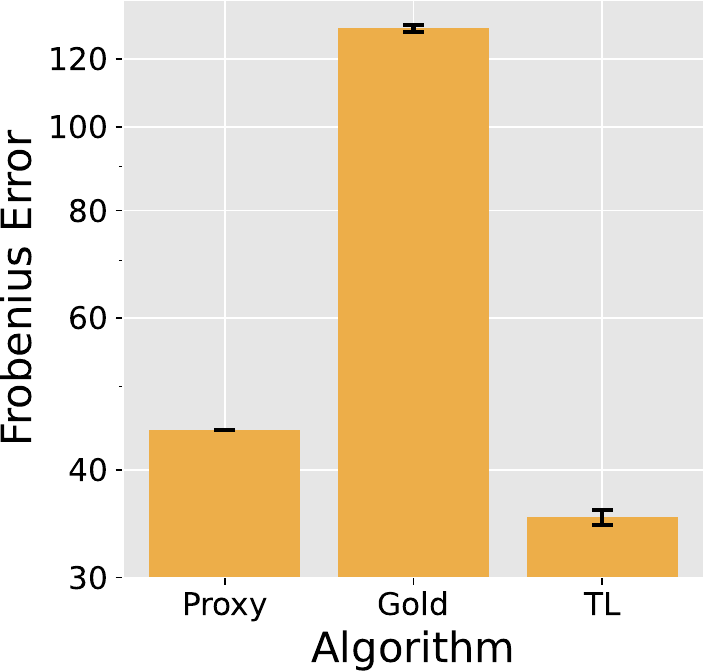}
  \caption{$d=40, r=5, s=15$}
\end{subfigure}
\caption{Bars depict Frobenius norm estimation errors of $\Theta_g$ averaged over 100 trials, with error bars the corresponding 95\% confidence intervals. `TL' represents our transfer learning estimator.}
\label{fig:synthetic_exp}
\end{figure}

\paragraph{Results.}

Figure~\ref{fig:synthetic_exp} shows the Frobenius error of our transfer learning estimator as well as the classical low-rank estimators using gold (\S\ref{sec:goldest}) or proxy (\S\ref{sec:proxyest}) data alone. Matching our theoretical results in Table~\ref{tab:upperbound_comp}, we find that our transfer learning estimator substantially outperforms the gold and proxy estimators by exploiting group-sparse structure, efficiently debiasing the proxy data with very limited gold data. The gold estimator generally performs poorly in the low-data regime, and its accuracy deteriorates with increasing model complexity (i.e., larger $d$ and $r$) as suggested by Theorem \ref{thm:gold_estimator}. The proxy estimator generally performs better than the gold estimator due to its large sample size (reflecting its popularity in practice) but performs worse than our estimator, particularly when the domain-specific bias is large (i.e., larger $s$) as suggested by Theorem~\ref{thm:proxy_estimator}.

We conduct a series of additional experiments evaluating our approach in Appendix~\ref{sec:exp_ex_sync}. At a high level, we find the following. Our transfer learning estimator performs substantially better than the gold estimator for small to moderate gold sample sizes---we improve performance even for moderate $n_g \geq d^2$, and perform comparably for large $n_g \gg d^2$. Next, we find that transfer learning becomes more challenging as the gold and proxy tasks become more heterogeneous, i.e., larger magnitude of $s, \Delta_U^*$. Finally, we test the robustness of our estimator's error to the specification of the hyperparameter $\lambda$, which is often unknown and must be estimated via cross-validation on noisy data. We find that our performance is remarkably stable with changes of even an order of magnitude in $\lambda$ in our low-data regime. Details and results are provided in Appendix~\ref{sec:exp_ex_sync}.

\subsection{Comparison to Gold and Proxy Estimators on Wikipedia Data}
\label{sec:exp_wiki}

While synthetic data is generated to meet our assumptions, this may not be the case on real data. Thus, we evaluate performance on real Wikipedia articles (a commonly used source for text data).

\paragraph{Setup.}

We cannot directly evaluate the estimation error of our embeddings since we don't have access to ground truth. Instead, we evaluate the accuracy of the word embeddings $\widehat{U}$ by measuring their out-of-sample prediction error for word co-occurrence---i.e.,
\begin{align*}
\frac{1}{|\mathcal{S}_{\text{test}}|} \sum_{i\in\mathcal{S}_{\text{test}}} (X_{g,i} - \langle A_{g,i}, \widehat{U}\widehat{U}^T\rangle)^2
\end{align*}
for a test set $\mathcal{S}_{\text{test}}\subset[n_g]$. Next, we selected 37 individual domain-specific Wikipedia articles from the following four domains: finance, math, computer science, and politics. These articles all contain words with domain-specific technical meanings that differ from their colloquial meaning, necessitating the use of transfer learning.\footnote{Specifically, the articles selected all have a domain-specific word in their title---e.g., ``put'' in the article ``put option'' (in finance), ``closed'' in ``closed set'' (in math), ``object'' in computing, and ``left'' in ``left wing politics'' (in politics). We define a word to be a domain-specific word if any of its definitions on Wiktionary is labeled with key words from that domain---i.e., ``finance'' or ``business'' for finance, ``math'', ``geometry'', ``algebra'', or ``group theory'' for math, ``computing'', ``computer'' or ``programming'' for computer science, and ``politics'' for politics.}
We leverage our transfer learning approach using the popular GloVe pre-trained word embeddings as the proxy estimator,\footnote{Our transfer learning approach aims to efficiently use publicly available pre-trained word embeddings along with domain textual data. As an alternative to pre-trained embeddings like GloVe, one can also train new proxy word embeddings on the entire Wikipedia dataset but this is highly computationally costly.} and evaluate its performance based on the out-of-sample predictive accuracy. We compare our approach with both the gold estimator and the proxy estimator. Additional experimental setup details are in Appendix~\ref{app:wiki}.

\paragraph{Results.}

Table~\ref{tb:Mse_domain_bench} shows the domain-level out-of-sample predictive errors (normalized by article length) across articles in each domain. We find that transfer learning is significantly more effective at estimating word embeddings, which is consistent with our previous results. Interestingly, in this case, the gold estimator substantially outperforms the proxy estimator, likely due to the specialized subset of articles we selected, creating a greater shift from proxy data to gold data.

\begin{table}[htbp]
\begin{center}
\SingleSpacedXI \small
\begin{tabular}{lccccc}
\toprule
\textbf{Domain} & \textbf{TL} & \textbf{Gold} & \textbf{Proxy}\\
\midrule
Finance  &  \textbf{4.6743}	& 8.7628 & 188.8744 \\ 
Math  &  \textbf{4.4005} & 8.1970 & 125.7886 \\
Computing  &  \textbf{5.0633} & 8.9493 & 162.6620 \\
Politics  &  \textbf{4.3570} & 8.2287 & 179.1600 \\
\bottomrule
\end{tabular}
\caption{Out-of-sample predictive errors of word co-occurrence (normalized by article length) averaged over 5 trials for four domains respectively. ``TL'' represents our transfer learning approach.}
\label{tb:Mse_domain_bench}
\end{center}
\end{table}

\subsection{Comparison to Existing Transfer Learning Algorithms on Wikipedia Data}

Finally, we compare our algorithm to existing transfer learning algorithms on Wikipedia data.

\paragraph{Setup.}

We compare our approach with several state-of-the-art techniques for learning domain-specific word embeddings from pre-trained word embeddings: Mittens~\citep{dingwall2018mittens}, Canonical Correlation Analysis (CCA), and a kernelized variant of CCA called KCCA~\citep{sarma2018domain}. These approaches are all heuristics that do not provide any theoretical guarantees on their performance. Furthermore, they are all based on the popular GloVe ovjective (see Eq.~\eqref{eq:glovejoint_optim} and accompanying discussion in \S\ref{ssec:glove}); as a consequence, they are not comparable with ours in the word co-occurrence prediction task used in the previous section. Instead, we measure efficacy on identifying domain-specific words for each individual Wikipedia article (i.e., to identify words that have special meaning/usage in the target domain). Intuitively, accurate estimation of domain word embeddings implies accurate estimation of domain-specific words (since domain-specific words typically exhibit larger embedding discrepancies, see our illustration in Figure~\ref{fig:toy_rowsparse}). For this task, we also compare to a completely na\"{i}ve baseline that randomly selects ``domain-specific words,'' as well as an adaptation of our transfer learning estimator to the GloVe objective. Hyperparameters for each method are tuned via 5-fold cross-validation for fair comparison. Additional experimental setup details as well as descriptions of CCA and KCCA are in Appendix~\ref{app:wiki}.

\paragraph{Results.}

Table~\ref{tb:F1_domain_bench} shows the average $F_1$ score of identifying domain-specific words (normalized by article length) across articles in each domain. While we observe that other approaches also identify domain-specific words, our approach does so more effectively, likely since our group-sparsity assumption is partly supported by these datasets (recall Figure~\ref{fig:sparsity_ill} from the introduction). Also, our transfer learning estimator and its GloVe analog perform comparably, demonstrating that our technical insights carry over naturally to off-the-shelf word embedding approaches. Next, Table~\ref{tb:sample_top10rank_article} shows the top 10 words ranked by our approach and by Mittens for one article in each domain---indeed, we observe that our approach is much more effective at identifying domain-specific words (shown in bold).
Finally, in Appendix~\ref{app:exp_ex_wiki}, we provide an additional experiment showing that the benefits of our algorithm are consistent across different thresholds that determine the criteria for a domain-specific word, illustrating that our results are robust to parameter selection.

\begin{table}[htbp]
\begin{center}
\SingleSpacedXI \small
\begin{tabular}{lcccccc}
\toprule
\textbf{Domain} & \textbf{TL} (ours) &  \textbf{GloVeTL} (ours) & \textbf{Mittens} & \textbf{CCA} & \textbf{KCCA} & \textbf{Random} \\
\midrule
Finance  & \textbf{0.2320} & 0.2282 & 0.1861 & 0.1347 & 0.1633 & 0.1376 \\ 
Math  &  \textbf{0.2660} & 0.2477 & 0.2128 & 0.2385 & 0.1690 & 0.1543 \\
Computing  &  \textbf{0.2527} & 0.2443 & 0.1861 & 0.1980 & 0.2319 & 0.1430 \\
Politics  &  0.1873 & \textbf{0.1904} & 0.1437 & 0.0602 & 0.1373 & 0.0640 \\
\bottomrule
\end{tabular}
\caption{Average $F1$ score of domain-specific word identification (normalized by article length) for four domains respectively. ``TL'' represents our transfer learning approach and ``GloVeTL'' represents our method adapted to
the GloVe objective.}
\label{tb:F1_domain_bench}
\end{center}
\end{table}

\begin{table}[htbp]
\begin{center}
\SingleSpacedXI \footnotesize
\begin{tabular}{cccccccc}
\toprule
\multicolumn{2}{c}{\textbf{Short}} & \multicolumn{2}{c}{\textbf{Prime Number}} & \multicolumn{2}{c}{\textbf{Cloud Computing}} & \multicolumn{2}{c}{\textbf{Conservatism}}  \\ 
TL & Mittens & TL & Mittens & TL & Mittens & TL & Mittens \\
\midrule
\textbf{short} & \textbf{short} & \textbf{prime} & \textbf{prime} & \textbf{cloud} & \textbf{cloud} & \textbf{party} & \textbf{party}  \\
\textbf{shares} & percent & \textbf{formula} & still & \textbf{data} & \textbf{private} & \textbf{conservative} & \textbf{conservative} \\
price & due & \textbf{numbers} & \textbf{formula} & \textbf{computing} & large & social & second \\
\textbf{stock} & public & \textbf{number} & de & \textbf{service} & information & conservatism & social \\
\textbf{security} & customers & \textbf{primes} & \textbf{numbers} & \textbf{services} & devices & government & research \\
selling & prices & \textbf{theorem} & \textbf{number} & \textbf{applications} & \textbf{applications} & \textbf{liberal} & svp \\
\textbf{securities} & high & \textbf{natural} & great & \textbf{private} & security & \textbf{conservatives} & government \\ 
\textbf{position} & hard & integers & side & users & work & political & de \\ 
may & \textbf{shares} & \textbf{theory} & way & use & \textbf{engine} & \textbf{right} & also \\ 
\textbf{margin} & price & \textbf{product} & algorithm & \textbf{software} & allows & economic & church \\ 
\bottomrule
\end{tabular}
\caption{Top 10 words, sorted by absolute change of word embedding from source to target domain. Domain-specific words (threshold set to top 10\% of the rank) are labeled in bold.}
\label{tb:sample_top10rank_article}
\end{center}
\end{table}

\section{Managerial Insights}

Our key insight stems from our theoretical analysis, which says that data analytics methods can be effective in ``small data'' settings as long as there is a substantial amount of data available from a related domain. Specifically, as long as the target domain only deviates from the related domain along a sparse number of components, then transfer learning can effectively learn both the structure in the related domain along with the sparse deviation between the source and target domains. While we have focused on text embeddings, as noted in \S \ref{ssec:other-apps}, our approach can extend to other applications as well. More broadly, our work has several managerial implications, which we discuss here.

First, our work demonstrates how integrating data from related domains can substantially reduce the data acquisition cost typically associated with data analytics. However, this integration must be done with care; simply pooling data ignores domain-specific structure that can degrade performance. Instead, effective transfer learning requires leveraging the limited gold data to learn sparse, domain-specific deviations in a structured way (in our case, at the embedding level).

Second, we demonstrate how leveraging sparse structure can not only make learning more effective, but also improve interpretability. By exposing sparse domain-specific deviations, experts can obtain insights into aspects of the target domain that differ compared to the source domain. Beyond our focus on word embeddings, experts might learn how online browsing behaviors differ from shopping decisions in physical stores, or how online social networks differ from real social networks. These deviations can drive decision-making beyond improving downstream prediction performance, e.g., where to invest in additional data infrastructure.

Together, these two factors can help accelerate the adoption of AI in decision-making. In particular, improving the effectiveness of learning and improving interpretability can help domain experts build trust in data analytics pipelines, enabling them to adopt AI in settings where reliability is paramount. Our approach points to a promising path for realizing this goal.

\section{Conclusions}

We propose a novel estimator for transferring knowledge from large text corpora to learn word embeddings in a data-scarce domain of interest. We cast this as a low-rank matrix factorization problem with a group-sparse penalty, regularizing the domain embeddings towards existing pre-trained embeddings. Under a group-sparsity assumption and standard regularity conditions, we prove that our estimator requires substantially less data to achieve the same error compared to conventional estimators that do not leverage transfer learning. Our experiments demonstrate the effectiveness of our approach in the low-data regime, both on synthetic data and on real-world textual data such as Wikipedia articles.

More broadly, our approach provides the first principled foundation for transfer learning in natural language processing models, and can potentially inform similar approaches in other models beyond low-rank matrix completion. For instance, fine-tuning remains popular among large language models (LLMs) such as BERT and GPT, and our insights could contribute to a theoretical understanding of methods that efficiently combine pre-trained LLMs with a small amount of domain data. It is worth noting that, despite the recent prominence of LLMs that generate dynamic word embeddings, static word embedding models such as GloVe and Word2Vec continue to be used in high-stakes settings where interpretability matters~\citep{bommasani2020interpreting}.

While our focus has been on learning word embeddings, unsupervised matrix factorization models have also been widely applied for other business applications (see, e.g., our discussion in \S \ref{ssec:other-apps}), such as social networks, supply chains, and recommender systems, which may open up new lines of inquiry in conjunction with our transfer learning approach. 

\ACKNOWLEDGMENT{The authors gratefully acknowledge indispensable financial support from the Wharton Analytics Initiative and the Wharton Dean's Fund.}


\bibliographystyle{ormsv080} 
{\SingleSpacedXI
\bibliography{refs} 

@inproceedings{zhou2008large,
  title={Large-scale parallel collaborative filtering for the netflix prize},
  author={Zhou, Yunhong and Wilkinson, Dennis and Schreiber, Robert and Pan, Rong},
  booktitle={International conference on algorithmic applications in management},
  pages={337--348},
  year={2008},
  organization={Springer}
}

@article{taylor2007supply,
	author = {Taylor, Terry A and Plambeck, Erica L},
	journal = {Management science},
	number = {10},
	pages = {1577--1593},
	publisher = {INFORMS},
	title = {Supply chain relationships and contracts: The impact of repeated interaction on capacity investment and procurement},
	volume = {53},
	year = {2007}}

@inproceedings{bommasani2020interpreting,
	author = {Bommasani, Rishi and Davis, Kelly and Cardie, Claire},
	booktitle = {Proceedings of the 58th Annual Meeting of the Association for Computational Linguistics},
	pages = {4758--4781},
	title = {Interpreting pretrained contextualized representations via reductions to static embeddings},
	year = {2020}}

@article{ren2010information,
	author = {Ren, Z Justin and Cohen, Morris A and Ho, Teck H and Terwiesch, Christian},
	journal = {Operations research},
	number = {1},
	pages = {81--93},
	publisher = {INFORMS},
	title = {Information sharing in a long-term supply chain relationship: The role of customer review strategy},
	volume = {58},
	year = {2010}}

@inproceedings{tang2012inferring,
	author = {Tang, Jie and Lou, Tiancheng and Kleinberg, Jon},
	booktitle = {Proceedings of the fifth ACM international conference on Web search and data mining},
	pages = {743--752},
	title = {Inferring social ties across heterogenous networks},
	year = {2012}}

@article{yang2020precise,
	author = {Yang, Fan and Zhang, Hongyang R and Wu, Sen and R{\'e}, Christopher and Su, Weijie J},
	journal = {arXiv preprint arXiv:2010.11750},
	title = {Precise high-dimensional asymptotics for quantifying heterogeneous transfers},
	year = {2020}}

@article{duan2024target,
	author = {Duan, Junting and Pelger, Markus and Xiong, Ruoxuan},
	journal = {Journal of Econometrics},
	number = {2},
	pages = {105521},
	publisher = {Elsevier},
	title = {Target PCA: Transfer learning large dimensional panel data},
	volume = {244},
	year = {2024}}

@inproceedings{ge2017no,
	author = {Ge, Rong and Jin, Chi and Zheng, Yi},
	booktitle = {International Conference on Machine Learning},
	organization = {PMLR},
	pages = {1233--1242},
	title = {No spurious local minima in nonconvex low rank problems: A unified geometric analysis},
	year = {2017}}

@article{borgs2017thy,
	author = {Borgs, Christian and Chayes, Jennifer and Lee, Christina E and Shah, Devavrat},
	journal = {Advances in neural information processing systems},
	title = {Thy friend is my friend: Iterative collaborative filtering for sparse matrix estimation},
	volume = {30},
	year = {2017}}

@article{feng2024designing,
	author = {Feng, Yifan and Caldentey, Ren{\'e} and Xin, Linwei and Zhong, Yuan and Wang, Bing and Hu, Haoyuan},
	journal = {Management Science},
	publisher = {INFORMS},
	title = {Designing sparse graphs for stochastic matching with an application to middle-mile transportation management},
	year = {2024}}

@inproceedings{shi2015community,
	author = {Shi, Xiaohua and Lu, Hongtao and He, Yangchen and He, Shan},
	booktitle = {Proceedings of the 2015 IEEE/ACM International Conference on Advances in Social Networks Analysis and Mining 2015},
	pages = {541--546},
	title = {Community detection in social network with pairwisely constrained symmetric non-negative matrix factorization},
	year = {2015}}

@article{li2017utility,
	author = {Li, Zhepeng and Fang, Xiao and Bai, Xue and Sheng, Olivia R Liu},
	journal = {Management Science},
	number = {6},
	pages = {1938--1952},
	publisher = {INFORMS},
	title = {Utility-based link recommendation for online social networks},
	volume = {63},
	year = {2017}}

@article{van2009conditions,
	author = {Van De Geer, Sara A and B{\"u}hlmann, Peter},
	title = {On the conditions used to prove oracle results for the lasso},
	year = {2009}}

@article{candes2005decoding,
	author = {Candes, Emmanuel J and Tao, Terence},
	journal = {IEEE transactions on information theory},
	number = {12},
	pages = {4203--4215},
	publisher = {IEEE},
	title = {Decoding by linear programming},
	volume = {51},
	year = {2005}}

@article{ge2016matrix,
	author = {Ge, Rong and Lee, Jason D and Ma, Tengyu},
	journal = {Advances in neural information processing systems},
	title = {Matrix completion has no spurious local minimum},
	volume = {29},
	year = {2016}}

@book{koltchinskii2011oracle,
	author = {Koltchinskii, Vladimir},
	publisher = {Springer Science \& Business Media},
	title = {Oracle inequalities in empirical risk minimization and sparse recovery problems: {\'E}cole D'{\'E}t{\'e} de Probabilit{\'e}s de Saint-Flour XXXVIII-2008},
	volume = {2033},
	year = {2011}}

@article{klopp2014noisy,
	author = {Klopp, Olga},
	title = {Noisy low-rank matrix completion with general sampling distribution},
	year = {2014}}

@article{athey2021matrix,
	author = {Athey, Susan and Bayati, Mohsen and Doudchenko, Nikolay and Imbens, Guido and Khosravi, Khashayar},
	journal = {Journal of the American Statistical Association},
	number = {536},
	pages = {1716--1730},
	publisher = {Taylor \& Francis},
	title = {Matrix completion methods for causal panel data models},
	volume = {116},
	year = {2021}}

@inproceedings{li2017learning,
	author = {Li, Quanzhi and Shah, Sameena},
	booktitle = {Proceedings of the 21st conference on computational natural language learning (CoNLL 2017)},
	pages = {301--310},
	title = {Learning stock market sentiment lexicon and sentiment-oriented word vector from stocktwits},
	year = {2017}}

@article{huang2010benefit,
	author = {Huang, Junzhou and Zhang, Tong},
	title = {The benefit of group sparsity},
	year = {2010}}

@article{chen2013reduced,
	author = {Chen, Kun and Dong, Hongbo and Chan, Kung-Sik},
	journal = {Biometrika},
	number = {4},
	pages = {901--920},
	publisher = {Oxford University Press},
	title = {Reduced rank regression via adaptive nuclear norm penalization},
	volume = {100},
	year = {2013}}

@article{cai2016structured,
	author = {Cai, Tianxi and Cai, T Tony and Zhang, Anru},
	journal = {Journal of the American Statistical Association},
	number = {514},
	pages = {621--633},
	publisher = {Taylor \& Francis},
	title = {Structured matrix completion with applications to genomic data integration},
	volume = {111},
	year = {2016}}

@inproceedings{kohavi1995study,
	author = {Kohavi, Ron and others},
	booktitle = {Ijcai},
	number = {2},
	organization = {Montreal, Canada},
	pages = {1137--1145},
	title = {A study of cross-validation and bootstrap for accuracy estimation and model selection},
	volume = {14},
	year = {1995}}

@book{hastie2009elements,
	author = {Hastie, Trevor and Tibshirani, Robert and Friedman, Jerome H and Friedman, Jerome H},
	publisher = {Springer},
	title = {The elements of statistical learning: data mining, inference, and prediction},
	volume = {2},
	year = {2009}}

@article{farias2019learning,
	author = {Farias, Vivek F and Li, Andrew A},
	journal = {Management Science},
	number = {7},
	pages = {3131--3149},
	publisher = {INFORMS},
	title = {Learning preferences with side information},
	volume = {65},
	year = {2019}}

@inproceedings{park2017non,
	author = {Park, Dohyung and Kyrillidis, Anastasios and Carmanis, Constantine and Sanghavi, Sujay},
	booktitle = {Artificial Intelligence and Statistics},
	organization = {PMLR},
	pages = {65--74},
	title = {Non-square matrix sensing without spurious local minima via the Burer-Monteiro approach},
	year = {2017}}

@inproceedings{lipton2018detecting,
	author = {Lipton, Zachary and Wang, Yu-Xiang and Smola, Alexander},
	booktitle = {International conference on machine learning},
	organization = {PMLR},
	pages = {3122--3130},
	title = {Detecting and correcting for label shift with black box predictors},
	year = {2018}}

@inproceedings{xu2021group,
	author = {Xu, Kan and Zhao, Xuanyi and Bastani, Hamsa and Bastani, Osbert},
	booktitle = {International Conference on Machine Learning},
	organization = {PMLR},
	pages = {11603--11612},
	title = {Group-sparse matrix factorization for transfer learning of word embeddings},
	year = {2021}}

@article{ramchandani2021unmasking,
	author = {Ramchandani, Pia and Bastani, Hamsa and Wyatt, Emily},
	journal = {Available at SSRN 3866259},
	title = {Unmasking human trafficking risk in commercial sex supply chains with machine learning},
	year = {2021}}

@article{li2021detecting,
	author = {Li, Ruoting and Tobey, Margaret and Mayorga, Maria and Caltagirone, Sherrie and \"{O}zalt{\i}n, Osman},
	journal = {Available at SSRN 3982796},
	title = {Detecting Human Trafficking: Automated Classification of Online Customer Reviews of Massage Businesses},
	year = {2021}}

@article{recht2007guaranteed,
	author = {Recht, Benjamin and Fazel, Maryam and Parrilo, Pablo A},
	journal = {arXiv preprint arXiv:0706.4138},
	title = {Guaranteed Minimum-Rank Solutions of Linear Matrix Equations via Nuclear Norm Minimization},
	year = {2007}}

@article{sarma2018domain,
	author = {Sarma, Prathusha K and Liang, Yingyu and Sethares, William A},
	journal = {arXiv preprint arXiv:1805.04576},
	title = {Domain adapted word embeddings for improved sentiment classification},
	year = {2018}}

@article{roy2017learning,
	author = {Roy, Arpita and Park, Youngja and Pan, SHimei},
	journal = {arXiv preprint arXiv:1709.07470},
	title = {Learning domain-specific word embeddings from sparse cybersecurity texts},
	year = {2017}}

@article{risch2019domain,
	author = {Risch, Julian and Krestel, Ralf},
	journal = {Data Technologies and Applications},
	publisher = {Emerald Publishing Limited},
	title = {Domain-specific word embeddings for patent classification},
	year = {2019}}

@article{ong2020machine,
	author = {Ong, Charlene Jennifer and Orfanoudaki, Agni and Zhang, Rebecca and Caprasse, Francois Pierre M and Hutch, Meghan and Ma, Liang and Fard, Darian and Balogun, Oluwafemi and Miller, Matthew I and Minnig, Margaret and others},
	journal = {PloS one},
	number = {6},
	pages = {e0234908},
	publisher = {Public Library of Science San Francisco, CA USA},
	title = {Machine learning and natural language processing methods to identify ischemic stroke, acuity and location from radiology reports},
	volume = {15},
	year = {2020}}

@article{rigollet201518,
	author = {Rigollet, Philippe},
	journal = {Lecture Notes, Cambridge, MA, USA: MIT Open-CourseWare},
	title = {18. s997: High dimensional statistics},
	year = {2015}}

@book{boucheron2013concentration,
	author = {Boucheron, St{\'e}phane and Lugosi, G{\'a}bor and Massart, Pascal},
	publisher = {Oxford university press},
	title = {Concentration inequalities: A nonasymptotic theory of independence},
	year = {2013}}

@article{raskutti2010restricted,
	author = {Raskutti, Garvesh and Wainwright, Martin J and Yu, Bin},
	journal = {The Journal of Machine Learning Research},
	pages = {2241--2259},
	publisher = {JMLR. org},
	title = {Restricted eigenvalue properties for correlated Gaussian designs},
	volume = {11},
	year = {2010}}

@article{negahban2012restricted,
	author = {Negahban, Sahand and Wainwright, Martin J},
	journal = {The Journal of Machine Learning Research},
	pages = {1665--1697},
	publisher = {JMLR. org},
	title = {Restricted strong convexity and weighted matrix completion: Optimal bounds with noise},
	volume = {13},
	year = {2012}}

@inproceedings{hsu-etal-2020-characterizing,
	address = {Online},
	author = {Hsu, Chao-Chun and Karnwal, Shantanu and Mullainathan, Sendhil and Obermeyer, Ziad and Tan, Chenhao},
	booktitle = {Findings of the Association for Computational Linguistics: EMNLP 2020},
	doi = {10.18653/v1/2020.findings-emnlp.187},
	month = nov,
	pages = {2062--2072},
	publisher = {Association for Computational Linguistics},
	title = {Characterizing the Value of Information in Medical Notes},
	url = {https://www.aclweb.org/anthology/2020.findings-emnlp.187},
	year = {2020}}

@inproceedings{blitzer2007biographies,
	author = {Blitzer, John and Dredze, Mark and Pereira, Fernando},
	booktitle = {Proceedings of the 45th annual meeting of the association of computational linguistics},
	pages = {440--447},
	title = {Biographies, bollywood, boom-boxes and blenders: Domain adaptation for sentiment classification},
	year = {2007}}

@article{liu2016structured,
	author = {Liu, Xiao and Singh, Param Vir and Srinivasan, Kannan},
	journal = {Marketing Science},
	number = {3},
	pages = {363--388},
	publisher = {INFORMS},
	title = {A structured analysis of unstructured big data by leveraging cloud computing},
	volume = {35},
	year = {2016}}

@article{agarwal2012noisy,
	author = {Agarwal, Alekh and Negahban, Sahand and Wainwright, Martin J},
	journal = {The Annals of Statistics},
	pages = {1171--1197},
	publisher = {JSTOR},
	title = {Noisy matrix decomposition via convex relaxation: Optimal rates in high dimensions},
	year = {2012}}

@inproceedings{ali2020implicit,
	author = {Ali, Alnur and Dobriban, Edgar and Tibshirani, Ryan},
	booktitle = {International Conference on Machine Learning},
	organization = {PMLR},
	pages = {233--244},
	title = {The implicit regularization of stochastic gradient flow for least squares},
	year = {2020}}

@article{bastani2020predicting,
	author = {Bastani, Hamsa},
	journal = {Management Science},
	publisher = {INFORMS},
	title = {Predicting with proxies: Transfer learning in high dimension},
	year = {2020}}

@article{ben2007analysis,
	author = {Ben-David, Shai and Blitzer, John and Crammer, Koby and Pereira, Fernando and others},
	journal = {Advances in neural information processing systems},
	pages = {137},
	publisher = {MIT; 1998},
	title = {Analysis of representations for domain adaptation},
	volume = {19},
	year = {2007}}

@article{ben2010theory,
	author = {Ben-David, Shai and Blitzer, John and Crammer, Koby and Kulesza, Alex and Pereira, Fernando and Vaughan, Jennifer Wortman},
	journal = {Machine learning},
	number = {1},
	pages = {151--175},
	publisher = {Springer},
	title = {A theory of learning from different domains},
	volume = {79},
	year = {2010}}

@article{bhojanapalli2016global,
	author = {Bhojanapalli, Srinadh and Neyshabur, Behnam and Srebro, Nathan},
	journal = {arXiv preprint arXiv:1605.07221},
	title = {Global optimality of local search for low rank matrix recovery},
	year = {2016}}

@book{buhlmann2011statistics,
	author = {B{\"u}hlmann, Peter and Van De Geer, Sara},
	publisher = {Springer Science \& Business Media},
	title = {Statistics for high-dimensional data: methods, theory and applications},
	year = {2011}}

@article{burer2003nonlinear,
	author = {Burer, Samuel and Monteiro, Renato DC},
	journal = {Mathematical Programming},
	number = {2},
	pages = {329--357},
	publisher = {Springer},
	title = {A nonlinear programming algorithm for solving semidefinite programs via low-rank factorization},
	volume = {95},
	year = {2003}}

@article{candes2011tight,
	author = {Candes, Emmanuel J and Plan, Yaniv},
	journal = {IEEE Transactions on Information Theory},
	number = {4},
	pages = {2342--2359},
	publisher = {IEEE},
	title = {Tight oracle inequalities for low-rank matrix recovery from a minimal number of noisy random measurements},
	volume = {57},
	year = {2011}}

@article{chi2019nonconvex,
	author = {Chi, Yuejie and Lu, Yue M and Chen, Yuxin},
	journal = {IEEE Transactions on Signal Processing},
	number = {20},
	pages = {5239--5269},
	publisher = {IEEE},
	title = {Nonconvex optimization meets low-rank matrix factorization: An overview},
	volume = {67},
	year = {2019}}

@inproceedings{dingwall2018mittens,
	author = {Dingwall, Nicholas and Potts, Christopher},
	booktitle = {Proceedings of the 2018 Conference of the North American Chapter of the Association for Computational Linguistics: Human Language Technologies, Volume 2 (Short Papers)},
	pages = {212--217},
	title = {Mittens: an Extension of GloVe for Learning Domain-Specialized Representations},
	year = {2018}}

@article{friedman2010note,
	author = {Friedman, Jerome and Hastie, Trevor and Tibshirani, Robert},
	journal = {arXiv preprint arXiv:1001.0736},
	title = {A note on the group lasso and a sparse group lasso},
	year = {2010}}

@inproceedings{ganin2015unsupervised,
	author = {Ganin, Yaroslav and Lempitsky, Victor},
	booktitle = {International conference on machine learning},
	organization = {PMLR},
	pages = {1180--1189},
	title = {Unsupervised domain adaptation by backpropagation},
	year = {2015}}

@article{hsu2012tail,
	author = {Hsu, Daniel and Kakade, Sham and Zhang, Tong and others},
	journal = {Electronic Communications in Probability},
	publisher = {The Institute of Mathematical Statistics and the Bernoulli Society},
	title = {A tail inequality for quadratic forms of subgaussian random vectors},
	volume = {17},
	year = {2012}}

@article{levy2014neural,
	author = {Levy, Omer and Goldberg, Yoav},
	journal = {Advances in neural information processing systems},
	pages = {2177--2185},
	title = {Neural word embedding as implicit matrix factorization},
	volume = {27},
	year = {2014}}

@article{li2019non,
	author = {Li, Qiuwei and Zhu, Zhihui and Tang, Gongguo},
	journal = {Information and Inference: A Journal of the IMA},
	number = {1},
	pages = {51--96},
	publisher = {Oxford University Press},
	title = {The non-convex geometry of low-rank matrix optimization},
	volume = {8},
	year = {2019}}

@article{loh2015regularized,
	author = {Loh, Po-Ling and Wainwright, Martin J},
	journal = {The Journal of Machine Learning Research},
	number = {1},
	pages = {559--616},
	publisher = {JMLR. org},
	title = {Regularized M-estimators with nonconvexity: Statistical and algorithmic theory for local optima},
	volume = {16},
	year = {2015}}

@article{lounici2011oracle,
	author = {Lounici, Karim and Pontil, Massimiliano and Van De Geer, Sara and Tsybakov, Alexandre B and others},
	journal = {The Annals of Statistics},
	number = {4},
	pages = {2164--2204},
	publisher = {Institute of Mathematical Statistics},
	title = {Oracle inequalities and optimal inference under group sparsity},
	volume = {39},
	year = {2011}}

@inproceedings{mikolov2013distributed,
	author = {Mikolov, Tomas and Sutskever, Ilya and Chen, Kai and Corrado, Greg S and Dean, Jeff},
	booktitle = {Advances in neural information processing systems},
	pages = {3111--3119},
	title = {Distributed representations of words and phrases and their compositionality},
	year = {2013}}

@article{negahban2011estimation,
	author = {Negahban, Sahand and Wainwright, Martin J},
	journal = {The Annals of Statistics},
	pages = {1069--1097},
	publisher = {JSTOR},
	title = {Estimation of (near) low-rank matrices with noise and high-dimensional scaling},
	year = {2011}}

@inproceedings{pennington2014glove,
	author = {Pennington, Jeffrey and Socher, Richard and Manning, Christopher D},
	booktitle = {Proceedings of the 2014 conference on empirical methods in natural language processing (EMNLP)},
	pages = {1532--1543},
	title = {Glove: Global vectors for word representation},
	year = {2014}}

@article{simon2013sparse,
	author = {Simon, Noah and Friedman, Jerome and Hastie, Trevor and Tibshirani, Robert},
	journal = {Journal of computational and graphical statistics},
	number = {2},
	pages = {231--245},
	publisher = {Taylor \& Francis Group},
	title = {A sparse-group lasso},
	volume = {22},
	year = {2013}}

@inproceedings{yang2017simple,
	author = {Yang, Wei and Lu, Wei and Zheng, Vincent},
	booktitle = {Proceedings of the 2017 Conference on Empirical Methods in Natural Language Processing},
	pages = {2898--2904},
	title = {A Simple Regularization-based Algorithm for Learning Cross-Domain Word Embeddings},
	year = {2017}}

@inproceedings{zhang2013domain,
	author = {Zhang, Kun and Sch{\"o}lkopf, Bernhard and Muandet, Krikamol and Wang, Zhikun},
	booktitle = {International Conference on Machine Learning},
	organization = {PMLR},
	pages = {819--827},
	title = {Domain adaptation under target and conditional shift},
	year = {2013}}

@article{negahban2012unified,
	author = {Negahban, Sahand N and Ravikumar, Pradeep and Wainwright, Martin J and Yu, Bin and others},
	journal = {Statistical science},
	number = {4},
	pages = {538--557},
	publisher = {Institute of Mathematical Statistics},
	title = {A unified framework for high-dimensional analysis of $ M $-estimators with decomposable regularizers},
	volume = {27},
	year = {2012}}

@article{mankad2016understanding,
	author = {Mankad, Shawn and Han, Hyunjeong ``Spring'' and Goh, Joel and Gavirneni, Srinagesh},
	journal = {Service Science},
	number = {2},
	pages = {124--138},
	publisher = {INFORMS},
	title = {Understanding online hotel reviews through automated text analysis},
	volume = {8},
	year = {2016}}

@article{bellstam2020text,
	author = {Bellstam, Gustaf and Bhagat, Sanjai and Cookson, J Anthony},
	journal = {Management Science},
	publisher = {INFORMS},
	title = {A text-based analysis of corporate innovation},
	year = {2020}}

@article{netzer2012mine,
	author = {Netzer, Oded and Feldman, Ronen and Goldenberg, Jacob and Fresko, Moshe},
	journal = {Marketing Science},
	number = {3},
	pages = {521--543},
	publisher = {INFORMS},
	title = {Mine your own business: Market-structure surveillance through text mining},
	volume = {31},
	year = {2012}}
}

%
%
%
\newpage
\renewcommand{\theHsection}{\Alph{section}}
\begin{APPENDICES}

\section{Quadratic Compatibility Condition}
\label{sec:qccdiscussion}

\begin{proof}{Proof of Proposition~\ref{thm:rsc_quadcomp}}
The RSC condition gives
\begin{align}\label{eq:rsc_plugin}
\frac{1}{n}\|\mathcal{A}(\Delta U^{*T} + U^*\Delta^T + \Delta \Delta^T)\|^2 \ge \eta\|\Delta U^{*T} + U^*\Delta^T + \Delta \Delta^T\|_F^2 - \tau(n, d, r) \|\Delta\|_{2, 1}^2.
\end{align}
We lower bound the first term in inequality (\ref{eq:rsc_plugin}):
\begin{align*}
\|\Delta U^{*T} + U^*\Delta^T + \Delta \Delta^T\|_F^2 = & \|\Delta U^{*T} + U^*\Delta^T\|_F^2 + \|\Delta \Delta^T\|_F^2 + 4 \langle U^*\Delta^T, \Delta \Delta^T \rangle \\
= & 4\|\Delta U^{*T}\|_F^2 + \|\Delta \Delta^T\|_F^2 + 4 \langle U^*\Delta^T, \Delta \Delta^T \rangle \\
\ge & 4\|\Delta U^{*T}\|_F^2 + \|\Delta \Delta^T\|_F^2 - 4 |U^*\Delta^T|_{\infty} |\Delta \Delta^T|_1 \\
\ge & 4\|\Delta U^{*T}\|_F^2 + \|\Delta \Delta^T\|_F^2 - 4 \|U^*\|_{2, \infty}\|\Delta\|_{2, \infty} \|\Delta\|_{2, 1}^2 \\
\ge & 4\|\Delta U^{*T}\|_F^2 + \|\Delta \Delta^T\|_F^2 - 4 \frac{D\bar{L}}{\sqrt{d}} \|\Delta\|_{2, 1}^2.
\end{align*}
where the second equality uses $\tr(X^2) = \|X\|_F^2$. This gives us
\begin{align*}
\frac{1}{n}\|\mathcal{A}(\Delta U^{*T} + U^*\Delta^T + \Delta \Delta^T)\|^2 \ge 4\eta\|\Delta U^{*T}\|_F^2 + \eta\|\Delta \Delta^T\|_F^2
- 4\frac{\eta D\bar{L}}{\sqrt{d}} \|\Delta\|_{2, 1}^2
- \tau(n, d, r) \|\Delta\|_{2,1}^2.
\end{align*}
Under the condition that
\begin{align*}
\sum_{j \in J^c} \|\Delta^{j}\| \le 7 \sum_{j \in J} \|\Delta^{j}\|,
\end{align*}
we can upper bound $\|\Delta\|_{2,1}^2$ with a constant scale of $\|\Delta\|_F^2$:
\begin{align*}
\|\Delta\|_{2,1}^2 = (\sum_{j \in J^c} \|\Delta^{j}\| + \sum_{j \in J} \|\Delta^{j}\|)^2 \le (8 \sum_{j \in J} \|\Delta^{j}\|)^2 \le 64 s \|\Delta\|_F^2.
\end{align*}
Therefore, we have 
\begin{align*}
\frac{1}{n}\|\mathcal{A}(\Delta U^{*T} + U^*\Delta^T + \Delta \Delta^T)\|^2 \ge \frac{4\eta\sigma_r(U^{*})^2}{s} (\sum_{j \in J} \|\Delta^{j}\|)^2 + \eta\|\Delta \Delta^T\|_F^2
- 64\left(4\frac{\eta D\bar{L}}{\sqrt{d}} + \tau(n, d, r) \right) (\sum_{j \in J} \|\Delta^{j}\|)^2.
\end{align*}
As long as $n$ and $d$ are such that 
\begin{align*}
\frac{\eta\sigma_r(U^{*})^2}{32s} \ge 4\frac{\eta D\bar{L}}{\sqrt{d}} + \tau(n, d, r),
\end{align*}
we can derive the quadratic compatibility condition
\begin{align*}
\frac{1}{n}\|\mathcal{A}(\Delta U^{*T} + U^*\Delta^T + \Delta \Delta^T)\|^2 
\ge \frac{2\eta \sigma_r(U^{*})^2}{s} (\sum_{j \in J} \|\Delta^{j}\|)^2
\end{align*}
with $\kappa=2\eta\sigma_r(U^{*})^2$. \Halmos
\end{proof}

\begin{proof}{Proof of Proposition~\ref{prop:rsc_concenineq}}
Our proof is adapted from the proof in \cite{raskutti2010restricted} and that of Proposition 1 in \cite{negahban2011estimation}.
Let $\bar{\mathcal{A}}: \mathbb{R}^{d \times d} \rightarrow \mathbb{R}^{n}$ with $\vecr(\bar{A}_{i}) \sim N(0, I)$. Then, we have by construction $\mathcal{A}(\Theta) = \bar{\mathcal{A}}(T_{\Sigma}(\Theta))$. 

Consider the set $\mathcal{R}(t) = \{\Delta \mid \|T_{\Sigma}(\Delta U^{*T} + U^*\Delta^T + \Delta \Delta^T)\|_F = b, \|\Delta\|_{2, 1} \le t \}$ for any given $b > 0$.
We aim to lower bound 
\[
\inf_{\Delta \in \mathcal{R}(t)} \|\mathcal{A}(\Delta U^{*T} + U^*\Delta^T + \Delta \Delta^T)\| = \inf_{\Delta \in \mathcal{R}(t)} \sup_{u \in S^{n-1}} \langle u, \mathcal{A}(\Delta U^{*T} + U^*\Delta^T + \Delta \Delta^T)\rangle,
\]
where $S^{n-1}$ is the $(n-1)$-dimension unit sphere. We define an associated zero-mean gaussian random variable $Z_{u, \Delta} = \langle u, \bar{\mathcal{A}}(T_{\Sigma}(\Delta U^{*T} + U^*\Delta^T + \Delta \Delta^T))\rangle$. For any pairs $(u, \Delta)$ and $(u', \Delta')$, we have 
\begin{align*}
\mathbb{E}[(Z_{u, \Delta} - Z_{u', \Delta'})^2] = \|u \otimes T_{\Sigma}(\Delta U^{*T} + U^*\Delta^T + \Delta \Delta^T) - u' \otimes T_{\Sigma}(\Delta' U^{*T} + U^*\Delta'^T + \Delta' \Delta'^T)\|_F^2,
\end{align*}
where $\otimes$ is the Kronecker product.
Now consider a second zero-mean gaussian process $Y_{u, \Delta} = b\langle g, u \rangle + \langle G,  T_{\Sigma}(\Delta U^{*T} + U^*\Delta^T + \Delta \Delta^T)\rangle $, where $g \in \mathbb{R}^{n}$ and $G \in \mathbb{R}^{d \times d}$ have i.i.d. $N(0, 1)$ entries.  For any pairs $(u, \Delta)$ and $(u', \Delta')$, we have 
\begin{align*}
\mathbb{E}[(Y_{u, \Delta} - Y_{u', \Delta'})^2] = b^2\|u - u'\|^2 + \|T_{\Sigma}(\Delta U^{*T} + U^*\Delta^T + \Delta \Delta^T) - T_{\Sigma}(\Delta' U^{*T} + U^*\Delta'^T + \Delta' \Delta'^T)\|_F^2.
\end{align*}
As $\|u\|=1$ and $\|T_{\Sigma}(\Delta U^{*T} + U^*\Delta^T + \Delta \Delta^T)\|_F=b$, it holds that
\begin{multline*}
\|u \otimes T_{\Sigma}(\Delta U^{*T} + U^*\Delta^T + \Delta \Delta^T) - u' \otimes T_{\Sigma}(\Delta' U^{*T} + U^*\Delta'^T + \Delta' \Delta'^T)\|_F^2 \\
\le b^2 \|u - u'\|^2 + \|T_{\Sigma}(\Delta U^{*T} + U^*\Delta^T + \Delta \Delta^T) - T_{\Sigma}(\Delta' U^{*T} + U^*\Delta'^T + \Delta' \Delta'^T)\|_F^2,
\end{multline*}
where we use the fact that $\langle X, X - X'\rangle \ge 0$ for any matrix $X,X'$ with $\|X\|_F=\|X'\|_F$.
Note that when $\Delta=\Delta'$, the equality holds. Consequently, gaussian comparison inequalities, specifically Gordon's inequality \citep[see, e.g.,][]{raskutti2010restricted}, gives rise to
\begin{align*}
\mathbb{E}\left[\inf_{\Delta \in \mathcal{R}(t)} \sup_{u \in S^{n - 1}} Z_{u, \Delta}\right] \ge \mathbb{E}\left[\inf_{\Delta \in \mathcal{R}(t)} \sup_{u \in S^{n - 1}} Y_{u, \Delta}\right]. 
\end{align*}
The gaussian process $Y_{u, \Delta}$ has
\begin{align*}
\mathbb{E}\left[\inf_{\Delta \in \mathcal{R}(t)} \sup_{u \in S^{n - 1}} Y_{u, \Delta}\right]
= & \mathbb{E}\left[b\sup_{u \in S^{n - 1}} \langle g, u \rangle \right] + \mathbb{E}\left[\inf_{\Delta \in \mathcal{R}(t)} \langle G,  T_{\Sigma}(\Delta U^{*T} + U^*\Delta^T + \Delta \Delta^T)\rangle \right] \\
= & b\mathbb{E}\left[\|g\| \right] - \mathbb{E}\left[\sup_{\Delta \in \mathcal{R}(t)} \langle G,  T_{\Sigma}(\Delta U^{*T} + U^*\Delta^T + \Delta \Delta^T)\rangle \right].
\end{align*}
Using Lemma~\ref{lemma:gammafuncineq}, the first term has $\mathbb{E}\left[\|g\| \right] \ge \sqrt{n}/2$ by calculating an integral of a chi-squared distribution.
For the second term, 
\begin{align*}
\langle G, T_{\Sigma}(\Delta U^{*T} + U^*\Delta^T + \Delta \Delta^T) \rangle 
& = \langle T_{\Sigma}(G), \Delta U^{*T} + U^*\Delta^T + \Delta \Delta^T \rangle \\
& = \langle T_{\Sigma}(G) U^{*}, \Delta \rangle + \langle T_{\Sigma}(G)^T U^{*}, \Delta \rangle + \langle T_{\Sigma}(G), \Delta \Delta^T \rangle \\
& \le (\| T_{\Sigma}(G) U^{*}\|_{2, \infty} + \|T_{\Sigma}(G)^T U^{*}\|_{2, \infty} + \bar{L}|T_{\Sigma}(G)|_{\infty}) \|\Delta\|_{2, 1},
\end{align*}
where we use $\|\Delta\|_{2,1} \le \bar{L}$. Note that $\vecr(T_{\Sigma}(G)) \sim N(0, \Sigma)$. Lemma~\ref{lemma:subgaussianmaxima} gives that 
\begin{align*}
\mathbb{E}[|T_{\Sigma}(G)|_{\infty}] \le 2 \sqrt{\max_{i \in [d^2]} \Sigma^{(i,i)} \log(\sqrt{2}d)}.
\end{align*}
Finally, using Lemma \ref{lemma:gaussiannormmaxima} - \ref{lemma:gammafuncineq} and Jensen's inequality, we have
\begin{align*}
\mathbb{E}[\|T_{\Sigma}(G) U^{*}\|_{2, \infty}] & \le \max_{i \in [d]} \sqrt{\tr(U^{*T} \bar{\Sigma}_{ii} U^*)} + \sqrt{2 \max_{i \in [d]} \|U^{*T} \bar{\Sigma}_{ii} U^*\| \log d} \\
& \le \sigma_1(U^*) \max_{i \in [d]}\sqrt{\sigma_1(\bar{\Sigma}_{ii})} (\sqrt{r} + \sqrt{2 \log d}).
\end{align*}
Similar results hold for $\mathbb{E}[\|T_{\Sigma}(G)^T U^{*}\|_{2, \infty}]$. Define $\Sigma'$ to be such that $\vecr(T_{\Sigma}(G)^T) \sim N(0, \Sigma')$ and hence $\Sigma'=K^{(d,d)} \Sigma K^{(d,d)}$. $K^{(d,d)} \in \mathbb{R}^{d^2 \times d^2}$ is a commutation matrix that transform $\vecr(X)$ to $\vecr(X^T)$ for $X \in \mathbb{R}^{d \times d}$:
\[
K^{(d,d)} \vecr(X) = \vecr(X^T).
\] 
Note that $\Sigma'$ shares a similar property as $\Sigma$ as $K^{(d, d)}$ is nonsingular and has only eigenvalues $1$ or $-1$. 
Therefore, we can obtain
\begin{align*}
\mathbb{E}[\| T_{\Sigma}(G) U^{*}\|_{2, \infty} + \|T_{\Sigma}(G)^T U^{*}\|_{2, \infty} + \bar{L}|T_{\Sigma}(G)|_{\infty}]
\le C_6 (\sqrt{\log(\sqrt{2}d)} +  (\sqrt{r} + \sqrt{2 \log d})) \le C_6 (2\sqrt{r} + 3\sqrt{\log d}),
\end{align*}
where 
$C_6 = 2\bar{L}\sqrt{\max_{i \in [d^2]}\Sigma^{(i,i)}} + \sigma_1(U^*) (\max_{i \in [d]}\sqrt{\sigma_1(\bar{\Sigma}_{ii})} + \max_{i \in [d]}\sqrt{\sigma_1(\bar{\Sigma}'_{ii})}).$
Combining all the above gives
\begin{align*}
\mathbb{E}\left[\inf_{\Delta \in \mathcal{R}(t)} \|\bar{\mathcal{A}}(T_{\Sigma}(\Delta U^{*T} + U^*\Delta^T + \Delta \Delta^T))\|\right] 
= \mathbb{E}\left[\inf_{\Delta \in \mathcal{R}(t)} \sup_{u \in S^{n - 1}} Z_{u, \Delta}\right] 
\ge \frac{b\sqrt{n}}{2} - C_6 (2\sqrt{r} + 3\sqrt{\log d})t.
\end{align*}

It's easy to show that the function $f(\bar{\mathcal{A}}) : = \inf_{\Delta \in \mathcal{R}(t)} \|\bar{\mathcal{A}}(T_{\Sigma}(\Delta U^{*T} + U^*\Delta^T + \Delta \Delta^T))\|$ is $b$-Lipschitz.
Then, applying Lemma~\ref{lemma:gaussianfuncmaxima} shows that 
\begin{align*}
\mathbb{P}\left(\sup_{\Delta \in \mathcal{R}(t)} (\frac{5b}{8} - \frac{\|\mathcal{A}(\Delta U^{*T} + U^*\Delta^T + \Delta \Delta^T)\|}{\sqrt{n}}) \ge \frac{3b}{2} g(t)\right) \le \exp\left(-\frac{ng(t)^2}{8}\right),
\end{align*}
where $g(t)=\frac{1}{8} + \frac{C_6 (2\sqrt{r} + 3\sqrt{\log d})t}{b\sqrt{n}}$. 
By a peeling argument \citep[see, e.g., Lemma 3 in][]{raskutti2010restricted}, we can derive that 
\begin{align*}
\frac{\|\mathcal{A}(\Delta U^{*T} + U^*\Delta^T + \Delta \Delta^T)\|}{\sqrt{n}} \ge \frac{b}{4} - \frac{3C_6 (2\sqrt{r} + 3\sqrt{\log d})}{2\sqrt{n}} \|\Delta\|_{2, 1}
\end{align*}
for any $\Delta \in \{\Delta \mid \|T_{\Sigma}(\Delta U^{*T} + U^*\Delta^T + \Delta \Delta^T)\|_F = b\}$ with probability greater than $1 - c \exp(-c'n)$ for some positive constants $c, c'$. \Halmos
\end{proof}

\begin{proof}{Proof of Proposition \ref{prop:rsc_concenineq_we}}
Our proof strategy is adapted from the proof of Lemma 14 in \cite{klopp2014noisy} and Theorem 1 in \cite{negahban2012restricted}.

Consider the following set 
\begin{align*}
\mathcal{R}(t)=\{\Delta \mid \|\Delta U^{*T} + U^*\Delta^T + \Delta \Delta^T\|_\infty=b, \frac{\|\Delta U^{*T} + U^*\Delta^T + \Delta \Delta^T\|_\infty}{\|\Delta U^{*T} + U^*\Delta^T + \Delta \Delta^T\|_{L_2(\Pi)}}\le B, \|\Delta\|_{2,1}\le t\}
\end{align*}
for any given $b>0$. Let 
\begin{align*}
Z_t = \sup_{\Delta\in\mathcal{R}(t)} |\frac{1}{n}\|\mathcal{A}(\Delta U^{*T} + U^*\Delta^T + \Delta \Delta^T)\|^2 - \|\Delta U^{*T} + U^*\Delta^T + \Delta \Delta^T\|_{L_2(\Pi)}^2|.
\end{align*}
Note that $A_i$'s are basis matrices with only one entry being 1 and others 0; thus, 
\begin{align*}
|\langle A_i, \Delta U^{*T} + U^*\Delta^T + \Delta \Delta^T\rangle^2 - \|\Delta U^{*T} + U^*\Delta^T + \Delta \Delta^T\|_{L_2(\Pi)}^2| \le 2b^2.
\end{align*}
for all $i\in[N]$ and any $\Delta \in \mathcal{R}(t)$. Then, we can use Massart's concentration inequality \citep[see, e.g., Theorem 14.2 in][]{buhlmann2011statistics} to obtain
\begin{align}\label{eq:prob_basis}
\mathbb{P}(Z_t \ge \mathbb{E}[Z_t] + 2b^2\chi) \le \exp\left(-\frac{n\chi^2}{8}\right)
\end{align}
for any $\chi>0$. Next, we bound the expectation $\mathbb{E}[Z_t]$. Using a standard symmetrization argument \citep[see, e.g., Theorem 2.1 in][]{koltchinskii2011oracle}, we have
\begin{align*}
\mathbb{E}[Z_t] 
& \le 2 \mathbb{E}\left[\sup_{\Delta\in\mathcal{R}(t)}|\frac{1}{n}\sum_{i\in[n]}\zeta_i\langle A_i, \Delta U^{*T} + U^*\Delta^T + \Delta \Delta^T\rangle^2|\right],
\end{align*}
where $\zeta_i$'s are i.i.d. Rademacher random variables. Since $|\langle A_i,  \Delta U^{*T} + U^*\Delta^T + \Delta \Delta^T\rangle| \le b$ for any $\Delta\in\mathcal{R}(t)$, using the contraction inequality \citep[see, e.g., ][]{koltchinskii2011oracle} gives
\begin{align*}
\mathbb{E}[Z_t] 
& \le 8b\mathbb{E}[\sup_{\Delta\in\mathcal{R}(t)}|\langle G, \Delta U^{*T} + U^*\Delta^T + \Delta \Delta^T\rangle|]
\end{align*}
where $G = \frac{1}{n}\sum_{i\in[n]}\zeta_iA_i$. Now we decompose the term on the right as follows, 
\begin{align*}
|\langle G, \Delta U^{*T} + U^*\Delta^T + \Delta \Delta^T\rangle|
& = |\langle GU^*, \Delta\rangle + \langle G^TU^*, \Delta\rangle + \langle G, \Delta\Delta^T\rangle| \\
& \le (\|GU^*\|_{2, \infty} + \|G^TU^*\|_{2, \infty} + \bar{L} |G|_{\infty})\|\Delta\|_{2, 1},
\end{align*}
where we use $\|\Delta\|_{2,1}\le\bar{L}$. Note that each entry $(j, k)$ of the matrix $G$ is $G^{(j, k)}=\frac{1}{n}\sum_{i\in[n]}\zeta_iA_i^{(j, k)}$, where $\zeta_iA_i^{(j, k)}$ has mean 0, variance $\mu_2/d^2$, and upper bound $1$. Therefore, using the Bernstein inequality yields
\begin{align*}
\mathbb{P}(|G^{(j, k)}|\ge \chi) \le  2\exp\left(-\frac{n\chi^2}{\frac{2\mu_2}{d^2}+\frac{2\chi}{3}}\right).
\end{align*}
With a union bound, we further have
\begin{align*}
\mathbb{P}(\|G\|_{\infty}\ge \chi) \le  2d^2\exp\left(-\frac{n\chi^2}{\frac{2\mu_2}{d^2}+\frac{2\chi}{3}}\right).
\end{align*}
Then, using the proof strategy of Lemma 6 in \cite{klopp2014noisy}, it follows that
\begin{align*}
\mathbb{E}[\|G\|_\infty] \le (\mathbb{E}[\|G\|_\infty^{2\log d}])^{1/(2\log d)} \le 11\left(\sqrt{\frac{2\mu_2\log d}{nd^2}}+\frac{4\log d}{3n}\right).
\end{align*}
Moreover, each row $j$ of $GU^*$ is $G^{(j, \cdot)}U^* = \frac{1}{n}\sum_{i\in[n]}\zeta_iA_i^{(j, \cdot)}U^*$, where $\zeta_iA_i^{(j,\cdot)}U^*$ has mean 0, $\ell_2$-norm upper bound $\sigma_1(U^*)$, and 
\[
\max\left\{\|\frac{1}{n}\sum_{i\in[n]}\mathbb{E}[A_i^{(j,\cdot)}U^*U^{*T}A_i^{(j,\cdot)T}]\|, \|\frac{1}{n}\sum_{i\in[n]}\mathbb{E}[U^{*T}A_i^{(j,\cdot)T}A_i^{(j,\cdot)}U^*]\|\right\} \le \frac{r\mu_2\sigma_1^2(U^*)}{d^2}.
\]
Therefore, using the matrix Bernstein inequality (see Lemma~\ref{lem: matrix_berstein_inequality}) and a union bound, we have
\begin{align*}
\mathbb{P}(\|GU^*\|_{2,\infty} \geq x) \leq 2d^2\exp\left(\frac{-nx^2}{\frac{2r\mu_2\sigma_1^2(U^*)}{d^2}+\frac{2\sigma_1(U^*)x}{3}}\right).
\end{align*}
Again using the proof strategy of Lemma 6 in \cite{klopp2014noisy}, we have
\begin{align*}
\mathbb{E}[\|GU^*\|_{2,\infty}]\le 11\sigma_1(U^*)\left(\sqrt{\frac{4r\mu_2\log d}{nd^2}} + \frac{8\log d}{3n}\right).
\end{align*}
Similarly, we can get a same upper bound of $\mathbb{E}[\|G^TU^*\|_{2,\infty}]$. Combining all the above, we have 
\[
\mathbb{E}[Z_t] \le C_7b\left(\sqrt{\frac{\log d}{nd^2}}+\frac{\log d}{n}\right)t,
\]
where $\tilde{C}_7 = 88(\bar{L}(\sqrt{2\mu_2}+4/3) + 2\sigma_1(U^*)(\sqrt{4r\mu_2}+8/3))$. Plugging it into (\ref{eq:prob_basis}) and setting $\chi=g(t)=\frac{1}{10B^2}+\frac{\tilde{C}_7(\sqrt{\frac{\log d}{nd^2}}+\frac{\log d}{n})t}{b}$, we derive that
\begin{align*}
\mathbb{P}(Z_t \ge -\frac{b^2}{10B^2} + 3b^2g(t)) \le \exp\left(-\frac{ng(t)^2}{8}\right).
\end{align*}
Again, using the peeling argument \citep[see, e.g., Lemma 3 in][]{raskutti2010restricted}, it yields that
\begin{align*}
\frac{1}{n}\|\mathcal{A}(\Delta U^{*T} + U^*\Delta^T + \Delta \Delta^T)\|^2 \ge \|\Delta U^{*T} + U^*\Delta^T + \Delta \Delta^T\|_{L_2(\Pi)}^2 - \frac{1}{2}b^2
- 6b\tilde{C}_7\left(\sqrt{\frac{\log d}{nd^2}}+\frac{\log d}{n}\right) \|\Delta\|_{2,1},
\end{align*}
for any $\Delta\in\{\Delta \mid \|\Delta U^{*T} + U^*\Delta^T + \Delta \Delta^T\|_\infty=b, \frac{\|\Delta U^{*T} + U^*\Delta^T + \Delta \Delta^T\|_\infty}{\|\Delta U^{*T} + U^*\Delta^T + \Delta \Delta^T\|_{L_2(\Pi)}}\le B\}$ with probability greater than $1-c\exp(-\frac{c'}{B^4}n)$ for some positive constants $c,c'$. Thus, it implies that
\begin{multline*}
\frac{1}{n}\|\mathcal{A}(\Delta U^{*T} + U^*\Delta^T + \Delta \Delta^T)\|^2 \ge \frac{1}{2}\|\Delta U^{*T} + U^*\Delta^T + \Delta \Delta^T\|_{L_2(\Pi)}^2 \\
- 6C_7\left(\sqrt{\frac{\log d}{nd^2}}+\frac{\log d}{n}\right)\|\Delta U^{*T} + U^*\Delta^T + \Delta \Delta^T\|_{L_2(\Pi)} \|\Delta\|_{2,1},
\end{multline*}
for any $\Delta\in\{\Delta \mid \frac{\|\Delta U^{*T} + U^*\Delta^T + \Delta \Delta^T\|_\infty}{\|\Delta U^{*T} + U^*\Delta^T + \Delta \Delta^T\|_{L_2(\Pi)}}\le B\}$, where $C_7 = B\tilde{C}_7$.

Note that when $\pi_{j,k}\ge\mu_1/d^2$ for any $j,k\in[d]$, it holds that $\frac{\mu_1}{d^2}\|\Delta U^{*T} + U^*\Delta^T + \Delta \Delta^T\|_F^2 \le \|\Delta U^{*T} + U^*\Delta^T + \Delta \Delta^T\|_{L_2(\Pi)}^2$. Our result then follows by using the inequality $a^2+b^2 \ge 2ab$.
\Halmos
\end{proof}

\section{Error Bound of Transfer Learning Estimator}
\label{sec:thmjointproof}

\begin{lemma}
\label{lem:joint_estimator_concenineq}
Assume $\mathcal{A}_p$ satisfies $2r$-\textsf{RWC}$(\alpha_p, \beta_p)$, and $\mathcal{A}_g$ satisfies the quadratic compatibility condition. 
Let $A_g^{lk} = \begin{bmatrix} A_{g,1}^{(l,k)} & \cdots & A_{g,n_g}^{(l,k)} \end{bmatrix}^T$. 
Define $\Psi_j,\Phi_j \in \mathbb{R}^{r \times r}$ to be
\begin{align*}
\Psi_j =  U_g^{*T} \frac{A_{g}^{jT} A_{g}^{j}}{n_g} U_g^{*}, \quad \Phi_j =  U_g^{*T} \frac{(A_{g}^{Tj})^T A_{g}^{Tj}}{n_g} U_g^{*},
\end{align*}
where $A_g^j,A_g^{Tj} \in \mathbb{R}^{n_g \times d}$ are defined as matrices that stacks up the $j^{\text{th}}$ rows of $A_{g, i}$ and $A_{g, i}^T,~ i \in [n_g]$ respectively, i.e.,
\[
A_g^j = \begin{bmatrix}
A_{g, 1}^{(j,\cdot)} \\
A_{g, 2}^{(j,\cdot)}\\
\vdots \\
A_{g, n_g}^{(j,\cdot)}
\end{bmatrix}, \quad
A_g^{Tj} = \begin{bmatrix}
A_{g, 1}^{T(j,\cdot)} \\
A_{g, 2}^{T(j,\cdot)} \\
\vdots \\
A_{g, n_g}^{T(j,\cdot)}
\end{bmatrix}.
\]
Then, our two-stage transfer learning estimator satisfies with any chosen values of $\lambda>0$ and $c>0$
\begin{align*}
\|\widehat{U}_g^{TL} - U_g^*\|_{2,1} \ge 16(\frac{\lambda s}{\kappa} + \frac{4 \sqrt{d} c}{\sigma_{r}(U_p^*)(3\alpha_p - 2\beta_p)})
\end{align*}
with probability at most
\begin{align*}
& 2 (36\sqrt{2})^{2r(2d+1)} \exp(-\frac{L^2 \sigma_{r}^2(U_p^*) (3\alpha_p - 2\beta_p)^2 n_p}{512\beta_p\sigma_p^2d}) \\
& + 2 d^2 \exp\left(-\frac{\lambda^2 n_g}{2048 L^2 \sigma_g^2 (\max_{l,k}\|A_g^{lk}\|^2/n_g)}\right) \\
& + d \max_{j \in [d]} \exp \left( -( \sqrt{\frac{ \frac{\lambda^2n_g}{256\sigma_g^2} -  (\tr(\Psi_j) - \frac{\|\Psi_j\|_F^2}{2 \| \Psi_j \|})}{2 \| \Psi_j \|}} - \frac{\|\Psi_j\|_F}{2\|\Psi_j\|} )^2 \right) \\
& + d \max_{j \in [d]} \exp \left( -( \sqrt{\frac{ \frac{\lambda^2n_g}{256\sigma_g^2} -  (\tr(\Phi_j) - \frac{\|\Phi_j\|_F^2}{2 \| \Phi_j \|})}{2 \| \Phi_j \|}} - \frac{\|\Phi_j\|_F}{2\|\Phi_j\|} )^2 \right) \\
& + 2 (36\sqrt{2})^{2r(2d+1)} \exp(-\frac{c^2n_p}{8\beta_p\sigma_p^2}).
\end{align*}
\end{lemma}

\begin{proof} {Proof of Lemma~\ref{lem:joint_estimator_concenineq}}

As problem (\ref{eq:joint_optimtrans}) is equivalent to problem (\ref{eq:joint_optim}), we analyze problem (\ref{eq:joint_optimtrans}) for simplicity. 

Note that the row sparsity is immune to rotations, that is, for any orthogonal matrix $R$, $\Delta_U^* R$ is still row sparse. After our first step of finding the proxy estimator, we align $\widehat{U}_p$ with $U_p^*$ in the direction of $R_{(\widehat{U}_p, U_p^*)}$. By our definition, 
\[
U_g^* R_{(\widehat{U}_p, U_p^*)} = U_p^* R_{(\widehat{U}_p, U_p^*)} + \Delta_U^* R_{(\widehat{U}_p, U_p^*)}.
\]
Through our previous analyses, $\widehat{U}_p$ is close to $U_p^* R_{(\widehat{U}_p, U_p^*)}$ with a high probability. Therefore, in our second step, we aim to find an estimator $\widehat{\Delta}_U$ for $\Delta_U^* R_{(\widehat{U}_p, U_p^*)}$ through $\ell_{2, 1}$ penalty. For simplicity, we use $U_g^*$, $U_p^*$ and $\Delta_U^*$ to represent $U_g^* R_{(\widehat{U}_p, U_p^*)}$, $U_p^* R_{(\widehat{U}_p, U_p^*)}$ and $\Delta_U^* R_{(\widehat{U}_p, U_p^*)}$ respectively in the following analyses, which are aligned in the direction of $R_{(\widehat{U}_p, U_p^*)}$. Define the first-stage estimation error $\nu = \widehat{U}_p - U_p^*$ and $\widetilde{\Delta}_U = \Delta_U^* - \nu$. Thus, $U_g^* = U_p^* + \Delta_U^* = \widehat{U}_p + \widetilde{\Delta}_U$. Since $\widehat{U}_p$ carries the estimation error from the first step, the parameter we actually want to recover is $\widetilde{\Delta}_U$, which is approximately row sparse when the proxy data is huge.
We define the adjoint of an operator $\mathcal{A} : \mathbb{R}^{d \times d} \rightarrow \mathbb{R}^n$ to be $\mathcal{A}^* : \mathbb{R}^n \rightarrow \mathbb{R}^{d \times d}$, with $\mathcal{A}^*(\epsilon) = \sum_{i=1}^{n} \epsilon_i A_i$. 

As we search within $\|\Delta_U\|_{2,1} \le 2L$ and $\|\Delta_U^{*}\|_{2,1} \le L$, we require the following event to hold
\begin{align}\label{eq:event_I}
\mathcal{I} = \left\{ \|\nu\|_{2,1} \le L \right\}
\end{align}
for $\widetilde{\Delta}_U$ to be feasible. Using a similar analysis to Theorem~\ref{thm:gold_estimator}, we can show the event $\mathcal{I}$ takes place with a high probability
\[
\mathbb{P}(\mathcal{I}) \ge 1 - 2 (36\sqrt{2})^{2r(2d+1)} \exp\left(-\frac{L^2 \sigma_{r}^2(U_p^*) (3\alpha_p - 2\beta_p)^2 n_p}{512\beta\sigma_p^2d}\right)
\]
on the event $\mathcal{I}$, the global optimality of $\widehat{\Delta}_U$ implies 
\begin{align*}
\frac{1}{n_g}\|X_g - \mathcal{A}_g((\widehat{U}_p+\widehat{\Delta}_U)(\widehat{U}_p+\widehat{\Delta}_U)^T)\|^2 + \lambda \|\widehat{\Delta}_U\|_{2, 1}
\le \frac{1}{n_g}\|X_g - \mathcal{A}_g((\widehat{U}_p+\widetilde{\Delta}_U)(\widehat{U}_p+\widetilde{\Delta}_U)^T)\|^2 + \lambda \|\widetilde{\Delta}_U\|_{2, 1}.
\end{align*}
Plugging in $X_g = \mathcal{A}_g((\widehat{U}_p+\widetilde{\Delta}_U)(\widehat{U}_p+\widetilde{\Delta}_U)^T) + \epsilon_g$ yields
\begin{multline*}
\frac{1}{n_g}\|\mathcal{A}_g((\widehat{U}_p+\widehat{\Delta}_U)(\widehat{U}_p+\widehat{\Delta}_U)^T - (\widehat{U}_p+\widetilde{\Delta}_U)(\widehat{U}_p+\widetilde{\Delta}_U)^T)\|^2 + \lambda \|\widehat{\Delta}_U\|_{2, 1} \\
\le \frac{2}{n_g} \langle \epsilon_g, \mathcal{A}_g((\widehat{U}_p+\widehat{\Delta}_U)(\widehat{U}_p+\widehat{\Delta}_U)^T - (\widehat{U}_p+\widetilde{\Delta}_U)(\widehat{U}_p+\widetilde{\Delta}_U)^T)\rangle + \lambda \|\widetilde{\Delta}_U\|_{2, 1}.
\end{multline*}
Rearranging the RHS with $U_g^* = \widehat{U}_p+\widetilde{\Delta}_U$, we get
\begin{multline}\label{eq:basic_ineq_derivation}
\frac{1}{n_g}\|\mathcal{A}_g((\widehat{U}_p+\widehat{\Delta}_U)(\widehat{U}_p+\widehat{\Delta}_U)^T - (\widehat{U}_p+\widetilde{\Delta}_U)(\widehat{U}_p+\widetilde{\Delta}_U)^T)\|^2 + \lambda \|\widehat{\Delta}_U\|_{2, 1} \\
\le \frac{2}{n_g} \langle \epsilon_g, \mathcal{A}_g((\widehat{\Delta}_U - \widetilde{\Delta}_U) U_g^{*T} + U_g^*(\widehat{\Delta}_U - \widetilde{\Delta}_U)^T + (\widehat{\Delta}_U - \widetilde{\Delta}_U)(\widehat{\Delta}_U - \widetilde{\Delta}_U)^T)\rangle + \lambda \|\widetilde{\Delta}_U\|_{2, 1}
\end{multline}
The first part of the first term on the RHS of inequality (\ref{eq:basic_ineq_derivation}) has
\begin{align*}
\langle \epsilon_g, \mathcal{A}_g((\widehat{\Delta}_U - \widetilde{\Delta}_U) U_g^{*T} + U_g^*(\widehat{\Delta}_U - \widetilde{\Delta}_U)^T)\rangle
= & \langle \mathcal{A}_g^*(\epsilon_g), (\widehat{\Delta}_U - \widetilde{\Delta}_U) U_g^{*T} + U_g^*(\widehat{\Delta}_U - \widetilde{\Delta}_U)^T\rangle \\
= & \langle \mathcal{A}_g^*(\epsilon_g)U_g^*, \widehat{\Delta}_U - \widetilde{\Delta}_U \rangle + \langle \mathcal{A}_g^*(\epsilon_g)^TU_g^*, \widehat{\Delta}_U - \widetilde{\Delta}_U \rangle \\
\le & ( \max_{j \in [d]} \|(\mathcal{A}_g^*(\epsilon_g)^j U_g^* \| + \max_{j \in [d]} \| \mathcal{A}_g^*(\epsilon_g)^{Tj} U_g^*\| ) \|\widehat{\Delta}_U-\widetilde{\Delta}_U\|_{2, 1}.
\end{align*}
Correspondingly, the second part of the first term on the RHS of inequality (\ref{eq:basic_ineq_derivation}) has 
\begin{align}
\langle \epsilon_g, \mathcal{A}_g((\widehat{\Delta}_U - \widetilde{\Delta}_U)(\widehat{\Delta}_U - \widetilde{\Delta}_U)^T)\rangle 
= & \langle \mathcal{A}_g^*(\epsilon_g), (\widehat{\Delta}_U - \widetilde{\Delta}_U)(\widehat{\Delta}_U - \widetilde{\Delta}_U)^T\rangle \nonumber \\ 
\le & |\mathcal{A}_g^*(\epsilon_g)|_{\infty} |(\widehat{\Delta}_U - \widetilde{\Delta}_U)(\widehat{\Delta}_U - \widetilde{\Delta}_U)^T|_1 \nonumber \\
\le & |\mathcal{A}_g^*(\epsilon_g)|_{\infty} \|\widehat{\Delta}_U-\widetilde{\Delta}_U\|_{2, 1}^2.
\end{align}

Next, consider the following events 
\[
\mathcal{G}_1 = \left\{ \frac{2}{n_g} \max_{j \in [d]} \| \mathcal{A}_g^*(\epsilon_g)^j U_g^* \| \le \frac{\lambda}{8} \right\}, \quad
\mathcal{G}_2 = \left\{ \frac{2}{n_g} \max_{j \in [d]} \| \mathcal{A}_g^*(\epsilon_g)^{Tj} U_g^* \| \le \frac{\lambda}{8} \right\},
\]
and 
\[
\mathcal{F} = \left\{ \frac{2}{n_g} | \mathcal{A}_g^*(\epsilon_g) |_{\infty} \le \frac{\lambda}{16L} \right\},
\]
which we prove holds with high probability in Lemma~\ref{lemma:G_probbound} after this lemma.
On the events $\mathcal{G}_1$, $\mathcal{G}_2$ and $\mathcal{F}$, we derive from inequality (\ref{eq:basic_ineq_derivation}) that 
\begin{multline*}
\frac{1}{n_g}\|\mathcal{A}_g((\widehat{U}_p+\widehat{\Delta}_U)(\widehat{U}_p+\widehat{\Delta}_U)^T - (\widehat{U}_p+\widetilde{\Delta}_U)(\widehat{U}_p+\widetilde{\Delta}_U)^T)\|^2 + \lambda \|\widehat{\Delta}_U\|_{2, 1} \\
\le \frac{\lambda}{4} \|\widehat{\Delta}_U-\widetilde{\Delta}_U\|_{2, 1} + \frac{\lambda}{16L} \|\widehat{\Delta}_U-\widetilde{\Delta}_U\|_{2, 1}^2 + \lambda \|\widetilde{\Delta}_U\|_{2, 1} \\
\le \frac{\lambda}{2} \|\widehat{\Delta}_U-\widetilde{\Delta}_U\|_{2, 1} + \lambda \|\widetilde{\Delta}_U\|_{2, 1}.
\end{multline*}
The second inequality uses
\begin{align*}
\|\widehat{\Delta}_U-\widetilde{\Delta}_U\|_{2, 1} \le 4L,
\end{align*}
which is derived from the definition of the search region $\|\Delta_U\|_{2, 1} \le 2L$, the definition of event $\mathcal{I}$, and the feasibility of $\Delta_U^*$ that $\|\Delta_U^*\|_{2, 1} \le L$. 
We can further arrange the inequality to get
\begin{multline}\label{eq:basic_ineq_derivation1}
\frac{1}{n_g}\|\mathcal{A}_g((\widehat{U}_p+\widehat{\Delta}_U)(\widehat{U}_p+\widehat{\Delta}_U)^T - (\widehat{U}_p+\widetilde{\Delta}_U)(\widehat{U}_p+\widetilde{\Delta}_U)^T)\|^2 + \frac{\lambda}{2} \sum_{j \in J^c} \|(\widehat{\Delta}_U-\widetilde{\Delta}_U)^{j}\| \\
\le \frac{3\lambda}{2} \sum_{j \in J} \|(\widehat{\Delta}_U-\widetilde{\Delta}_U)^{j}\| + 2\lambda \sum_{j \in J^c} \|\nu^{j}\|.
\end{multline}

Now consider the following two cases respectively: 
\begin{enumerate}[(i).]
\item \label{case:basicineq_1} $\sum_{j \in J^c} \|\nu^{j}\| \le \sum_{j \in J} \|(\widehat{\Delta}_U-\widetilde{\Delta}_U)^{j}\|$,
\item \label{case:basicineq_2} $\sum_{j \in J^c} \|\nu^{j}\|  > \sum_{j \in J} \|(\widehat{\Delta}_U-\widetilde{\Delta}_U)^{j}\|$. 
\end{enumerate}
Under Case (\ref{case:basicineq_1}), we derive from the inequality (\ref{eq:basic_ineq_derivation1}) that
\begin{multline*}
\frac{1}{n_g}\|\mathcal{A}_g((\widehat{U}_p+\widehat{\Delta}_U)(\widehat{U}_p+\widehat{\Delta}_U)^T - (\widehat{U}_p+\widetilde{\Delta}_U)(\widehat{U}_p+\widetilde{\Delta}_U)^T)\|^2 
+ \frac{\lambda}{2} \sum_{j \in J^c} \|(\widehat{\Delta}_U-\widetilde{\Delta}_U)^{j}\|
\le \frac{7\lambda}{2} \sum_{j \in J} \|(\widehat{\Delta}_U-\widetilde{\Delta}_U)^{j}\|.
\end{multline*}
Thus, we have $\sum_{j \in J^c} \|(\widehat{\Delta}_U-\widetilde{\Delta}_U)^{j}\| \le 7 \sum_{j \in J} \|(\widehat{\Delta}_U-\widetilde{\Delta}_U)^{j}\|$ and $\mathcal{A}_g$ satisfies QCC. Further write the above as
\begin{multline*}
\frac{1}{n_g}\|\mathcal{A}_g((\widehat{U}_p +\widehat{\Delta}_U)(\widehat{U}_p+\widehat{\Delta}_U)^T - (\widehat{U}_p+\widetilde{\Delta}_U)(\widehat{U}_p+\widetilde{\Delta}_U)^T)\|^2
 + \frac{\lambda}{2} \|\widehat{\Delta}_U - \widetilde{\Delta}_U\|_{2, 1}
\le \frac{8\lambda^2 s}{\kappa} + \frac{\kappa}{2s} ( \sum_{j \in J} \|(\widehat{\Delta}_U-\widetilde{\Delta}_U)^{j}\| )^2,
\end{multline*}
where we use the inequality $2ab \le a^2 + b^2$. 
Apply QCC to the RHS, and
\begin{align*}
\frac{1}{2n_g}\|\mathcal{A}_g((\widehat{U}_p +\widehat{\Delta}_U)(\widehat{U}_p+\widehat{\Delta}_U)^T - (\widehat{U}_p+\widetilde{\Delta}_U)(\widehat{U}_p+\widetilde{\Delta}_U)^T)\|^2
 + \frac{\lambda}{2} \|\widehat{\Delta}_U - \widetilde{\Delta}_U\|_{2, 1} 
\le \frac{8\lambda^2 s}{\kappa}
\end{align*}
Under Case (\ref{case:basicineq_2}), the inequality (\ref{eq:basic_ineq_derivation1}) gives
\[
\frac{1}{n_g}\|\mathcal{A}((\widehat{U}_p+\widehat{\Delta}_U)(\widehat{U}_p+\widehat{\Delta}_U)^T - (\widehat{U}_p+\widetilde{\Delta}_U)(\widehat{U}_p+\widetilde{\Delta}_U)^T)\|^2 
+ \frac{\lambda}{2} \|\widehat{\Delta}_U-\widetilde{\Delta}_U\|_{2, 1} \le 4\lambda \sum_{j \in J^c} \|\nu^{j}\|.
\]
Therefore, under any circumstances, we have 
\[
\|\widehat{\Delta}_U - \widetilde{\Delta}_U\|_{2, 1} 
 \le 8(\frac{2\lambda s}{\kappa} + \sum_{j \in J^c} \|\nu^{j}\|)
 \le 8(\frac{2\lambda s}{\kappa} + \|\nu\|_{2, 1}).
 \]
Consider the event 
\begin{align}\label{eq:event_J}
\mathcal{J} = \left\{ \|\nu\|_{2, 1} \le \frac{8 \sqrt{d} c}{(3\alpha_p - 2\beta_p)\sigma_{r}(U_p^*)} \right\}.
\end{align}
Using a similar analysis to Theorem~\ref{thm:gold_estimator} as our analysis on event $\mathcal{I}$, we have
\[
\mathbb{P}(\mathcal{J}) \ge 1 - 2 (36\sqrt{2})^{2r(2d+1)} \exp(-\frac{c^2n_p}{8\beta_p\sigma_p^2}).
\]
Therefore, on the event $\mathcal{J}$, the estimation error is bounded by 
\[
\|\widehat{\Delta}_U - \widetilde{\Delta}_U\|_{2, 1} 
 \le 16(\frac{\lambda s}{\kappa} + \frac{4 \sqrt{d} c}{(3\alpha_p - 2\beta_p)\sigma_{r}(U_p^*)}).
\]
Combining all the above and using Lemma~\ref{lemma:G_probbound}, we have the following concentration inequality
\begin{multline}\label{eq:joint_concen}
\mathbb{P}\left( \|\widehat{\Delta}_U - \widetilde{\Delta}_U\|_{2, 1} \ge 16(\frac{\lambda s}{\kappa} + \frac{4 \sqrt{d} c}{(3\alpha_p - 2\beta_p)\sigma_{r}(U_p^*)}) \right) \\
\le \mathbb{P}(\mathcal{I}^c) + \mathbb{P}(\mathcal{F}^c) + \mathbb{P}(\mathcal{G}_1^c) + \mathbb{P}(\mathcal{G}_2^c) + \mathbb{P}(\mathcal{J}^c) \\
\le 2 (36\sqrt{2})^{2r(2d+1)} \exp(-\frac{L^2 \sigma_{r}^2(U_p^*) (3\alpha_p - 2\beta_p)^2 n_p}{512\beta_p\sigma_p^2d}) \\
 + 2 d^2 \exp\left(-\frac{\lambda^2 n_g}{2048 L^2 \sigma_g^2 (\max_{l,k}\|A_g^{lk}\|^2/n_g)}\right) \\
 + d \max_{j \in [d]} \exp \left( -( \sqrt{\frac{ \frac{\lambda^2n_g}{256\sigma_g^2} -  (\tr(\Psi_j) - \frac{\|\Psi_j\|_F^2}{2 \| \Psi_j \|})}{2 \| \Psi_j \|}} - \frac{\|\Psi_j\|_F}{2\|\Psi_j\|} )^2 \right) \\
 + d \max_{j \in [d]} \exp \left( -( \sqrt{\frac{ \frac{\lambda^2n_g}{256\sigma_g^2} -  (\tr(\Phi_j) - \frac{\|\Phi_j\|_F^2}{2 \| \Phi_j \|})}{2 \| \Phi_j \|}} - \frac{\|\Phi_j\|_F}{2\|\Phi_j\|} )^2 \right) \\
 + 2 (36\sqrt{2})^{2r(2d+1)} \exp(-\frac{c^2n_p}{8\beta_p\sigma_p^2}). \Halmos
\end{multline}
\end{proof}

\begin{lemma} \label{lemma:G_probbound}
The events $\mathcal{G}_1$, $\mathcal{G}_2$ and $\mathcal{F}$ satisfy the following concentration inequalities
\begin{align*}
\mathbb{P}(\mathcal{G}_1^c) 
\le d \max_{j \in [d]} \exp \left( -( \sqrt{\frac{ \frac{\lambda^2n_g}{256\sigma_g^2} -  (\tr(\Psi_j) - \frac{\|\Psi_j\|_F^2}{2 \| \Psi_j \|})}{2 \| \Psi_j \|}} - \frac{\|\Psi_j\|_F}{2\|\Psi_j\|} )^2 \right),
\end{align*}
\begin{align*}
\mathbb{P}(\mathcal{G}_2^c) 
\le d \max_{j \in [d]} \exp \left( -( \sqrt{\frac{ \frac{\lambda^2n_g^2}{256\sigma_g^2} -  (\tr(\Phi_j) - \frac{\|\Phi_j\|_F^2}{2 \| \Phi_j \|})}{2 \| \Phi_j \|}} - \frac{\|\Phi_j\|_F}{2\|\Phi_j\|} )^2 \right),
\end{align*}
and
\begin{align*}
\mathbb{P}(\mathcal{F}^c) \le 2 d^2 \exp\left(-\frac{\lambda^2 n_g}{2048 L^2 \sigma_g^2 (\max_{l,k}\|A_g^{lk}\|^2/n_g)}\right).
\end{align*}
\end{lemma}

\begin{proof} {Proof of Lemma~\ref{lemma:G_probbound}}
Consider the event $\mathcal{F}$ first. With $\epsilon_g$ being $\sigma_g$-subgaussian, 
\begin{align*}
\mathbb{P}(\mathcal{F}^c) & = \mathbb{P}(\frac{2}{n_g} | \mathcal{A}_g^*(\epsilon_g) |_{\infty} \ge \frac{\lambda}{16L}) \\
& \le d^2 \max_{l,k \in [d]} \mathbb{P}(\frac{2}{n_g} | \sum_{i=1}^{n_g} A_{g,i}^{(l,k)} \epsilon_{g, i}| \ge \frac{\lambda}{16L}) \\
& \le 2 d^2 \exp\left(-\frac{\lambda^2 n_g}{2048 L^2 \sigma_g^2 (\max_{l,k}\|A_g^{lk}\|^2/n_g)}\right),
\end{align*}
In the last inequality, we use the fact that $\epsilon_g$ is $\sigma_g$-subgaussian in the final inequality. 

Next, we look at the event $\mathcal{G}_1$. 
\begin{align*}
\mathbb{P}(\mathcal{G}_1^c) & = \mathbb{P}(\frac{2}{n_g} \max_{j \in [d]} \| \mathcal{A}_g^*(\epsilon_g)^j U_g^* \| \ge \frac{\lambda}{8}) 
\le d \max_{j \in [d]} \mathbb{P}(\frac{2}{n_g} \| \mathcal{A}_g^*(\epsilon_g)^j U_g^* \| \ge \frac{\lambda}{8}).
\end{align*}
For a given $j$, observe that 
\begin{align*}
\frac{4}{n_g^2} \| \mathcal{A}_g^*(\epsilon_g)^j U_g^* \|^2 
= \frac{4}{n_g^2} \| \sum_{i=1}^{n_g} A_{g,i}^{j} U_g^* \epsilon_{g, i} \|^2 
= \frac{4}{n_g}  \epsilon_{g}^T  \frac{A_{g}^{j} U_g^{*} U_g^{*T} A_{g}^{jT}}{n_g} \epsilon_{g},
\end{align*}
Note that $\Psi_j$ has the same positive eigenvalues as $\frac{A_{g}^{j} U_g^{*} U_g^{*T} A_{g}^{jT}}{n_g}$.
Different from \citet{lounici2011oracle}, we assume subgaussian random noises instead of gaussian noises. Therefore, instead, we have from Lemma~\ref{lemma:subgaussiansquare}
\[
\mathbb{P}(\frac{4}{n_g}  \epsilon_{g}^T  \frac{A_{g}^{j} U_g^{*} U_g^{*T} A_{g}^{jT}}{n_g} \epsilon_{g} \ge \frac{\lambda^2}{64}) \le \exp 
\left( - ( \sqrt{\frac{ \frac{\lambda^2n_g}{256\sigma_g^2} -  (\tr(\Psi_j) - \frac{\|\Psi_j\|_F^2}{2 \| \Psi_j \|})}{2 \| \Psi_j \|}} - \frac{\|\Psi_j\|_F}{2\|\Psi_j\|} )^2 \right).
\]
Combining the results above, we can derive that
\begin{align*}
\mathbb{P}(\mathcal{G}_1^c) 
\le d \max_{j \in [d]} \exp \left( -( \sqrt{\frac{ \frac{\lambda^2n_g}{256\sigma_g^2} -  (\tr(\Psi_j) - \frac{\|\Psi_j\|_F^2}{2 \| \Psi_j \|})}{2 \| \Psi_j \|}} - \frac{\|\Psi_j\|_F}{2\|\Psi_j\|} )^2 \right).
\end{align*}

Similarly for event $\mathcal{G}_2$, we have
\begin{align*}
\mathbb{P}(\mathcal{G}_2^c) 
\le d \max_{j \in [d]} \exp \left( -( \sqrt{\frac{ \frac{\lambda^2n_g^2}{256\sigma_g^2} -  (\tr(\Phi_j) - \frac{\|\Phi_j\|_F^2}{2 \| \Phi_j \|})}{2 \| \Phi_j \|}} - \frac{\|\Phi_j\|_F}{2\|\Phi_j\|} )^2 \right). \Halmos
\end{align*}
\end{proof}

\begin{proof} {Proof of Theorem~\ref{thm:joint_estimator}}
Theorem~\ref{thm:joint_estimator} follows Lemma~\ref{lem:joint_estimator_concenineq}. 
Suppose $\frac{L \sigma_{r}(U_p^*) (3\alpha_p - 2\beta_p) }{8\sqrt{d}} \ge c$. On this event, the first term on the RHS of inequality (\ref{eq:joint_concen}) is smaller than the last term on the RHS. In order to make each term on the RHS to be smaller than $\frac{\delta}{5}$, we require
\begin{align*}
\lambda \ge \max & \left\{ \sqrt{\frac{2048 L^2 \sigma_g^2 (\max_{l,k}\|A_g^{lk}\|^2/n_g)}{n_g} \log(\frac{10d^2}{\delta})}, \right. \\
& \left. \max_{j \in [d]} \sqrt{\frac{256 \sigma_g^2}{n_g} (\tr(\Psi_j) + 2 \|\Psi_j\|_F \sqrt{\log(\frac{5d}{\delta})} + 2 \|\Psi_j\| \log(\frac{5d}{\delta}))}, \right. \\
& \left. \max_{j \in [d]} \sqrt{\frac{256 \sigma_g^2}{n_g} (\tr(\Phi_j) + 2 \|\Phi_j\|_F \sqrt{\log(\frac{5d}{\delta})} + 2 \|\Phi_j\| \log(\frac{5d}{\delta}))} \right\},
\end{align*}
and let $c$ take the value
\begin{align*}
c = \sqrt{\frac{8\beta_p\sigma_p^2}{n_p} (2r(2d+1)\log(36\sqrt{2}) + \log(\frac{10}{\delta}))}.
\end{align*}

Note that by definition of $1$-smoothness$(\beta_g)$
\begin{align*}
\frac{1}{n_g}\|A_g^{lk}\|^2 = \langle E_{lk}, \frac{1}{n_g}\mathcal{A}_g^*(\mathcal{A}_g(E_{lk})) \rangle \le \beta_g,
\end{align*}
where $E_{lk} \in \mathbb{R}^{d \times d}$ is a basis matrix whose $(l, k)$ entry is 1 and otherwise 0.
On the other hand, 
\begin{align*}
\|\Psi_j\| & = \max_{\|x\|=1, x\in \mathbb{R}^r} x^T U_g^{*T} \frac{A_{g}^{jT} A_{g}^{j}}{n_g} U_g^{*} x 
= \max_{\|x\|=1} x^T U_g^{*T} \frac{A_{g}^{jT} A_{g}^{j}}{n_g} U_g^{*} x.
\end{align*}
If we define a matrix $E_j(x)$ whose $j^{\text{th}}$ row is $x^T$ and otherwise 0, then 
\begin{align*}
\|\Psi_j\|
& = \max_{\|x\|=1} \frac{1}{n_g} \langle E_j(U_g^{*} x), A_{g}^*(A_{g}(E_j(U_g^{*} x))) \rangle.
\end{align*}
As $\|x\|=1$, we have
\[
\|E_j(U_g^{*} x)\|_F = \|U_g^{*} x\| \le \sigma_1(U_g^{*}).
\]
Therefore, we have
\begin{align*}
\|\Psi_j\|
& \le \max_{\|R\|_F \le \sigma_1(U_g^{*})} \frac{1}{n_g} \langle R, A_{g}^*(A_{g}(R) \rangle 
\le \beta_g \sigma_1^2(U_g^{*}).
\end{align*}
With a similar analysis, we have
\begin{align*}
\|\Phi_j\| & \le \beta_g \sigma_1^2(U_g^{*}).
\end{align*}
Given the above results, we can bound the trace and Frobenius norm of $\Psi_j$ and $\Phi_j$ proportional to their rank:
\begin{gather*}
\tr(\Psi_j) \le r \|\Psi_j\| \le r \beta_g \sigma_1^2(U_g^{*}), \quad \|\Psi_j\|_F \le \sqrt{r} \|\Psi_j\| \le \sqrt{r} \beta_g \sigma_1^2(U_g^{*}) \\
\tr(\Phi_j) \le r \beta_g \sigma_1^2(U_g^{*}), \quad \|\Phi_j\|_F \le \sqrt{r} \beta_g \sigma_1^2(U_g^{*}).
\end{gather*}

Combining all the above results, we can instead set $\lambda$ as:
\begin{align*}
\lambda = \max & \left\{ \sqrt{\frac{2048 L^2 \beta_g \sigma_g^2 }{n_g} \log(\frac{10d^2}{\delta})}, \right. \\
& \left. \sqrt{\frac{256 \beta_g \sigma_g^2 \sigma_1^2(U_g^{*}) }{n_g} (r + 2 \sqrt{r\log(\frac{5d}{\delta})} + 2 \log(\frac{5d}{\delta}))} \right\}.
\end{align*}

Therefore, with the above choice of $\lambda$ and with $n_p$ and $d$ such that 
\[
\sqrt{\frac{8\beta_p\sigma_p^2}{n_p} (2r(2d+1)\log(36\sqrt{2}) + \log(\frac{10}{\delta}))} \le \frac{L \sigma_{r}(U_p^*) (3\alpha_p - 2\beta_p) }{8\sqrt{d}},
\]
we obtain the following bound for estimation error of $\widetilde{\Delta}_U$:
\[
\|\widehat{\Delta}_U - \widetilde{\Delta}_U\|_{2, 1} =  \mathcal{O}\left(\sqrt{\frac{\sigma_g^2s^2(r+\log(\frac{d^2}{\delta}))}{n_g}} + \sqrt{\frac{\sigma_p^2(rd^2 + d \log(\frac{1}{\delta}))}{n_p}}\right), 
\]
with probability greater than $1 - \delta$. Consequently, 
\[
\ell(\widehat{U}_g^{TL},U_g^*) = 
\mathcal{O}\left(\sqrt{\frac{\sigma_g^2s^2(r+\log(\frac{d^2}{\delta}))}{n_g}} + \sqrt{\frac{\sigma_p^2(rd^2 + d \log(\frac{1}{\delta}))}{n_p}}\right), 
\]
with probability at least $1 - \delta$. \Halmos
\end{proof}

\section{Local Minima}
\label{sec:localminima}

\begin{proof} {Proof of Proposition~\ref{thm:rsc_local}}
By definition, 
\begin{align*}
\mathbb{E}_{X_g \mid \mathcal{A}_g}[\langle \nabla f(\widetilde{\Delta}_U + \Delta) - \nabla f(\widetilde{\Delta}_U), \Delta \rangle] = \frac{2}{n_g} \mathcal{A}_g(\Delta U_g^{*T} + U_g^*\Delta^T + \Delta \Delta^T)^T \mathcal{A}_g(\Delta U_g^{*T} + U_g^*\Delta^T + 2\Delta \Delta^T).
\end{align*}
As $\mathcal{A}_g$ satisfies \textsf{RSC}$(\sqrt{\frac{2}{3}}U_g^*, \eta, \tau)$, 
\begin{align*}
& \frac{2}{n_g} \mathcal{A}_g(\Delta U_g^{*T} + U_g^*\Delta^T + \Delta \Delta^T)^T \mathcal{A}_g(\Delta U_g^{*T} + U_g^*\Delta^T + 2\Delta \Delta^T) \\
= & \frac{2}{n_g} (\|\mathcal{A}_g(\Delta U_g^{*T} + U_g^*\Delta^T + \frac{3}{2}\Delta \Delta^T)\|^2 - \frac{1}{4} \|\mathcal{A}_g(\Delta \Delta^T)\|^2) \\
\ge & 2\eta \|\Delta U_g^{*T} + U_g^*\Delta^T + \frac{3}{2}\Delta \Delta^T\|_F^2 - \frac{\beta_g}{2} \|\Delta \Delta^T\|_F^2 - \frac{3}{2} \tau(n,d,r) \|\Delta\|_{2, 1}^2.
\end{align*}
The first two terms of the above have
\begin{align*}
& 2\eta \|\Delta U_g^{*T} + U_g^*\Delta^T + \frac{3}{2}\Delta \Delta^T\|_F^2 - \frac{\beta_g}{2} \|\Delta \Delta^T\|_F^2 \\
\ge & 8\eta \|\Delta U_g^{*T}\|_F^2 - 12 \eta \langle \Delta U_g^{*T}, \Delta \Delta^T \rangle + \frac{9\eta}{2}\|\Delta \Delta^T\|_F^2 - \frac{\beta_g}{2} \|\Delta \Delta^T\|_F^2 \\
\ge & 8\eta \|\Delta U_g^{*T}\|_F^2 - 12 \eta \sup_{\Delta} \frac{\|\Delta \Delta^T\|_F}{\|\Delta U_g^{*T}\|_F} \|\Delta U_g^{*T}\|_F^2  + \frac{9\eta}{2}\|\Delta \Delta^T\|_F^2 - \frac{\beta_g}{2} \|\Delta \Delta^T\|_F^2.
\end{align*}
When $\rho \le \sigma_r(U_g^*)/3$, it holds that
\begin{align*}
\sup_{\Delta} \frac{\|\Delta \Delta^T\|_F}{\|\Delta U_g^{*T}\|_F} \le \frac{\|\Delta\|_F}{\sigma_r(U_g^*)} \le \frac{1}{3}.
\end{align*}
Therefore, we can lower bound the first two terms with
\begin{align*}
& 2\eta \|\Delta U_g^{*T} + U_g^*\Delta^T + \frac{3}{2}\Delta \Delta^T\|_F^2 - \frac{\beta_g}{2} \|\Delta \Delta^T\|_F^2 \ge 4\eta \|\Delta U_g^{*T}\|_F^2,
\end{align*}
where we use the assumption that $9\eta \ge \beta_g$. Combining all of the above results gives rise to our first statement. 

On the other hand, when $\|\Delta\|_F \le \rho$,
\begin{align*}
\mathbb{E}_{X_g \mid \mathcal{A}_g}[\langle \nabla f(\widetilde{\Delta}_U + \Delta) - \nabla f(\widetilde{\Delta}_U), \Delta \rangle] & = \frac{2}{n_g} (\|\mathcal{A}_g(\Delta U_g^{*T} + U_g^*\Delta^T + \frac{3}{2}\Delta \Delta^T)\|^2 - \frac{1}{4} \|\mathcal{A}_g(\Delta \Delta^T)\|^2) \\
& \ge \eta_1 \|\Delta\|_F^2 - \tau_1(n_g, d, r) \|\Delta\|_{2, 1}^2.
\end{align*}
Therefore, we have
\begin{align*}
\frac{2}{n_g} \|\mathcal{A}_g(\Delta U_g^{*T} + U_g^*\Delta^T + \frac{3}{2}\Delta \Delta^T)\|^2 \ge \eta_1 \|\Delta\|_F^2 - \tau_1(n_g, d, r) \|\Delta\|_{2, 1}^2.
\end{align*}
Note that
\begin{align*}
\|\Delta U_g^{*T} + U_g^*\Delta^T + \frac{3}{2}\Delta \Delta^T\|_F^2 = & \|\Delta U_g^{*T} + U_g^*\Delta^T\|_F^2 + \|\frac{3}{2}\Delta \Delta^T\|_F^2 + 6\langle U_g^*\Delta^T, \Delta \Delta^T \rangle \\
= & 4\|\Delta U_g^{*T}\|_F^2 + \|\frac{3}{2}\Delta \Delta^T\|_F^2 + 6\langle U_g^*\Delta^T, \Delta \Delta^T \rangle \\
\le & 4\|\Delta U_g^{*T}\|_F^2 + \|\frac{3}{2}\Delta \Delta^T\|_F^2 + 6\|U_g^*\Delta^T\|_F \|\Delta \Delta^T\|_F \\
\le & 4\sigma_1(U_g^*)^2\|\Delta\|_F^2 + \frac{9\rho^2}{4} \|\Delta\|_F^2 + 6\sigma_1(U_g^*)\rho\|\Delta\|_F^2 \\
= & (2\sigma_1(U_g^*) + \frac{3\rho}{2})^2\|\Delta\|_F^2.
\end{align*}
Therefore, $\mathcal{A}_g$ satisfies \textsf{RSC}$(\sqrt{\frac{2}{3}}U_g^*, \eta, \tau)$ with $\eta = \frac{\eta_1}{2(2\sigma_1(U_g^*) + 3\rho/2)^2}$ and $\tau=\tau_1/3$. \Halmos
\end{proof}

\begin{proof}{Proof of Theorem~\ref{thm:joint_estimator_local}}
Let $\Delta = \widehat{\Delta}_U - \widetilde{\Delta}_U$. We first show that the local minima fall within $\|\Delta\|_F \le \rho$ with high probability. If $\|\Delta\|_F \ge \rho$, condition~(\ref{eq:rsc_local_2}) gives
\begin{align}\label{eq:rsc_local_2_specific}
\mathbb{E}_{X_g\mid\mathcal{A}_g}[\langle \nabla f(\widehat{\Delta}_U) - \nabla f(\widetilde{\Delta}_U), \Delta \rangle] & = \langle \nabla f(\widehat{\Delta}_U) - \nabla f(\widetilde{\Delta}_U), \Delta \rangle + \frac{4}{n_g} \epsilon_g^T \mathcal{A}_g(\Delta\Delta^T) \nonumber \\
& \ge \eta_2 \|\Delta\|_F - \tau_2(n_g,d,r) \|\Delta\|_{2, 1}.
\end{align}
As $\widehat{\Delta}_U$ is a local minimum, a necessary condition of the optimization problem is
\begin{align}\label{eq:optimum_neccond}
\langle \nabla f(\widehat{\Delta}_U) + \lambda \partial \|\widehat{\Delta}_U\|_{2,1}, \Delta_U - \widehat{\Delta}_U \rangle \ge 0,
\end{align}
for any $\Delta_U$ within the search area. $\partial \|\widehat{\Delta}_U\|_{2,1}$ is the subgradient of the $\ell_{2,1}$ norm at $\widehat{\Delta}_U$, i.e.,
\[
\partial \|X\|_{2,1}
\begin{cases}
= \nabla \|X\|_{2,1}, & \|X^(j, \cdot)\| > 0, \forall j \in [d] \\
\in \{Z \mid \|Z\|_{2, \infty} \le 1\}, & \text{otherwise}.
\end{cases}
\]
The combination of (\ref{eq:rsc_local_2_specific}) and (\ref{eq:optimum_neccond}) implies
\begin{align}\label{eq:rsc_local_2_specific_de}
\langle - \nabla f(\widetilde{\Delta}_U) - \lambda \partial \|\widehat{\Delta}_U\|_{2,1}, \Delta \rangle + \frac{4}{n_g} \epsilon_g^T \mathcal{A}_g(\Delta\Delta^T) \ge \eta_2 \|\Delta\|_F - \tau_2(n_g,d,r) \|\Delta\|_{2, 1}.
\end{align}
Using H\"older's inequality, we have
\begin{align*}
\langle \lambda \partial \|\widehat{\Delta}_U\|_{2,1}, \Delta \rangle \le \lambda \|\partial \|\widehat{\Delta}_U\|_{2,1}\|_{2, \infty} \|\Delta\|_{2, 1} \le \lambda \|\Delta\|_{2, 1},
\end{align*}
where we use $\|\partial \|\widehat{\Delta}_U\|_{2,1}\|_{2, \infty} \le 1$. 
Using the above result, replacing $\nabla f(\widetilde{\Delta}_U)$, and rearranging inequality (\ref{eq:rsc_local_2_specific_de}) give
\begin{align*}
\eta_2 \|\Delta\|_F - \tau_2 (\sqrt{\frac{r}{n_g}} + \sqrt{\frac{\log d}{n_g}}) \|\Delta\|_{2, 1}
\le \frac{2}{n_g} \langle \epsilon_g, \mathcal{A}_g(\Delta U_g^{*T} + U_g^* \Delta^T + 2\Delta\Delta^T)\rangle  + \lambda \|\Delta\|_{2, 1}.
\end{align*}
Consider the same event of $\mathcal{I}$ in (\ref{eq:event_I}) and the following events
\begin{align*}
\bar{\mathcal{G}}_1 = \left\{ \frac{2}{n_g} \max_{j \in [d]} \| \mathcal{A}_g^*(\epsilon_g)^j U_g^* \| \le \frac{\lambda}{16} \right\}, \quad
\bar{\mathcal{G}}_2 = \left\{ \frac{2}{n_g} \max_{j \in [d]} \| \mathcal{A}_g^*(\epsilon_g)^{Tj} U_g^* \| \le \frac{\lambda}{16} \right\},
\end{align*}
and
\[
\bar{\mathcal{F}} = \left\{ \frac{4}{n_g} | \mathcal{A}_g^*(\epsilon_g) |_{\infty} \le \frac{\lambda}{32L} \right\}.
\]
We know from Lemma~\ref{lemma:G_probbound} these events hold with high probability. On the event $\mathcal{I}$, we further have $\|\Delta\|_{2,1} \le 4L$. Therefore, under all these events and assuming that $\lambda \ge \frac{4}{3}\tau_2(n_g,d,r)$, we have
\begin{align*}
\eta_2 \|\Delta\|_F \le 2 \lambda \|\Delta\|_{2, 1} \le 8\lambda L.
\end{align*}
With $\lambda \le (\rho \eta_2)/(8L)$, we have $\|\Delta\|_F \le \rho$, which is a contradiction. 

Consequently, we only need to consider $\|\Delta\|_F \le \rho$. Condition~(\ref{eq:rsc_local_1}) gives
\begin{align} \label{eq:rsc_local_1_specific}
\langle \nabla f(\widehat{\Delta}_U) - \nabla f(\widetilde{\Delta}_U), \widehat{\Delta}_U - \widetilde{\Delta}_U \rangle + \frac{4}{n_g} \epsilon_g^T \mathcal{A}_g(\Delta\Delta^T) \ge \eta_1 \|\Delta\|_F^2
- \tau_1(n_g, d, r) \|\Delta\|_{2, 1}^2.
\end{align}
Since the $\ell_{2,1}$ norm is convex, we have for any $\Delta_U$
\begin{align}\label{eq:convex_ineq}
\langle \partial \|\widehat{\Delta}_U\|_{2,1}, \Delta_U - \widehat{\Delta}_U \rangle \le \|\Delta_U\|_{2,1} - \|\widehat{\Delta}_U\|_{2,1}.
\end{align}
Combining inequalities (\ref{eq:rsc_local_1_specific}), (\ref{eq:optimum_neccond}), and (\ref{eq:convex_ineq}), we have 
\begin{multline*}
\eta_1 \|\Delta \|_F^2 \le \frac{2}{n_g} \langle \epsilon_g, \mathcal{A}_g(\Delta U_g^{*T} + U_g^* \Delta^T + 2\Delta\Delta^T)\rangle + \lambda (\|\widetilde{\Delta}_U\|_{2,1} - \|\widehat{\Delta}_U\|_{2,1}) + 4L\tau_1(n_g, d, r) \|\Delta\|_{2,1},
\end{multline*}
where we use $\|\Delta\|_{2,1} \le 4L$. Since $\lambda \ge 16 L \tau_1(n_g, d, r)$, 
we can derive that
\begin{align*}
\eta_1 \|\Delta\|_F^2 \le \frac{\lambda}{2} \|\widetilde{\Delta}_U - \widehat{\Delta}_U\|_{2,1} + \lambda (\|\widetilde{\Delta}_U\|_{2,1} - \|\widehat{\Delta}_U\|_{2,1}),
\end{align*}
on the events $\bar{\mathcal{G}}_1$, $\bar{\mathcal{G}}_2$ and $\bar{\mathcal{F}}$. 
Further arrange the inequality and we have
\begin{align}\label{eq:local_ineq_final}
\eta_1 \|\widehat{\Delta}_U-\widetilde{\Delta}_U\|_F^2 + \frac{\lambda}{2} \sum_{j \in J^c} \|(\widehat{\Delta}_U-\widetilde{\Delta}_U)^{j}\| \le \frac{3\lambda}{2} \sum_{j \in J} \|(\widehat{\Delta}_U-\widetilde{\Delta}_U)^{j}\| + 2\lambda \sum_{j \in J^c} \|\nu^{j}\|.
\end{align}
Inequality~(\ref{eq:local_ineq_final}) gives us
\begin{align*}
\|\widehat{\Delta}_U-\widetilde{\Delta}_U\|_F \le \max\left\{\frac{3\lambda \sqrt{s}}{\eta_1}, \sqrt{\frac{4\lambda \sum_{j \in J^c} \|\nu^{j}\|}{\eta_1}} \right\},
\end{align*}
and
\begin{align*}
\|\widehat{\Delta}_U-\widetilde{\Delta}_U\|_{2,1} \le 4\sqrt{s} \|\widehat{\Delta}_U-\widetilde{\Delta}_U\|_F + 4 \sum_{j \in J^c} \|\nu^{j}\|,
\end{align*}
which implies
\begin{align*}
\|\widehat{\Delta}_U-\widetilde{\Delta}_U\|_{2,1} & \le \max\left\{\frac{12\lambda s}{\eta_1} + 4 \sum_{j \in J^c} \|\nu^{j}\|, 8\sqrt{\frac{\lambda s \sum_{j \in J^c} \|\nu^{j}\|}{\eta_1}} + 4 \sum_{j \in J^c} \|\nu^{j}\| \right\} \le \frac{12\lambda s}{\eta_1}+ 6 \sum_{j \in J^c} \|\nu^{j}\|.
\end{align*}
Under the same event $\mathcal{J}$ in equation (\ref{eq:event_J}), we have that
\begin{align*}
\|\widehat{\Delta}_U-\widetilde{\Delta}_U\|_{2,1} 
\le 12(\frac{\lambda s}{\eta_1} + \frac{4 \sqrt{d} c}{(3\alpha_p - 2\beta_p)\sigma_{r}(U_p^*)}).
\end{align*}
The final concentration inequality is as follows:
\begin{multline}\label{eq:joint_concen_local}
\mathbb{P}\left( \|\widehat{\Delta}_U - \widetilde{\Delta}_U\|_{2, 1} \ge 12(\frac{\lambda s}{\eta_1} + \frac{4 \sqrt{d} c}{(3\alpha_p - 2\beta_p)\sigma_{r}(U_p^*)}) \right) \\
\le \mathbb{P}(\mathcal{I}^c) + \mathbb{P}(\bar{\mathcal{G}}_1^c) + \mathbb{P}(\bar{\mathcal{G}}_2^c) + \mathbb{P}(\bar{\mathcal{F}}^c) + \mathbb{P}(\mathcal{J}^c) \\
\le 2 (36\sqrt{2})^{2r(2d+1)} \exp(-\frac{L^2 \sigma_{r}^2(U_p^*) (3\alpha_p - 2\beta_p)^2 n_p}{512\beta_p\sigma_p^2d}) \\
 + 2 d^2 \exp\left(-\frac{\lambda^2 n_g}{32768 L^2 \sigma_g^2 (\max_{l,k}\|A_g^{lk}\|^2/n_g)}\right) \\
 + d \max_{j \in [d]} \exp \left( -( \sqrt{\frac{ \frac{\lambda^2n_g}{512\sigma_g^2} -  (\tr(\Psi_j) - \frac{\|\Psi_j\|_F^2}{2 \| \Psi_j \|})}{2 \| \Psi_j \|}} - \frac{\|\Psi_j\|_F}{2\|\Psi_j\|} )^2 \right) \\
 + d \max_{j \in [d]} \exp \left( -( \sqrt{\frac{ \frac{\lambda^2n_g}{512\sigma_g^2} -  (\tr(\Phi_j) - \frac{\|\Phi_j\|_F^2}{2 \| \Phi_j \|})}{2 \| \Phi_j \|}} - \frac{\|\Phi_j\|_F}{2\|\Phi_j\|} )^2 \right) \\
 + 2 (36\sqrt{2})^{2r(2d+1)} \exp(-\frac{c^2n_p}{8\beta_p\sigma_p^2}). 
\end{multline}

Following a similar analysis in the proof of Theorem~\ref{thm:joint_estimator}, set $\lambda$ as
\begin{multline*}
\lambda = \max \left\{ \sqrt{\frac{32768 L^2 \beta_g \sigma_g^2 }{n_g} \log(\frac{10d^2}{\delta})}, \sqrt{\frac{512 \beta_g \sigma_g^2 \sigma_1^2(U_g^{*}) }{n_g} (r + 2 \sqrt{r\log(\frac{5d}{\delta})} + 2 \log(\frac{5d}{\delta}))}, \right.\\
\left. \frac{4}{3}\tau_2(n_g, d, r), 16 L \tau_1(n_g, d, r) \right\},
\end{multline*}
and take 
\begin{align*}
c = \sqrt{\frac{8\beta_p\sigma_p^2}{n_p} (2r(2d+1)\log(36\sqrt{2}) + \log(\frac{10}{\delta}))}.
\end{align*}
Then, given
\[
\sqrt{\frac{8\beta_p\sigma_p^2}{n_p} (2r(2d+1)\log(36\sqrt{2}) + \log(\frac{10}{\delta}))} \le \frac{L \sigma_{r}(U_p^*) (3\alpha_p - 2\beta_p) }{8\sqrt{d}},
\]
we obtain statistical guarantees on all local minima
\[
\ell(\widehat{U}_g^{TL},U_g^*) = 
\mathcal{O}\left(\sqrt{\frac{\sigma_g^2s^2(r+\log(\frac{d^2}{\delta}))}{n_g}} + \sqrt{\frac{\sigma_p^2(rd^2 + d \log(\frac{1}{\delta}))}{n_p}}\right), 
\]
with probability at least $1 - \delta$. 
This provides us with a same bound as the global minimum. \Halmos
\end{proof}

\section{Error Bound of Gold Estimator}
\label{sec:thmgoldproof}

Before discussing the estimation error bound of the gold estimator, we first introduce a lemma that helps with our proof.
\begin{lemma}\label{lemma:gold_concenineq}
Let $\mathcal{Z} \subset \mathbb{R}^{d \times d}$ be the subspace of matrices with rank at most $r$. The operator $\mathcal{A}$ is $r$-smooth($\beta$) in $\mathcal{Z}$ and $\epsilon$ is $\sigma$-subgaussian. Then, we have
\[
\mathbb{P}(\sup_{Z \in \mathcal{Z}} | \frac{1}{n} \sum_{i\in[n]} \epsilon_i \langle A_i, Z \rangle | \le c \|Z\|_F) \ge 1 - 2 (36\sqrt{2})^{r(2d+1)} \exp\left(-\frac{c^2n}{8\beta\sigma^2}\right).
\]
\end{lemma}

\begin{proof}{Proof of Lemma~\ref{lemma:gold_concenineq}}
Without loss of generality, consider $\mathcal{Z} = \{Z \in \mathbb{R}^{d \times d} | \rank(Z) \le r, \|Z\|_F=1\}$. Define $\mathcal{N}$ to be a $\frac{1}{4\sqrt{2}}$-net of $\mathcal{Z}$. Lemma~\ref{lemma:covernum} gives the covering number for the set $\mathcal{Z}$: 
\[
|\mathcal{N}| \le (36\sqrt{2})^{r(2d+1)}.
\]
For any $Z \in \mathcal{Z}$, there exists $Z' \in \mathcal{N}$ with $\|Z - Z'\|_F \le \frac{1}{4\sqrt{2}}$, such that 
\begin{equation} \label{eq:ep_net_decomp}
| \sum_{i\in[n]} \epsilon_i \langle A_i, Z \rangle | \le | \sum_{i\in[n]} \epsilon_i\langle A_i, Z' \rangle | + | \sum_{i\in[n]} \epsilon_i \langle A_i, Z - Z' \rangle |.
\end{equation}
Set $\Delta_Z = Z - Z'$ and note that $\rank(\Delta_Z) \le 2r$. We decompose $\Delta_Z$ into two matrices, $\Delta_Z = \Delta_{Z,1} + \Delta_{Z,2}$, that satisfy $\rank(\Delta_{Z,j}) \le r$ for $j=1, 2$ and $\langle \Delta_{Z,1}, \Delta_{Z,2} \rangle = 0$ (e.g. through SVD).
As $\|\Delta_{Z,1}\|_F + \|\Delta_{Z,2}\|_F \le \sqrt{2} \|\Delta_Z\|_F$, we have $\|\Delta_{Z,j}\|_F \le \frac{1}{4}, \, j=1, 2$. Combined with inequality (\ref{eq:ep_net_decomp}), we have
\[
| \sum_{i\in[n]} \epsilon_i \langle A_i, Z \rangle | \le \sup_{Z' \in \mathcal{N}} | \sum_{i\in[n]} \epsilon_i\langle A_i, Z' \rangle | + \frac{1}{2} \sup_{Z \in \mathcal{Z}} | \sum_{i\in[n]} \epsilon_i \langle A_i, Z \rangle |.
\]
Since the above holds for any $Z \in \mathcal{Z}$, the following holds:
\[
\sup_{Z \in \mathcal{Z}} | \sum_{i\in[n]} \epsilon_i \langle A_i, Z \rangle | \le 2 \sup_{Z' \in \mathcal{N}} | \sum_{i\in[n]} \epsilon_i \langle A_i, Z' \rangle |.
\]
Then it follows from the union bound that
\begin{align*}
\mathbb{P}(\sup_{Z \in \mathcal{Z}} |\frac{1}{n} \sum_{i\in[n]}\epsilon_i \langle A_i, Z \rangle | \ge c)
& \le \mathbb{P}(\sup_{Z \in \mathcal{N}}|\frac{1}{n} \sum_{i\in[n]} \epsilon_i \langle A_i, Z \rangle | \ge \frac{c}{2}) \\
& \le |\mathcal{N}| \max_{Z \in \mathcal{N}} \mathbb{P}(|\frac{1}{n} \sum_{i\in[n]} \epsilon_i \langle A_i, Z \rangle | \ge \frac{c}{2}) \\
& \le 2|\mathcal{N}| \exp\left(-\frac{c^2n}{8\beta\sigma^2}\right) \\
& = 2 (36\sqrt{2})^{r(2d+1)} \exp\left(-\frac{c^2n}{8\beta\sigma^2}\right).
\end{align*}
The last inequality uses $r$-smoothness($\beta$) of $\mathcal{A}$ and a tail inequality of $\sigma$-subgaussian random variables. \Halmos
\end{proof}

\begin{proof} {Proof of Theorem~\ref{thm:gold_estimator}}
The proof mainly follows Theorem 8 and Theorem 31 of \citet{ge2017no} but we also provide here for completeness. As in \citet{ge2017no}, we use the notation $U:\mathcal{H}:V$ to denote the inner product $\langle U, \mathcal{H}(V) \rangle$ for $U, V \in \mathbb{R}^{d_1 \times d_2}$. The linear operator $\mathcal{H}$ can be viewed as a $d_1d_2 \times d_1d_2$ matrix. In our problem (\ref{eq:gold_optim}), we define 
\[
\Theta:\mathcal{H}:\Theta = \frac{1}{n_g} \|\mathcal{A}_{g}(\Theta)\|^2
\]
for any $\Theta \in \mathbb{R}^{d \times d}$.
We can rewrite problem (\ref{eq:gold_optim}) as 
\[
\min_{U_g} f(U_g) = \frac{1}{n_g} \| X_g - \mathcal{A}_g(U_g U_g^T) \|^2.
\]
Rearrange the objective function with $\mathcal{H}$ and we have
\begin{align*}
f(U_g) = (U_g U_g^T - \Theta_g^*):\mathcal{H}:(U_g U_g^T - \Theta_g^*) + Q(U_g),
\end{align*}
with
\begin{align*}
Q(U_g) = -\frac{2}{n_g} \sum_{i \in [n_g]} \langle A_{g,i}, U_g U_g^T - \Theta_g^* \rangle \epsilon_{g, i} + \frac{1}{n_g} \sum_{i \in [n_g]} \epsilon_{g, i}^2.
\end{align*}
Define $\Delta = \widehat{U}_g - U_g^* R_{(\widehat{U}_g, U_g^*)}$.
By Lemma 7 from \citet{ge2017no}, we have for the Hessian $\nabla^2 f(\widehat{U}_g)$ with $\nabla f(\widehat{U}_g) = 0$
\[
\Delta:\nabla^2 f(\widehat{U}_g):\Delta = 2\Delta\Delta^T:\mathcal{H}:\Delta\Delta^T - 6(\widehat{\Theta}_g - \Theta_g^*):\mathcal{H}:(\widehat{\Theta}_g - \Theta_g^*) + \Delta:\nabla^2 Q(\widehat{U}_g):\Delta - 4 \langle \nabla Q(\widehat{U}_g), \Delta\rangle.
\]
Using Lemma~\ref{lemma:factormatrixbound} and the $2r$-RWC assumption, the above inequality can be simplified as 
\[
\Delta:\nabla^2 f(\widehat{U}_g):\Delta \le -2(3\alpha_g - 2\beta_g) \|\widehat{\Theta}_g - \Theta_g^*\|_F^2 + \Delta:\nabla^2 Q(\widehat{U}_g):\Delta - 4 \langle \nabla Q(\widehat{U}_g), \Delta\rangle.
\]
We then bound the terms related to function $Q$. Note that
\begin{align*}
\Delta:\nabla^2 Q(\widehat{U}_g):\Delta - 4 \langle \nabla Q(\widehat{U}_g), \Delta\rangle 
= \frac{4}{n_g} \sum_{i \in [n_g]} \langle A_{g,i}, \widehat{\Theta}_g - \Theta_g^* \rangle \epsilon_{g, i} + \frac{4}{n_g} \sum_{i \in [n_g]} \langle A_{g,i}, \widehat{U}_g \Delta^T - \Delta\widehat{U}_g^T \rangle \epsilon_{g, i}
\end{align*}
Define $\mathcal{Z} = \{Z \in \mathbb{R}^{d \times d} \mid \rank(Z) \le 2r\}$.
On the event
\[
\mathcal{E}_g = \left\{ \sup_{Z \in \mathcal{Z}} | \frac{1}{n_g} \sum_{i\in[n_g]} \epsilon_{g,i} \langle A_{g,i}, Z \rangle | \le c \|Z\|_F \right\},
\]
it holds that
\begin{align*}
\frac{4}{n_g} \sum_{i \in [n_g]} \langle A_{g,i}, \widehat{\Theta}_g - \Theta_g^* \rangle \epsilon_{g, i} & \le 4c\|\widehat{\Theta}_g - \Theta_g^*\|_F \\
\frac{4}{n_g} \sum_{i \in [n_g]} \langle A_{g,i}, \widehat{U}_g \Delta^T - \Delta\widehat{U}_g^T \rangle \epsilon_{g, i} & \le 4(1 + \sqrt{2})c\|\widehat{\Theta}_g - \Theta_g^*\|_F,
\end{align*}
where the second inequality uses Lemma~\ref{lemma:factormatrixbound}. Therefore, we have
\[
\Delta : \nabla^2 f(\widehat{U}_g) : \Delta \le -2(3\alpha_g - 2\beta_g) \|\widehat{\Theta}_g - \Theta_g^*\|_F^2 + (8+4\sqrt{2})c \|\widehat{\Theta}_g - \Theta_g^*\|_F.
\]
Since $\widehat{U}_g$ is a local minimum, we must have
\[
-2(3\alpha_g - 2\beta_g) \|\widehat{\Theta}_g - \Theta_g^*\|_F^2 + (8+4\sqrt{2})c \|\widehat{\Theta}_g - \Theta_g^*\|_F \ge 0,
\]
that is, $\widehat{\Theta}_g$ satisfies
\[
\|\widehat{\Theta}_g - \Theta_g^*\|_F \le \frac{(4+2\sqrt{2})c}{3\alpha_g - 2\beta_g}.
\]
Again using Lemma~\ref{lemma:factormatrixbound} gives
\[
\|\widehat{U}_g - U_g^* R_{(\widehat{U}_g, U_g^*)}\|_F \le \frac{1}{\sqrt{2(\sqrt{2} - 1)\sigma_{r}(\Theta_g^*)}} \|\widehat{\Theta}_g - \Theta_g^*\|_F \le \frac{8 c}{(3\alpha_g - 2\beta_g)\sigma_{r}(U_g^*)}.
\]
Further by Cauchy-Schwarz, we have
\[
\|\widehat{U}_g - U_g^* R_{(\widehat{U}_g, U_g^*)}\|_{2,1} \le \frac{8 c \sqrt{d}}{(3\alpha_g - 2\beta_g)\sigma_{r}(U_g^*)}.
\]
Lemma~\ref{lemma:gold_concenineq} shows with high probability $\mathcal{E}_g$ holds:
\[
\mathbb{P}(\mathcal{E}_g) \ge 1 - 2 (36\sqrt{2})^{2r(2d+1)} \exp\left(-\frac{c^2n_g}{8\beta_g\sigma_g^2}\right). 
\]
The result follows by taking
\[
c = \sqrt{\frac{8\beta_g\sigma_g^2 (2r(2d+1)\log(36\sqrt{2}) + \log(\frac{2}{\delta}))}{n_g}}. \Halmos
\]
\end{proof}

\section{Error Bound of Proxy Estimator}
\label{sec:thmproxyproof}

\begin{proof}{Proof of Theorem~\ref{thm:proxy_estimator}}
Same as the proof of Theorem~\ref{thm:gold_estimator}, we get
\[
\|\widehat{U}_p - U_p^* R_{(\widehat{U}_p, U_p^*)}\|_{2, 1} \le \frac{8 c \sqrt{d}}{(3\alpha_p - 2\beta_p)\sigma_{r}(U_p^*)}.
\]
On the event 
\[
\mathcal{E}_p = \left\{ \sup_{Z \in \mathcal{Z}} | \frac{1}{n_p} \sum_{i\in[n_p]} \epsilon_{p,i} \langle A_{p,i}, Z \rangle | \le c \|Z\|_F \right\}.
\]
To measure the estimation error of $\widehat{U}_p$ for $U_g^*$, we need to align $\widehat{U}_p$ with $U_g^*$.
The estimation error of using proxy estimator for gold data is
\begin{align*}
\|\widehat{U}_p - U_g^* R_{(\widehat{U}_p, U_g^*)}\|_{2, 1}
= & \|\widehat{U}_p - U_p^* R_{(\widehat{U}_p, U_p^*)} + U_p^* R_{(\widehat{U}_p, U_p^*)} - (U_p^* +  \Delta_U^*) R_{(\widehat{U}_p, U_g^*)}\|_{2, 1} \\
\le & \|\widehat{U}_p - U_p^* R_{(\widehat{U}_p, U_p^*)}\|_{2,1} + \|U_p^* (R_{(\widehat{U}_p, U_p^*)} - R_{(\widehat{U}_p, U_g^*)})\|_{2, 1} + \|\Delta_U^{*}\|_{2, 1}.
\end{align*}
Therefore, we have
\[
\|\widehat{U}_p - U_g^* R_{(\widehat{U}_p, U_g^*)}\|_{2, 1} \le \|\Delta_U^{*}\| + \|U_p^* (R_{(\widehat{U}_p, U_p^*)} - R_{(\widehat{U}_p, U_g^*)})\|_{2, 1} + \frac{8 c \sqrt{d}}{(3\alpha_p - 2\beta_p)\sigma_{r}(U_p^*)}.
\]
By Lemma~\ref{lemma:gold_concenineq}, 
\[
\mathbb{P}(\mathcal{E}_p) \ge 1 - 2 (36\sqrt{2})^{2r(2d+1)} \exp\left(-\frac{c^2n_p}{8\beta_p\sigma_p^2}\right).
\]
Similarly, the result follows by taking 
\[
\omega = \|U_p^* (R_{(\widehat{U}_p, U_p^*)} - R_{(\widehat{U}_p, U_g^*)})\|_{2,1},
\]
and
\[
c = \sqrt{\frac{8\beta_p\sigma_p^2 (2r(2d+1)\log(36\sqrt{2}) + \log(\frac{2}{\delta}))}{n_p}}. \Halmos
\]
\end{proof}

\section{Useful Lemmas}

\begin{lemma}\label{lemma:covernum} 
Let $\mathcal{Z} = \{Z \in \mathbb{R}^{d_1 \times d_2} | \rank(Z) \le r, \|Z\|_F=1\}$. Then there exists an $\epsilon$-net $\mathcal{N} \subseteq \mathcal{Z}$ with respect to the Frobenius norm obeying
\[
|\mathcal{N}| \le (9/\epsilon)^{(d_1 + d_2 + 1)r}.
\]
\end{lemma}
\begin{proof}{Proof of Lemma~\ref{lemma:covernum}}
See Lemma 3.1 of \citet{candes2011tight}. \Halmos
\end{proof}

\begin{lemma}\label{lemma:factormatrixbound} 
Let $\Delta = \widehat{U} - U^* R_{(\widehat{U}, U^*)}$, $\Theta^* = U^* U^{*T}$ and $\widehat{\Theta} = \widehat{U} \widehat{U}^T$, where $R_{(\widehat{U}, U^*)}$ is defined in Definition~\ref{def:error_dir}. Then, 
\begin{align*}
\|\Delta \Delta^T\|_F^2 & \le 2 \|\widehat{\Theta} - \Theta^*\|_F^2 \\
\sigma_r(\Theta^*)\|\Delta\|_F^2 & \le \frac{1}{2(\sqrt{2}-1)} \|\widehat{\Theta} - \Theta^*\|_F^2.
\end{align*}
\end{lemma}
\begin{proof}{Proof of Lemma~\ref{lemma:factormatrixbound}}
See Lemma 6 of \citet{ge2017no}. \Halmos
\end{proof}

\begin{lemma}\label{lemma:subgaussiansquare} 
Let $X \in \mathbb{R}^n$ be a $\sigma$-subgaussian random vector,
$A \in \mathbb{R}^{m \times n}$ and $\Sigma = A^T A$.
Then, for any $t > 0$,
\[
\mathbb{P}(\|AX\|^2 > \sigma^2 (\tr(\Sigma) + 2 \|\Sigma\|_F \sqrt{t} + 2 \|\Sigma\|t)) \le \exp(-t).
\]
\end{lemma}
\begin{proof}{Proof of Lemma~\ref{lemma:subgaussiansquare}}
See Theorem 1 of \citet{hsu2012tail}. \Halmos
\end{proof}

\begin{lemma}\label{lemma:gaussianfuncmaxima} 
Let gaussian random vector $X = \begin{bmatrix} X_1 & \cdots & X_n \end{bmatrix}^T \in \mathbb{R}^n$ with i.i.d. $X_i \sim N(0, 1)$, and $f: \mathbb{R}^n \rightarrow \mathbb{R}$ an $L$-Lipschitz function, i.e., $|f(x) - f(y)|\le L \|x-y\|$ for any $x, y \in \mathbb{R}^n$. Then, for any $t > 0$
\[
\mathbb{P}(f(X) - \mathbb{E}[f(X)] > t) \le \exp(-\frac{t^2}{2 L^2}).
\]
\end{lemma}
\begin{proof}{Proof of Lemma~\ref{lemma:gaussianfuncmaxima}}
 See Theorem 5.6 in \citet{boucheron2013concentration}. \Halmos
\end{proof}

\begin{lemma}\label{lemma:gaussiannormmaxima} For gaussian random vector $X = \begin{bmatrix} X_1 & \cdots & X_n \end{bmatrix}^T \in \mathbb{R}^n$ with $X \sim N(0, \Sigma)$, the demeaned $\|X\|$ is subgaussian, i.e., for any $t > 0$,
\[
\mathbb{P}(\|X\| - \mathbb{E}[\|X\|] > t) \le \exp(-\frac{t^2}{2 \|\Sigma\|}).
\]
\end{lemma}
\begin{proof}{Proof of Lemma~\ref{lemma:gaussiannormmaxima}}
It is a direct application of Lemma~\ref{lemma:gaussianfuncmaxima}. \Halmos
\end{proof}

\begin{lemma}\label{lemma:subgaussianmaxima} For random variables $X_i, i \in [n]$ with $X_i$ drawn from $\sigma_i$-subgaussian,
\begin{align*}
\mathbb{E}[\max_{i \in [n]} X_i] \le \max_{i \in [n]} \sigma_i \sqrt{2\log n}, 
\quad \mathbb{E}[\max_{i \in [n]} |X_i|] \le \max_{i \in [n]} \sigma_i \sqrt{2\log(2n)}.
\end{align*}
\end{lemma}
\begin{proof}{Proof of Lemma~\ref{lemma:subgaussianmaxima}}
It is a simple extension of Theorem 1.14 of \cite{rigollet201518}. \Halmos
\end{proof}

\begin{lemma}\label{lemma:gammafuncineq} For $\Gamma$ function and any integer $n$, we have
\begin{align*}
\frac{n}{\sqrt{n+1}} \le \sqrt{2} \frac{\Gamma(\frac{n+1}{2})}{\Gamma(\frac{n}{2})} \le \sqrt{n}.
\end{align*}
\end{lemma}
\begin{proof}{Proof of Lemma~\ref{lemma:gammafuncineq}}
It is easy to prove by induction. \Halmos
\end{proof}

\begin{lemma}\label{lem: matrix_berstein_inequality}
    Let $Z_1,\cdots,Z_n$ be independent matrices in $\mathbb{R}^{d_1\times d_2}$ s.t. $\mathbb{E}[Z_i]=0$ and $\|Z_i\|\le D$ almost surely for all $i\in[n].$ Let $\sigma_Z$ be such that
    \begin{align*}
        \sigma_Z^2\geq \max\left\{ \|\frac{1}{n}\sum_{i\in[n]}\mathbb{E}[Z_i^TZ_i]\|, \|\frac{1}{n}\sum_{i\in[n]}\mathbb{E}[Z_iZ_i^T]\|\right\}.
    \end{align*}
    Then for any $t>0$,
    \begin{align*}
        \mathbb{P}(\|\frac{1}{n}\sum_{i\in[n]} Z_i\| \geq t) \leq (d_1 + d_2)\exp(\frac{-nt^2}{2\sigma_Z^2+(2Dt)/3}).
    \end{align*}
\end{lemma}
\begin{proof}{Proof.}
    This is a simple extension of Proposition 1 of \cite{athey2021matrix}. \Halmos
\end{proof}

\section{Experimental Details}\label{sec:add_exp}

This section provides details on the setup of experiments described in \S\ref{sec:experiments} on both synthetic and real data.

\subsection{Synthetic Data} \label{app:exp-synth}

\paragraph{Experimental Details.} We consider a low-data regime where the gold and proxy sample sizes are $n_g=50$ and $n_p=5,000$ respectively. The observation matrices $A_{p, i}, A_{g, i} \in \mathbb{R}^{d\times d}$ are independent gaussian random matrices whose entries are i.i.d. $N(0, 1)$. The parameter $\Theta_p^*\in \mathbb{R}^{d\times d}$ is created by choosing $U_p^* \in\mathbb{R}^{d\times r}$ with i.i.d. $N(0, 1)$ entries. Then, we generate $\Theta_g^*$ by setting the row sparsity of $\Delta_U^*$ to s, randomly picking $s$ rows our of $d$, and drawing the value of each entry from a uniform distribution $\texttt{Uniform}[-1, 1]$. Additionally, we sample noise terms $\epsilon_{p, i}, \epsilon_{g, i}$ independently from a gaussian distribution $N(0, 1)$. We compute the gold, proxy, and transfer learning estimators by solving optimization problems (\ref{eq:gold_optim}), (\ref{eq:proxy_optim}), and (\ref{eq:joint_optim}), respectively.

\paragraph{Cross Validation.} To construct our transfer learning estimator, we also need to pick a proper value for the hyperparameter $\lambda$ to balance bias and variance. Although we provide a theoretically justified expression for $\lambda$ in Theorem \ref{thm:joint_estimator}, this choice depends on problem-dependent parameters that are typically not known in practice. 

Rather, in practice, the hyperparameters are typically chosen using the popular cross-validation method \citep{kohavi1995study,hastie2009elements}; specifically, in the context of low-rank matrix factorization and group-sparse regression, $k$-fold cross validation is typically used to tune hyperparameters \citep{huang2010benefit,chen2013reduced,cai2016structured}. 
Here we use a 5-fold cross validation to tune $\lambda$ on a pre-specified grid. In particular, we split the full gold sample into 5 parts; for each of the 5 times, we use 4 parts (i.e., 80\% of the gold data) as the training set and calculate the Frobenius error on the remaining one part (i.e., 20\% of the gold data), which is the validation set. Note that our estimates of the embeddings $U_g^*$ are invariant under an orthogonal change-of-basis (see discussion in \S\ref{sec:cla_mff}), so we use the rotation-invariant Frobenius norm to measure estimation error of $\Theta_g^*$. We pick the value of $\lambda$ that gives the lowest average Frobenius error across the 5 runs. 

\subsection{Wikipedia} \label{app:wiki}

\paragraph{Data Pre-processing.} All the Wikipedia text data were downloaded from the English Wikipedia database dumps\footnote{\url{https://dumps.wikimedia.org/enwiki/latest/}} in January 2020. We preprocess the text using a standard approach---i.e., splitting and tokenizing sentences, removing short sentences that contain less than 20 characters or 5 tokens, and removing stop words. 
We download the pre-trained word embeddings from GloVe's official website.\footnote{\url{https://nlp.stanford.edu/projects/glove/}} In our experiment, we use the pre-trained vectors trained from the 2014 Wikipedia dump and Gigaword 5, which contains around 6 billion tokens and 400K vocabulary words. 

\paragraph{Experimental Details.} We estimate the gold estimator based on (\ref{eq:gold_optim}) using domain textual data. Given the high computational cost of training proxy word embeddings from the entire Wikipedia corpus using our method, we adopt GloVe pre-trained embeddings as a proxy estimator. We implement our transfer learning method based on (\ref{eq:joint_optim}). To extend our approach to the GloVe objective, we solve the optimization problem (\ref{eq:glovejoint_optim}). The Mittens word embeddings are obtained solving a similar problem as (\ref{eq:glovejoint_optim}), but with the Frobenius norm penalty---i.e., $\sum_{i\in[d]}\|(U_i + V_i)- \widehat{U}_p^i\|^2.$
Our model setup follows that of GloVe. We create the co-occurrence matrix using a symmetric context window of length 5. We choose the dimension of the word embedding to be 100. To ensure fair comparison, we tune the hyperparameters for all methods. We found our results to be robust to the choice of hyperparameters. 

To measure the predictive accuracy, we split our data of word co-occurrence randomly into training set (i.e., 90\% of the gold data) and test set (i.e., 10\% of the gold data) for each article. Then, we compute the average predictive error of each estimator for each article over 5 trials. To compare different estimators, we aggregate article-level predictive errors into a domain-level predictive error, which is an average across article-level predictive errors normalized by article length of articles within the corresponding domain. 

To identify domain-specific words for each article, we score each word $i$ by the $\ell_2$ distance between its new embedding (e.g, our transfer learning estimator or Mittens) and its pre-trained embedding; a higher score indicates a higher likelihood of being a domain-specific word. To evaluate the accuracy of domain-specific word identification, we choose a threshold of 10\%, and select and compare the top 10\% of words according to this score for each estimator. In other words, we treat all words in the top 10\% as positives identified by each estimator, and accordingly the rest 90\% of the words are negatives identified by each estimator. Then, we calculate an $F1$ score of each estimator for each article. We compare different estimators using a domain-level $F1$ score, which is an average over article-level $F1$ scores, normalized by article length, across articles within the same domain. 

\paragraph{Description of CCA and KCCA baselines.}

We provide additional details on the Canonical Correlation Analysis (CCA) and the closely related kernelized version (KCCA)~\citep{sarma2018domain}. The CCA estimator aligns the domain specific word embedding $\widehat{U}_{g, i}$ and the pre-trained word embedding $\widehat{U}_{p, i}$ for each word $i$, transforming them to $\bar{U}_{g, i}$ and $\bar{U}_{p, i}$ respectively. It is computed as $\frac{1}{2}\bar{U}_{g, i}+\frac{1}{2}\bar{U}_{p, i}$ for each word $i$, i.e., the average of the aligned domain-specific word embeddings and the pre-trained word embeddings. In contrast, the KCCA estimator transforms the embeddings $\widehat{U}_{g, i}$ and $\widehat{U}_{p, i}$ for word $i$ through a kernel CCA, instead of CCA. Following \cite{sarma2018domain}, we set the hyperparameter $\sigma$ of the gaussian kernel to be the median of pairwise distances between domain-specific word embeddings and pre-trained word embeddings. See \cite{sarma2018domain} for additional details on the implementations of both of these estimators.

\section{Additional Experiments}\label{sec:exp_ex}

This section details additional experiments to complement and extend the main experimental results in \S\ref{sec:experiments}.

\subsection{Synthetic Data}\label{sec:exp_ex_sync}

We conduct additional experiments to show how our estimator's performance varies based on problem-specific parameters, as well as to assess the robustness of performance with respect to the choice of hyperparameters. The experimental results are shown in Figure \ref{fig:synthetic_vary_gsize} and \ref{fig:synthetic_vary}. In all these experiments, we follow the same experimental setting of Figure~\ref{fig:synthetic_exp} (a) with $d=20$, $r=5$, and $s=2$, except for Figure \ref{fig:synthetic_vary} (a) and (b) where we take $d=40$, $r=5$, and $s=2$. We consider the regime with limited gold data $(n_g=50)$ and large proxy data ($n_p=5,000$). For our transfer learning estimator, we use cross-validation to tune hyperparameter $\lambda$ across all settings. Details on data generation and hyperparameter tuning are provided in Appendix~\ref{app:exp-synth}.

First, aligned with our theory, Figure \ref{fig:synthetic_vary_gsize} shows the estimation error declines with the gold sample size for both the gold estimator and our transfer learning estimator. Note that the proxy estimator only depends on the proxy sample size and hence its performance does not vary as a function of the gold sample size. These results supports our theoretical finding that our estimator performs consistently better than other benchmarks in the low-data regime where $n_p \gg d^2$ and $n_g \ll d^2$ (in this example, $n_p=5,000$ and $d=20$). Intuitively, if the gold estimator has low estimation error, we can always set the hyperparameter $\lambda$ to be 0 so that our transfer learning estimator equals the gold estimator. Therefore, with flexibility in tuning hyperparameters, our estimator should always weakly outperform the gold estimator in practice. More interestingly, our results also suggest that our estimator performs relatively well even in the regime with moderate gold sample size (in this example, when $n_g \ge d^2 = 400$). 

\begin{figure}[htbp]
\begin{center}
\includegraphics[width=.45\columnwidth]{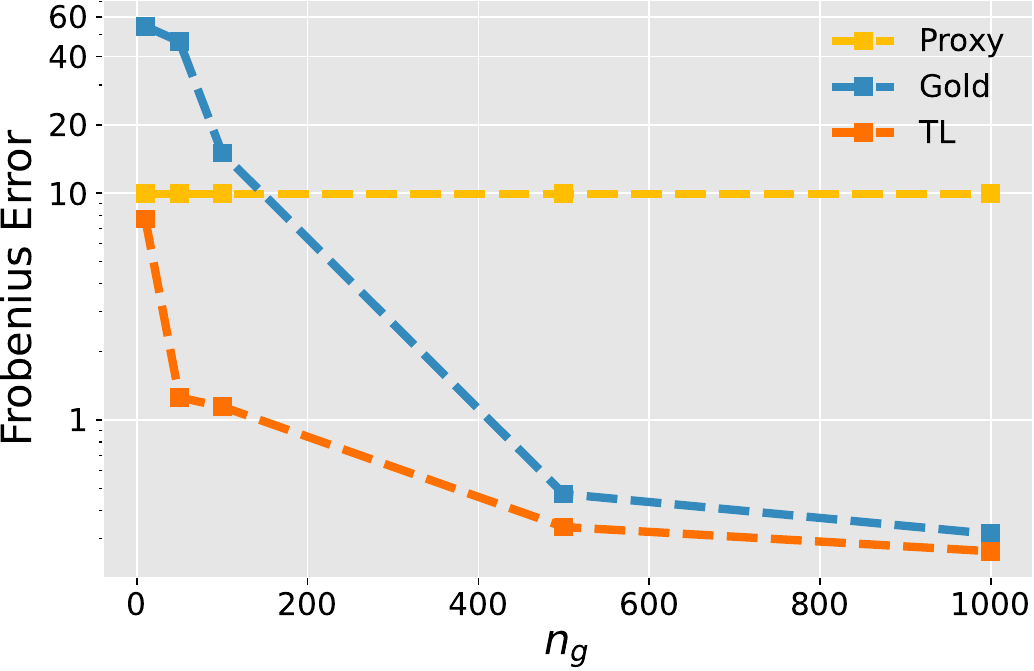}
\caption{Lines depict the Frobenius norm estimation errors averaged over 100 trials. The 95\% confidence intervals are too small to be observed and hence not shown. `TL' denotes our transfer learning estimator.}
\label{fig:synthetic_vary_gsize}
\end{center}
\end{figure}

Next, also consistent with our theory, we show that the estimation error of our estimator increases with the group sparsity level $s$, the matrix rank $r$, and the magnitude of $\Delta_U^*$, i.e., $L$, (remember $\|\Delta_U^*\|_{2,1} \le L$). Figure \ref{fig:synthetic_vary} (a) shows the estimation error of our method increases with the group sparsity level $s$. Intuitively, when the gold and proxy tasks become more heterogeneous---e.g., a higher sparsity level implying that less information can be shared---transfer learning becomes harder. Figure \ref{fig:synthetic_vary} (b) analyzes the performance of our method with different values of matrix rank $r$. Intuitively, higher matrix rank $r$ means more within-group parameters to learn, and thus will increase the learning difficulty for all estimators. Figure~\ref{fig:synthetic_vary} (c) shows our estimation error for different values of $L$. Specifically, we draw the value of each entry of the $s$ nonzero rows of $\Delta_U^*$ from a uniform distribution $\texttt{Uniform}[-a, a]$. Thus, $L$ becomes larger when $a$ takes larger values. Again, this result shows that transfer learning becomes harder when the gold and proxy problems are more different. Note that when $a$ takes smaller values, $\Delta_U^*$ is easier to estimate, and the result here is consistent with that in Figure~\ref{fig:synthetic_vary} (a).

\begin{figure}[htbp]
\centering
\begin{subfigure}[b]{0.32\textwidth}
  \centering
  \includegraphics[width=\textwidth]{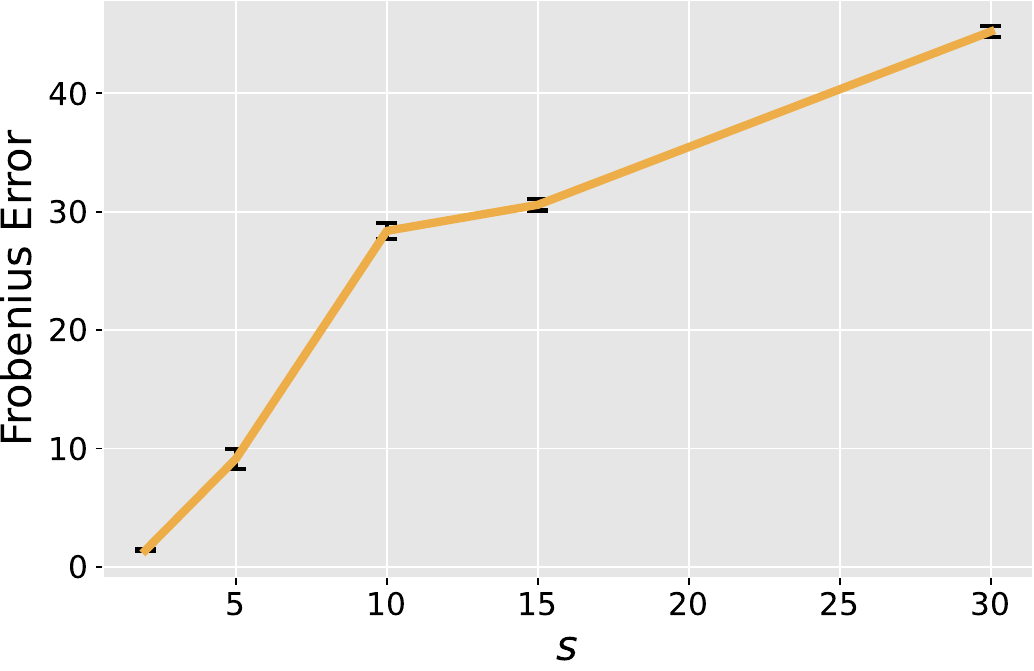}
  \caption{Varying $s$}
\end{subfigure}
\begin{subfigure}[b]{0.32\textwidth}
  \centering
  \includegraphics[width=\textwidth]{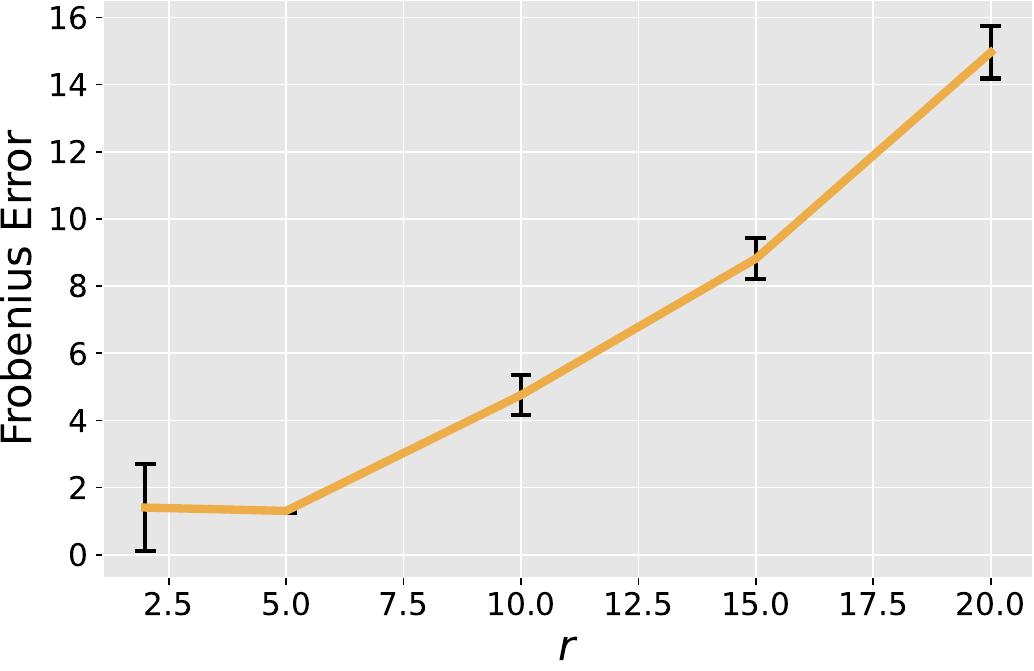}
  \caption{Varying $r$}
\end{subfigure}
\begin{subfigure}[b]{0.32\textwidth}
  \centering
  \includegraphics[width=\textwidth]{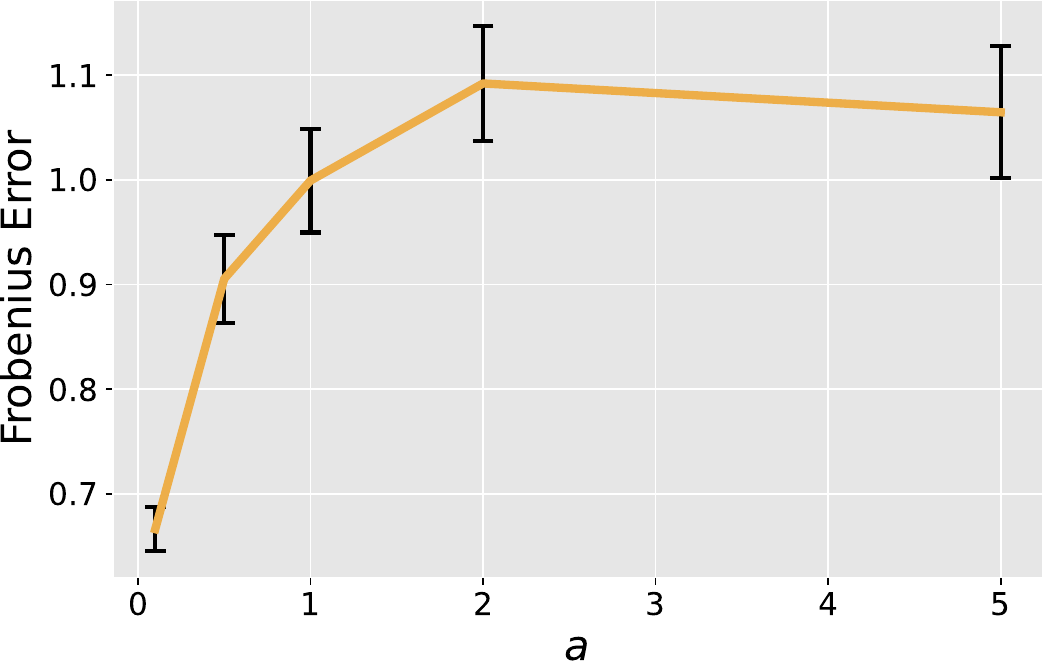}
  \caption{Varying $a$}
\end{subfigure}
\caption{Lines depict the Frobenius norm estimation errors of our transfer learning estimator averaged over 100 trials, with error bars the corresponding 95\% confidence intervals.}
\label{fig:synthetic_vary}
\end{figure}

Lastly, we study the robustness of our estimator towards the hyperparameters. Figure \ref{fig:synthetic_rbs} shows the Frobenius norm estimation error of our transfer learning estimator with varying values of the hyperparameter $\lambda$, compared with the benchmark errors of proxy and gold estimators. In our low-data regime, the Frobenius error of our method is not substantially affected by varying values of the hyperparameter; particularly, our method still dominates the two other benchmarks consistently over different values of $\lambda$. This suggests that our algorithm is robust in low-data regime, which is important especially in empirical applications where these hyperparameters might not be well specified. 

\begin{figure}[htbp]
\centering
\includegraphics[width=.45\columnwidth]{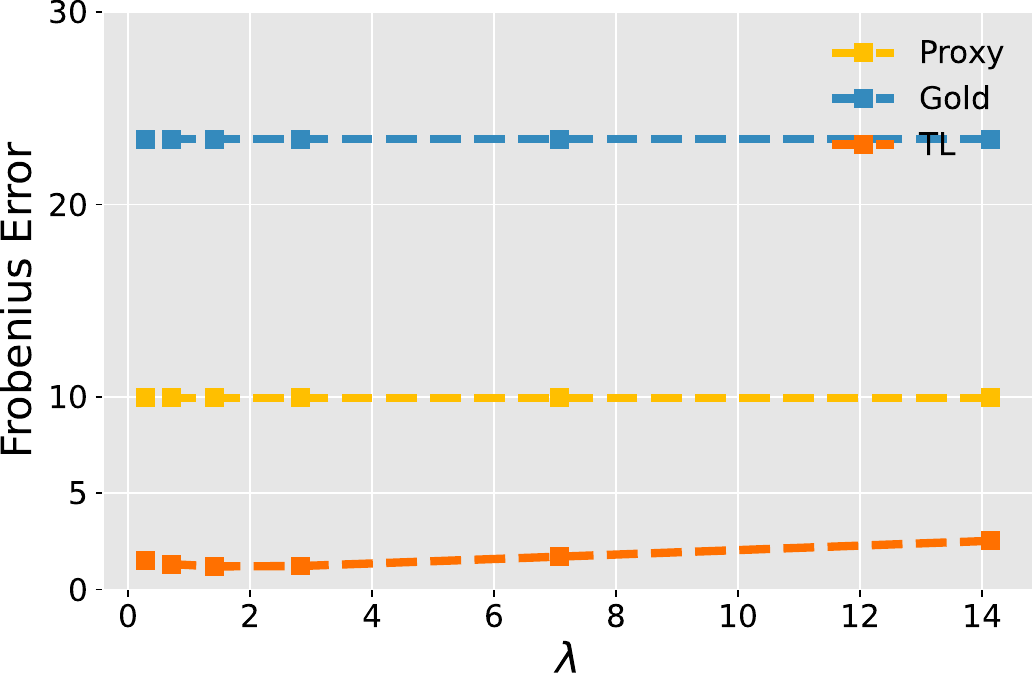}
\caption{Lines depict the Frobenius norm estimation errors averaged over 100 trials. The 95\% confidence intervals are too small to be observed and hence not shown. `TL' denotes our transfer learning estimator.}
\label{fig:synthetic_rbs}
\end{figure}

\subsection{Wikipedia}\label{app:exp_ex_wiki}

We additionally evaluate how our transfer learning estimator performs when varying the value of selection threshold, which determines the criteria for domain-specific words; in particular, we consider 10\%, 20\%, and 30\% (note our main experimental result uses a threshold of 10\%). Figure~\ref{fig:F1_perc_fin} shows the weighted $F_1$-score versus the top percentage set for the threshold in the finance domain. Our approach consistently outperforms all baselines including CCA and KCCA over different selection thresholds, illustrating that it is robust to how we define domain-specific words.
\begin{figure}[htbp]
\begin{center}
\includegraphics[width=.5\columnwidth]{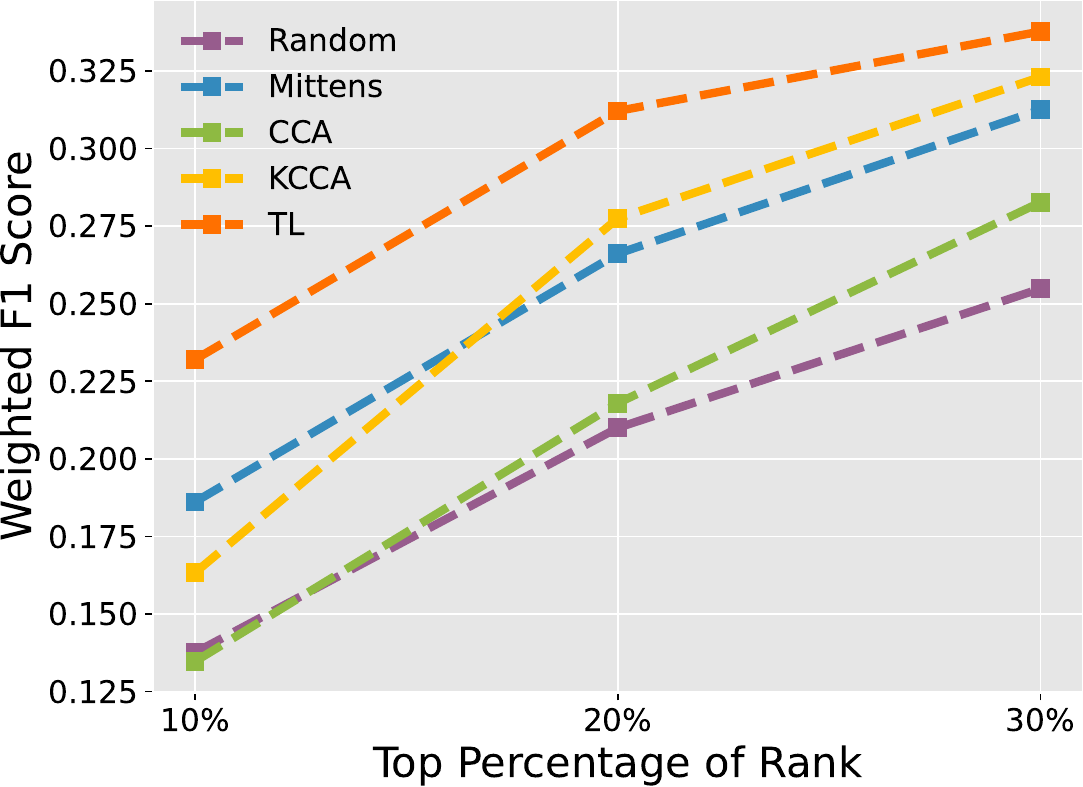}
\caption{Average $F1$ score (weighted by article length) versus top percentage of the rank set for the threshold in the finance domain. ``TL'' represents our transfer learning approach.}
\label{fig:F1_perc_fin}
\end{center}
\end{figure}

\end{APPENDICES}

\end{document}